\documentclass[a4paper,lineno]{paper}

\usepackage{hyperref}
\usepackage{amsmath,amssymb,amsfonts,amsthm}
\usepackage{mathrsfs}
\usepackage{cleveref}
\usepackage{acro}
\usepackage{float}
\usepackage{graphicx}
\usepackage{subcaption}
\usepackage[squaren,thickqspace]{SIunits}
\usepackage{booktabs}
\usepackage{placeins}
\usepackage{paralist}

\newtheorem{theorem}{Theorem}

\newtheorem{proposition}[theorem]{Proposition}

\newtheorem{claim}[theorem]{Claim}

\newcommand{\Real}{\ensuremath{\mathbb{R}}}
\newcommand{\RecSpace}{\ensuremath{X}}
\newcommand{\DataSpace}{\ensuremath{Y}}
\newcommand{\OtherSpace}{\ensuremath{W}}
\DeclareMathOperator{\Lip}{Lip}
\newcommand{\LipClass}[2]{#1 \in \Lip(#2)}
\DeclareMathOperator{\Prob}{Prob}
\DeclareMathOperator{\Wasserstein}{\mathcal{W}}
\newcommand{\ProbDist}{\ell}
\newcommand{\stochastic}[1]{\mathsf{#1}}
\newcommand{\PClass}{\mathcal{P}}
\DeclareMathOperator{\Expect}{\mathbb{E}}

\DeclareMathOperator{\PointwiseVariance}{pVar}
\newcommand{\jointlaw}{\mu}
\newcommand{\posterior}[1]{\pi(\stsignal \mid #1)}
\newcommand{\dataprob}{\sigma}
\DeclareMathOperator*{\argmin}{arg\,min}
\newcommand{\learnedInv}[1]{#1^{\dagger}}
\DeclareMathOperator{\ForwardOp}{\ensuremath{\mathcal{T}}}
\newcommand{\RecOp}{\learnedInv{\ForwardOp}}
\newcommand{\signal}{\ensuremath{x}}
\newcommand{\signalother}{\ensuremath{v}}
\newcommand{\data}{\ensuremath{y}}
\newcommand{\otherdata}{w}
\newcommand{\stsignal}{\stochastic{\signal}}
\newcommand{\stsignalother}{\stochastic{\signalother}}
\newcommand{\stdata}{\stochastic{\data}}
\newcommand{\stotherdata}{\stochastic{\otherdata}}
\newcommand{\stepsilon}{\stochastic{\epsilon}}
\newcommand{\recparam}{\theta}
\DeclareMathOperator{\GenProbMapSpace}{\mathfrak{G}}

\DeclareMathOperator{\GeneratorSpace}{\mathscr{G}}
\DeclareMathOperator{\Generator}{G}
\newcommand{\GenProb}{\mathcal{G}}
\newcommand{\GenParamSet}{\Theta}
\newcommand{\genparam}{\theta}
\newcommand{\GenSpace}{Z}
\newcommand{\genvar}{z}
\newcommand{\genvarprob}{\eta}
\newcommand{\stgenvar}{\stochastic{\genvar}}
\DeclareMathOperator{\DiscrSpace}{\mathscr{D}}
\DeclareMathOperator{\Discriminator}{D}
\newcommand{\DiscrParamSet}{\Phi}
\newcommand{\discrparam}{\phi}
\DeclareMathOperator{\DeepVariation}{h}
\newcommand{\variationparam}{\phi}
\makeatletter
\newcommand{\subalign}[1]{%
  \vcenter{%
    \Let@ \restore@math@cr \default@tag
    \baselineskip\fontdimen10 \scriptfont\tw@
    \advance\baselineskip\fontdimen12 \scriptfont\tw@
    \lineskip\thr@@\fontdimen8 \scriptfont\thr@@
    \lineskiplimit\lineskip
    \ialign{\hfil$\m@th\scriptstyle##$&$\m@th\scriptstyle{}##$\crcr
      #1\crcr
    }%
  }
}
\makeatother

\usepackage{paperacronyms}
\acuse{MRI,ADMM,GPU,CPU}
\crefname{equation}{}{}
\Crefname{equation}{}{}
\crefname{item}{}{}
\Crefname{item}{}{}
\hypersetup{%
	bookmarks=true,          
	pdfmenubar=true,	        
	pdffitwindow=false,     
	pdfstartview={FitH},    
	colorlinks=true,        
	linkcolor=red,          
	citecolor=red,          
	filecolor=magenta,      
	urlcolor=cyan,          
	pdfborder = {0,0,0}
}
\addunit{\pixel}{pixel}
\addunit{\pixels}{pixels}
\addunit{\voxel}{voxel}
\addunit{\decibel}{dB}
\addunit{\byte}{B}
\addunit{\hounsfield}{HU}

\title{\Huge Deep Bayesian Inversion}
\subtitle{Computational uncertainty quantification for large scale inverse problems}
\author{
	\textbf{Jonas Adler} \\
	\normalfont{\normalsize Department of Mathematics} \\[-0.35em]
	\normalfont{\normalsize KTH - Royal institute of Technology} \\[-0.35em]
	\normalfont{\normalsize \texttt{jonasadl@kth.se}} \\
	\normalfont{\normalsize Research and Physics, Elekta}  
	\and
	\textbf{Ozan \"{O}ktem} \\ 
	\normalfont{\normalsize Department of Mathematics} \\[-0.35em]
	\normalfont{\normalsize KTH - Royal institute of Technology} \\[-0.35em]
	\normalfont{\normalsize \texttt{ozan@kth.se}}
}

\begin{document}
\maketitle

\begin{abstract}
Characterizing statistical properties of solutions of inverse problems is essential for decision making.
%
Bayesian inversion offers a tractable framework for this purpose, but current approaches are computationally unfeasible for most realistic imaging applications in the clinic.
We introduce two novel deep learning based methods for solving large-scale inverse problems using Bayesian inversion: a sampling based method using a \acl{WGAN} with a novel mini-discriminator and a direct approach that trains a neural network using a novel loss function.
The performance of both methods is demonstrated on image reconstruction in ultra low dose 3D helical CT.
We compute the posterior mean and standard deviation of the 3D images followed by a hypothesis test to assess whether a ``dark spot'' in the liver of a cancer stricken patient is present.
Both methods are computationally efficient and our evaluation shows very promising performance that clearly supports the claim that Bayesian inversion is usable for 3D imaging in time critical applications.
\end{abstract}
\acresetall

\section{Introduction}
In several areas of science and industry there is a need to reliably recover a hidden multidimensional model parameter from noisy indirect observations.
A typical example is when imaging/sensing technologies are used in medicine, engineering, astronomy, and geophysics.

These inverse problems are often ill-posed, meaning that small errors in data may lead to large errors in the model parameter and there are several possible model parameter values that are consistent with observations.
Addressing ill-posedness is critical in applications where decision making is based on the recovered model parameter, like in image guided medical diagnostics. 
Furthermore, many highly relevant inverse problems are large-scale; they involve large amounts of data and high-dimensional model parameter spaces.

\paragraph{Bayesian inversion}
Bayesian inversion is a framework for assigning probabilities to a model parameter given data (posterior) by combining a \emph{data model} with a \emph{prior model} (\cref{sec:InvProb}). 
The former describes how measured data is generated from a model parameter whereas the latter accounts for information about the unknown model parameter that is known beforehand.
Exploring the posterior not only allows for recovering the model parameter in a reliable manner by computing suitable estimators, it also opens up for a complete statistical analysis including quantification of the uncertainty.

A key part of Bayesian inversion is to express the posterior using Bayes' theorem, which in turn requires access to the data likelihood, a prior, and a probability measure for data. 
The data likelihood is often given from insight into the physics of how data is generated (simulator). 
The choice of prior (\cref{sec:PriorModel}) is less obvious but important since it accounts for a priori information about the true model parameter. 
It is also very difficult to specify a probability distribution for data, which is required by many estimators.
Finally, the computational burden associated with exploring the posterior (\cref{sec:CompBayes}) prevents usage of Bayesian inversion in most imaging applications. 

To exemplify the above, consider clinical 3D \ac{CT} imaging where the model parameter represents the interior anatomy and data is x-ray radiographs taken from various directions.
A natural prior in this context is that the object (model parameter) being imaged is a human being, but explicitly handcrafting such a prior is yet to be done.
Instead, current priors prescribe roughness or sparsity, which suppresses unwanted oscillatory behavior at the expense of finer details.
Next, the model parameter is typically $512^3$-dimensional and data is of at least same order of magnitude.
Hence, exploring the posterior in a timely manner is challenging, e.g., uncertainty quantification in Bayesian inversion remains intractable for such large-scale inverse problems.



\section{Statistical Approach to Inverse Problems}\label{sec:InvProb}
Uncertainty refers in general to the accuracy by which one can determine a model parameter.
In an inverse problems, this rests upon the ability to explore the statistical distribution of model parameters given measured data.
More precisely, the posterior probability of the model parameter conditioned on observed data describes all possible solutions to the inverse problem along with their probabilities \cite{Evans:2002aa,Dashti:2016aa} and it is essential for uncertainty quantification.

Bayesian inversion uses Bayes' theorem \cite[Theorem~14]{Dashti:2016aa} to characterize the posterior:
\[
  p(\signal \mid \data)
  =
  \frac{p(\signal)	p(\data \mid \signal)}{p(\data)}.
\]
Here, $p(\data \mid \signal)$ is given by the data model that is usually derived from knowledge about how data is generated and $p(\signal)$ is given by the prior model that represents information known beforehand about the true (unknown) model parameter.

A tractable property of Bayesian inversion is that small changes in data lead to small changes in the posterior even when the inverse problem is ill-posed in the classical sense \cite[Theorem~16]{Dashti:2016aa}, so Bayesian inversion is stable.
Different reconstructions can be obtained by computing different estimators from the posterior and there is also a natural framework for uncertainty quantification, e.g., by computing Bayesian credible sets.

The posterior is however quite complicated with no closed form expression, so much of the contemporary research focuses on realizing the aforementioned advantages with Bayesian inversion without requiring access to the full posterior, see \cite{Dashti:2016aa} for a nice survey. 
Some related key challenges were mentioned earlier in the introduction; choosing a ``good'' \emph{prior}, specifying the \emph{probability distribution of data}, and to explore the posterior in a  \emph{computationally feasible} manner.

\subsection{Choosing a prior model}\label{sec:PriorModel}
The difficulty in selecting a prior model lies in capturing the relevant a priori information. 
Bayesian non-parametric theory \cite{Ghosal:2017ab} provides a large class of handcrafted priors, but these only capture a fraction of the a priori information that is available.
\Cref{fig:analytic_priors} illustrates this by showing random samples generated from priors commonly used by state-of-the-art approaches in image recovery \cite{Kaipio:2005aa,Calvetti:2017aa} as well as samples from typical clinical \ac{CT} images. The handcrafted priors primarily encode regularity properties, like roughness or sparsity, and it would clearly be stretching our imagination to claim that corresponding samples represent natural images.

\begin{figure*}[t]
	\centering	
	\newlength{\tmpfigwidth}
	\setlength{\tmpfigwidth}{0.155\linewidth}
	
	\begin{subfigure}[t]{\tmpfigwidth}
		\centering
		\includegraphics[width=\linewidth]{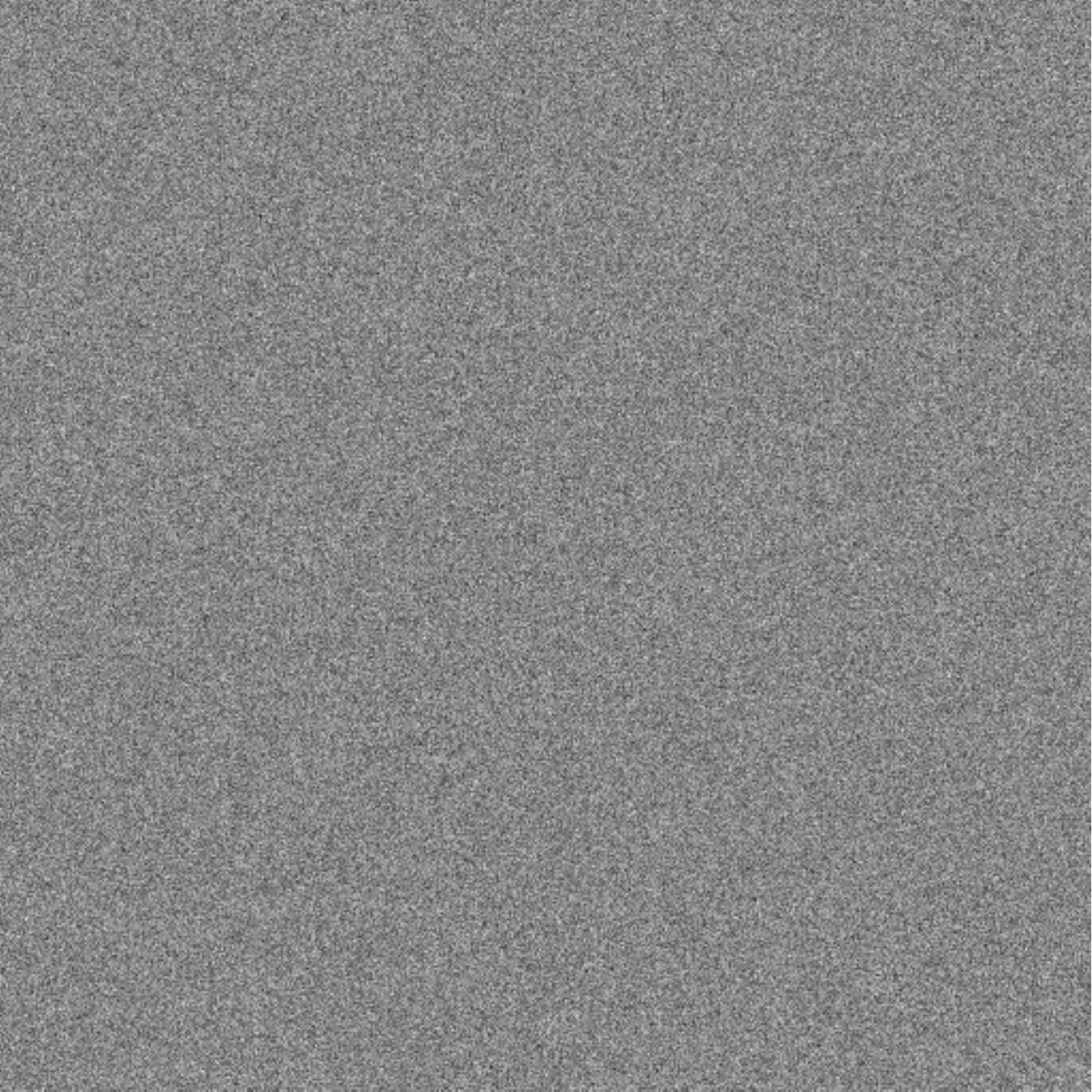}
		\caption*{$\| \signal \|_2^2$}
	\end{subfigure} %
	\hfill %
	\begin{subfigure}[t]{\tmpfigwidth}
		\centering
		\includegraphics[width=\linewidth]{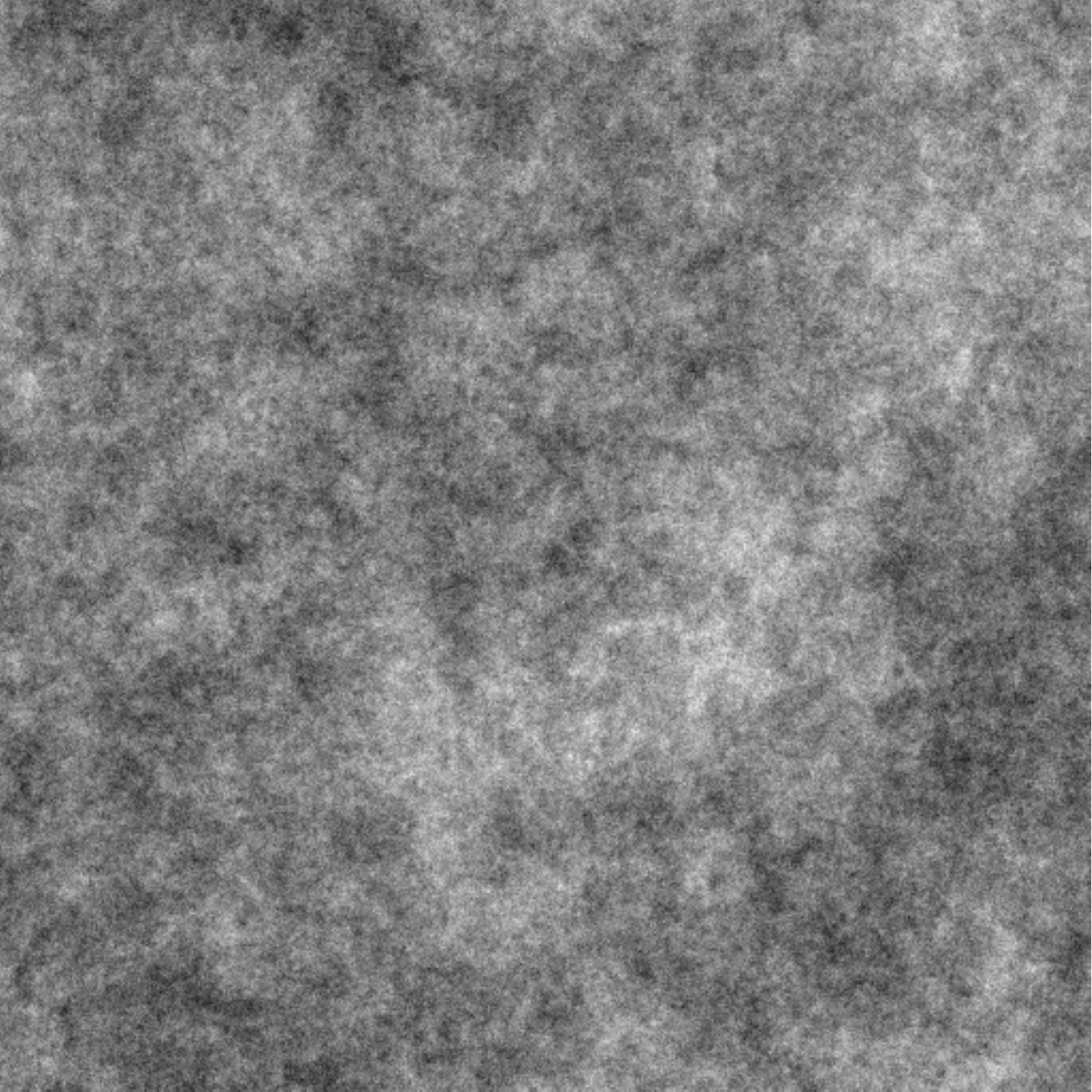}
		\caption*{$\|\nabla \signal \|_2^2$}
	\end{subfigure} %
	\hfill %
	\begin{subfigure}[t]{\tmpfigwidth}
		\centering
		\includegraphics[width=\linewidth]{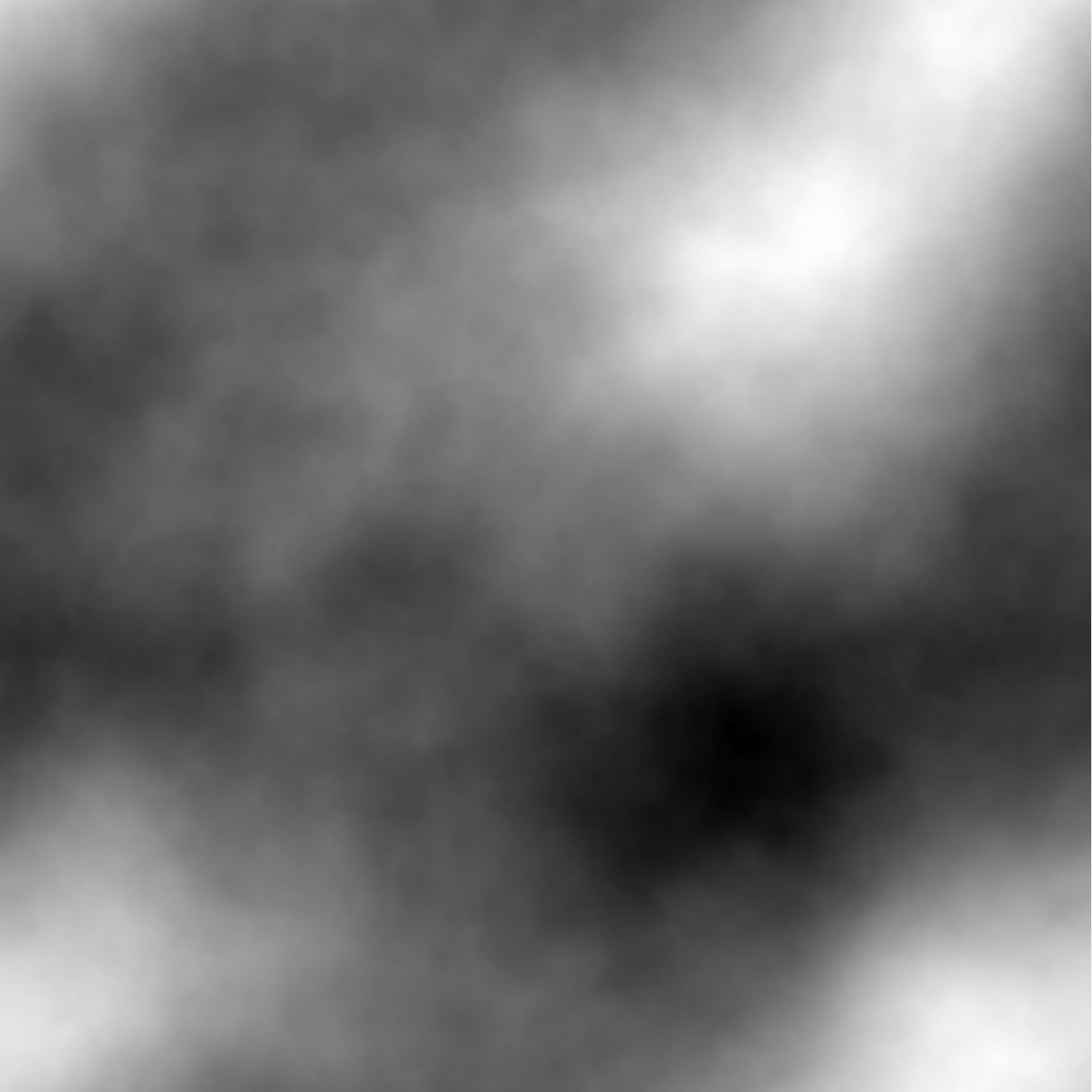}
		\caption*{$\|\Delta \signal \|_2^2$}
	\end{subfigure} %
	\hfill %
	\begin{subfigure}[t]{\tmpfigwidth}
		\centering
		\includegraphics[width=\linewidth]{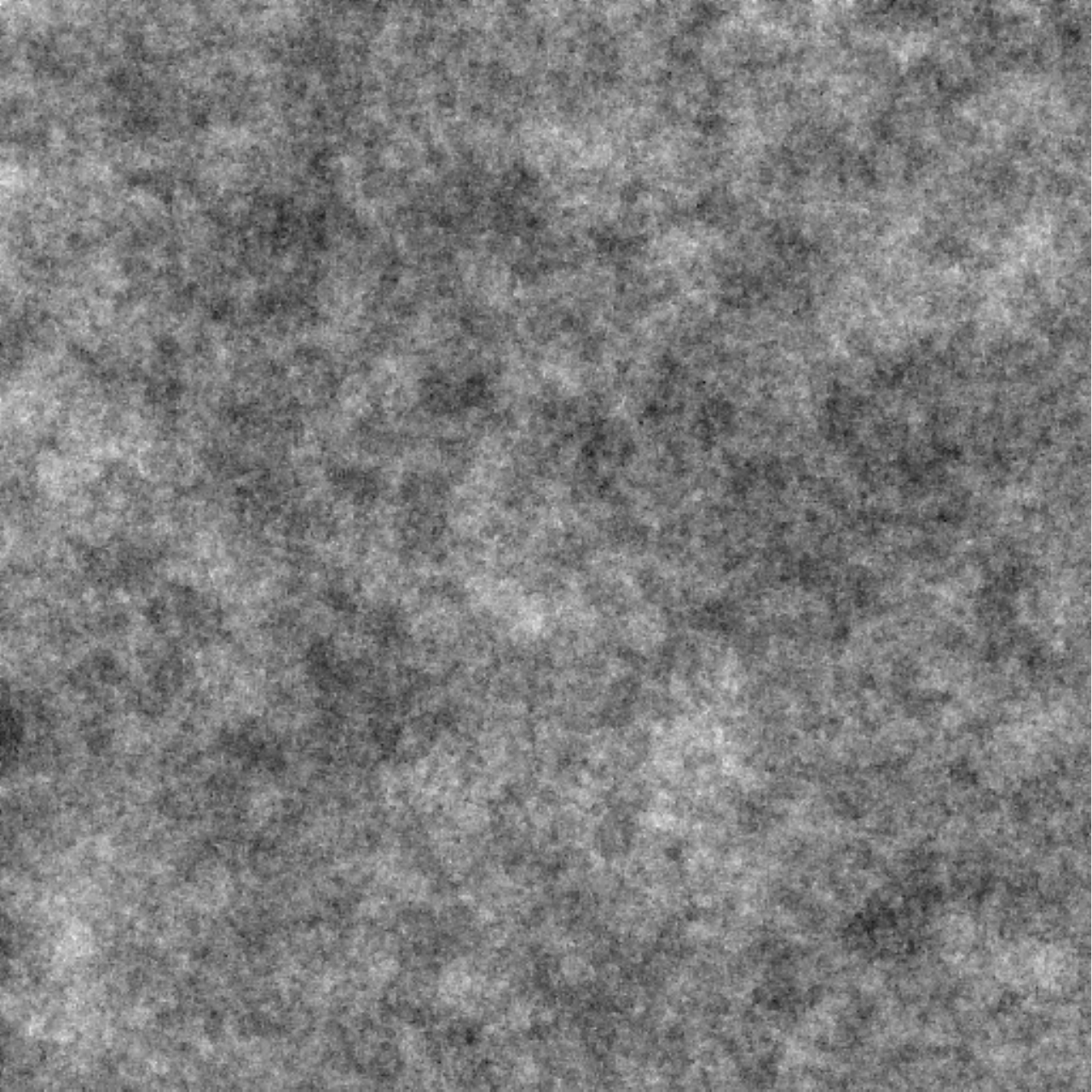}
		\caption*{$\|\nabla \signal \|_1$}
	\end{subfigure} %
	\hfill %
	\begin{subfigure}[t]{\tmpfigwidth}
		\centering
		\includegraphics[width=\linewidth]{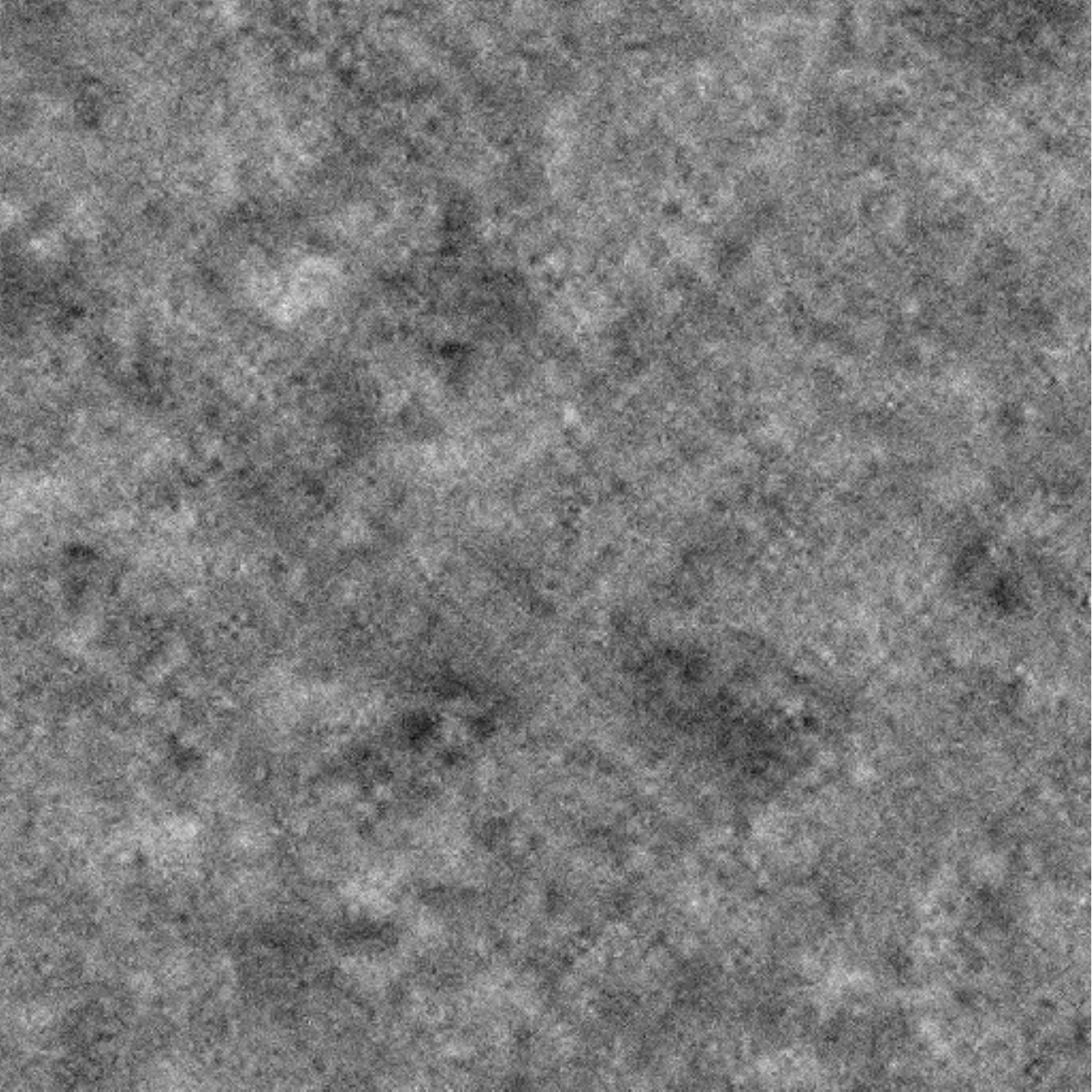}
		\caption*{$\| \signal \|_{B_{1,1}^1}$}
	\end{subfigure} %
	\hfill %
	\begin{subfigure}[t]{\tmpfigwidth}
		\centering
		\includegraphics[width=\linewidth]{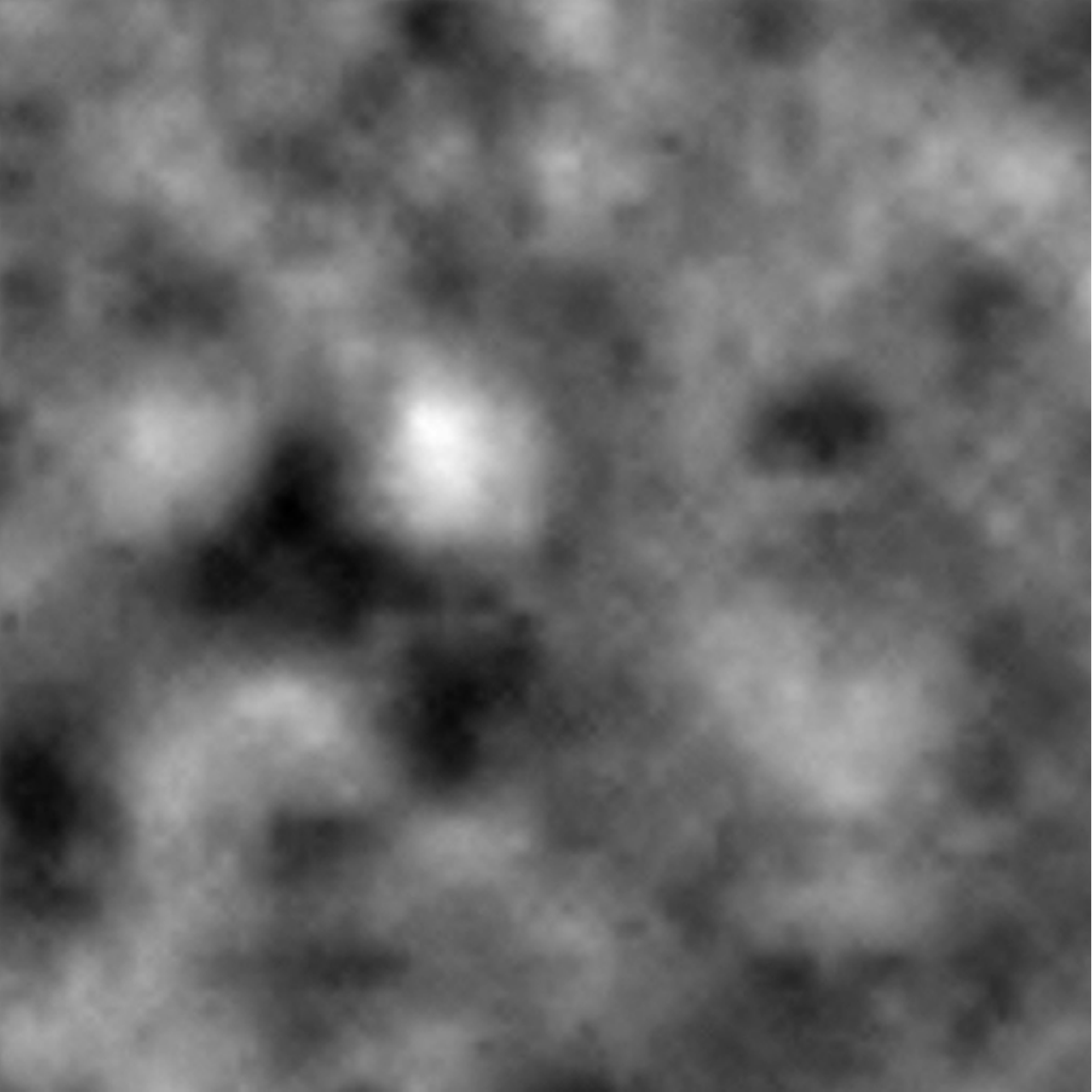}
		\caption*{$\| \signal \|_{B_{1,1}^2}$}
	\end{subfigure}%
	\\[1em]
	\begin{subfigure}[t]{0.155\linewidth}
		\centering	
		\includegraphics[width=\linewidth]{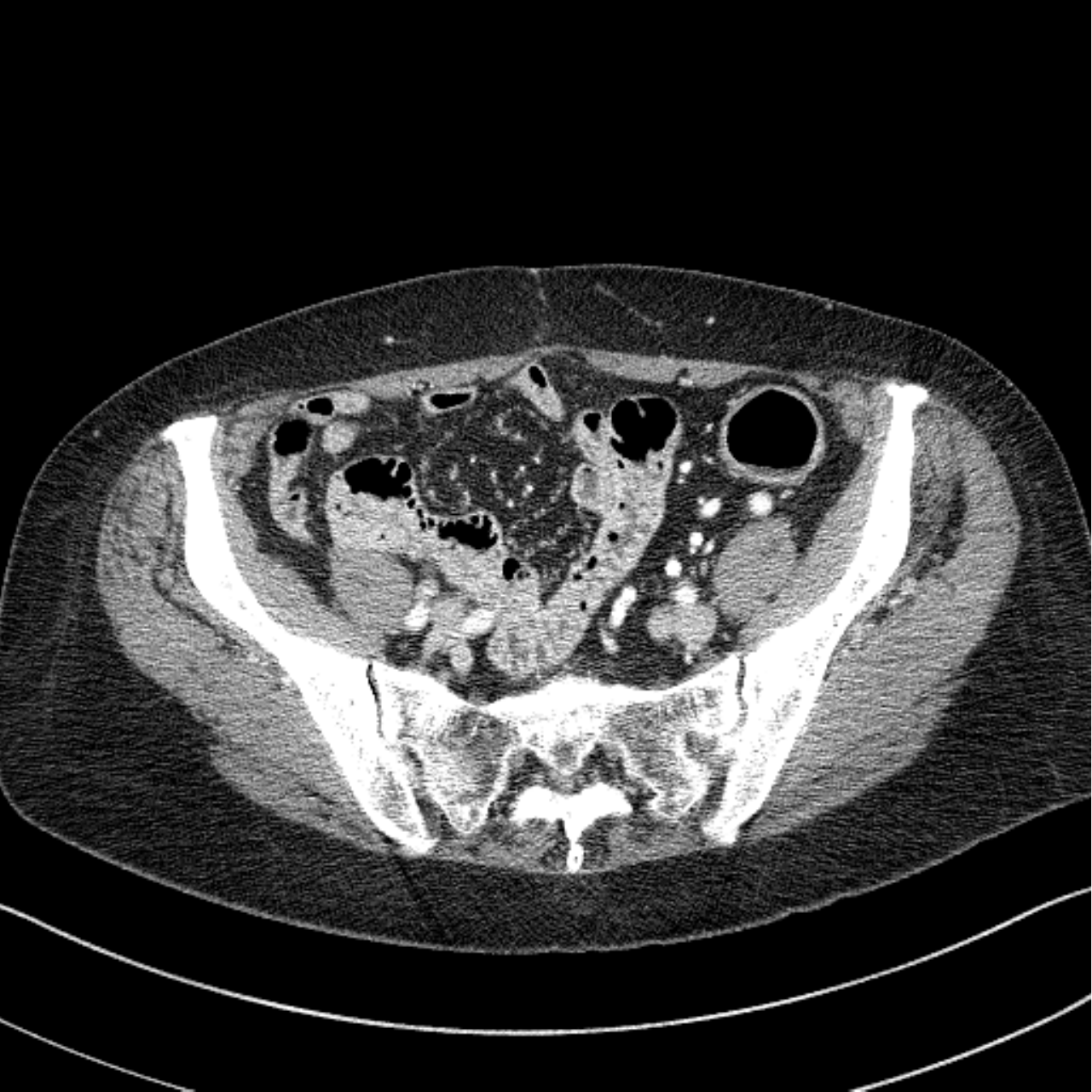}
	\end{subfigure}%
	\hfill
	\begin{subfigure}[t]{0.155\linewidth}
		\centering	
		\includegraphics[width=\linewidth]{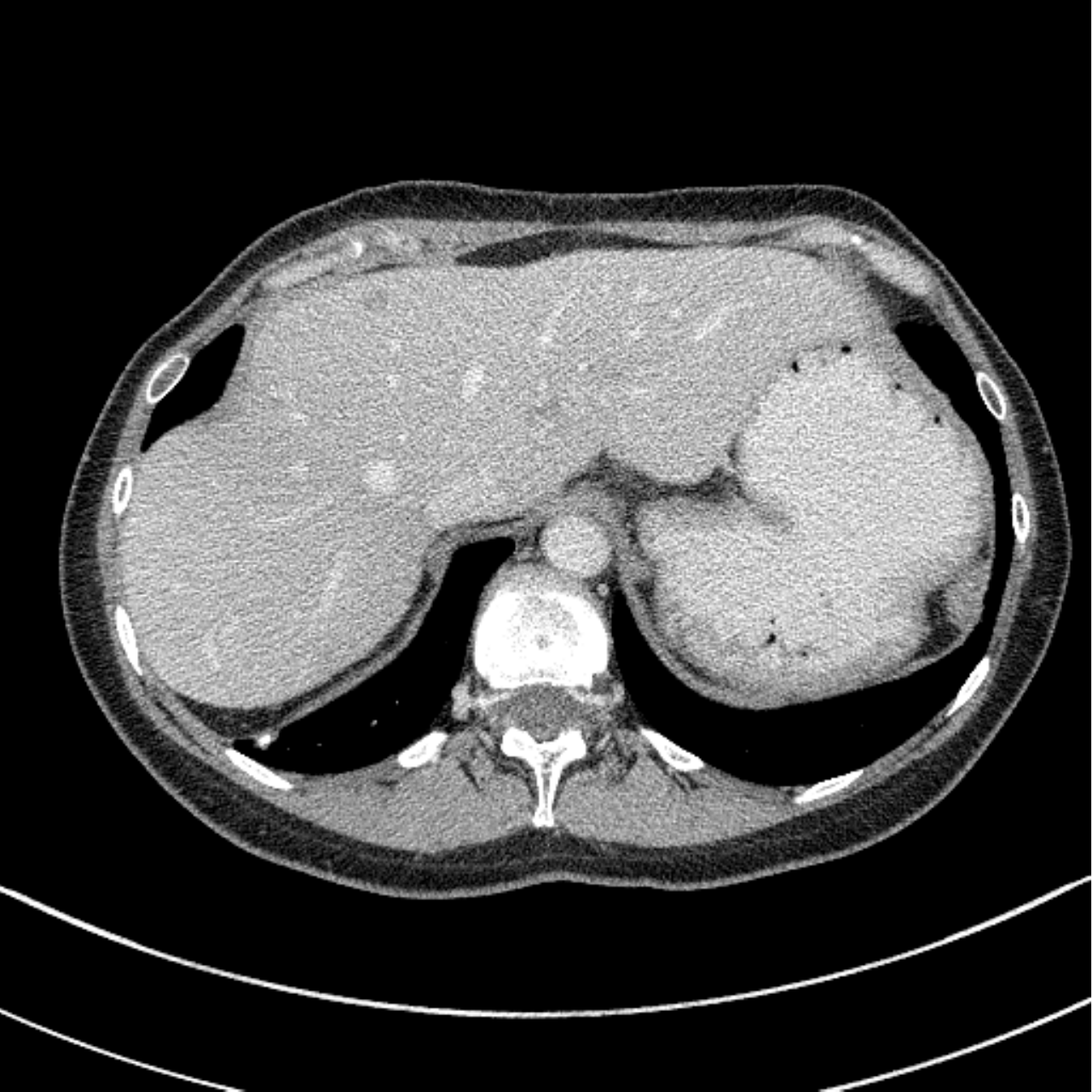}
	\end{subfigure}%
	\hfill
	\begin{subfigure}[t]{0.155\linewidth}
		\centering	
		\includegraphics[width=\linewidth]{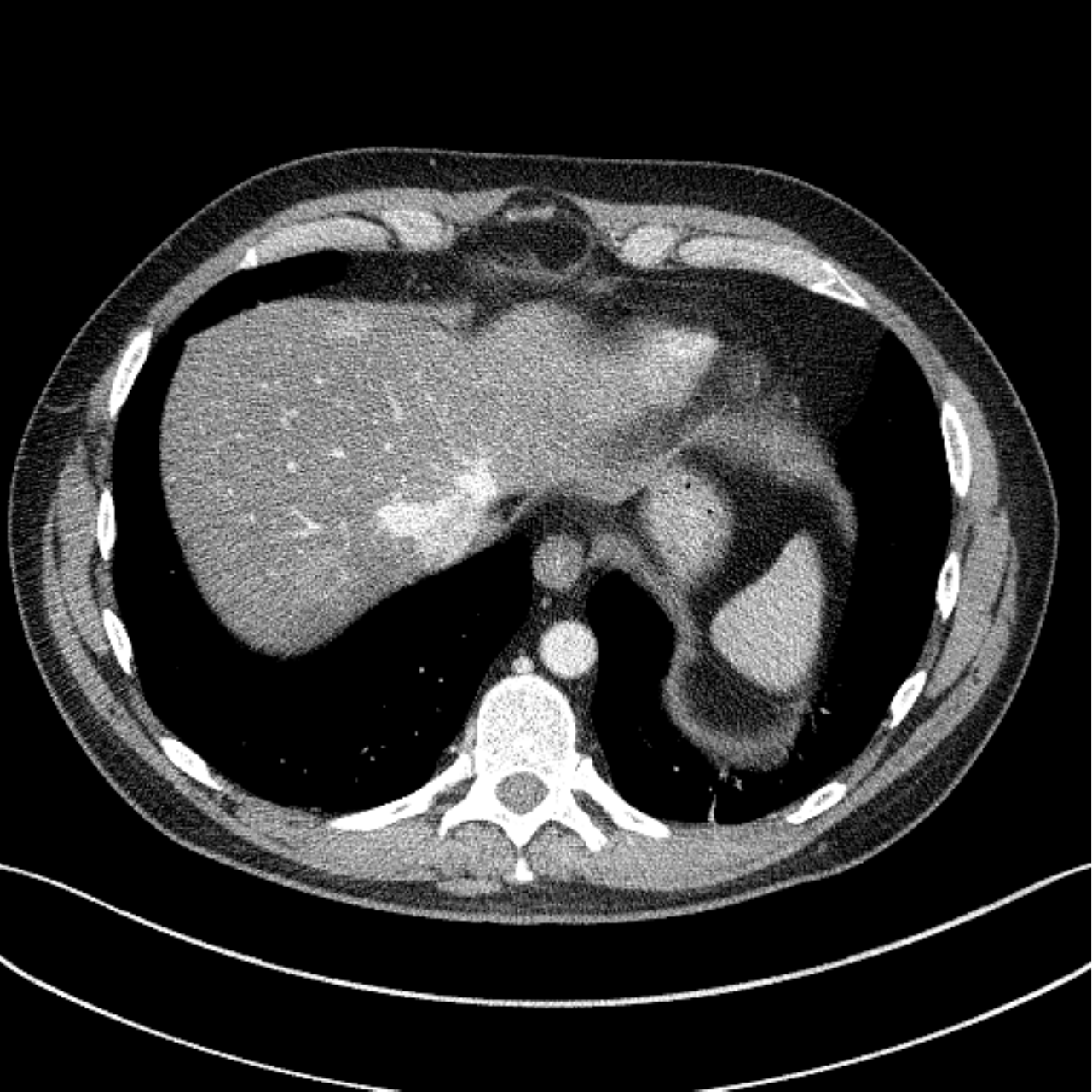}
	\end{subfigure}%
	\hfill
	\begin{subfigure}[t]{0.155\linewidth}
		\centering	
		\includegraphics[width=\linewidth]{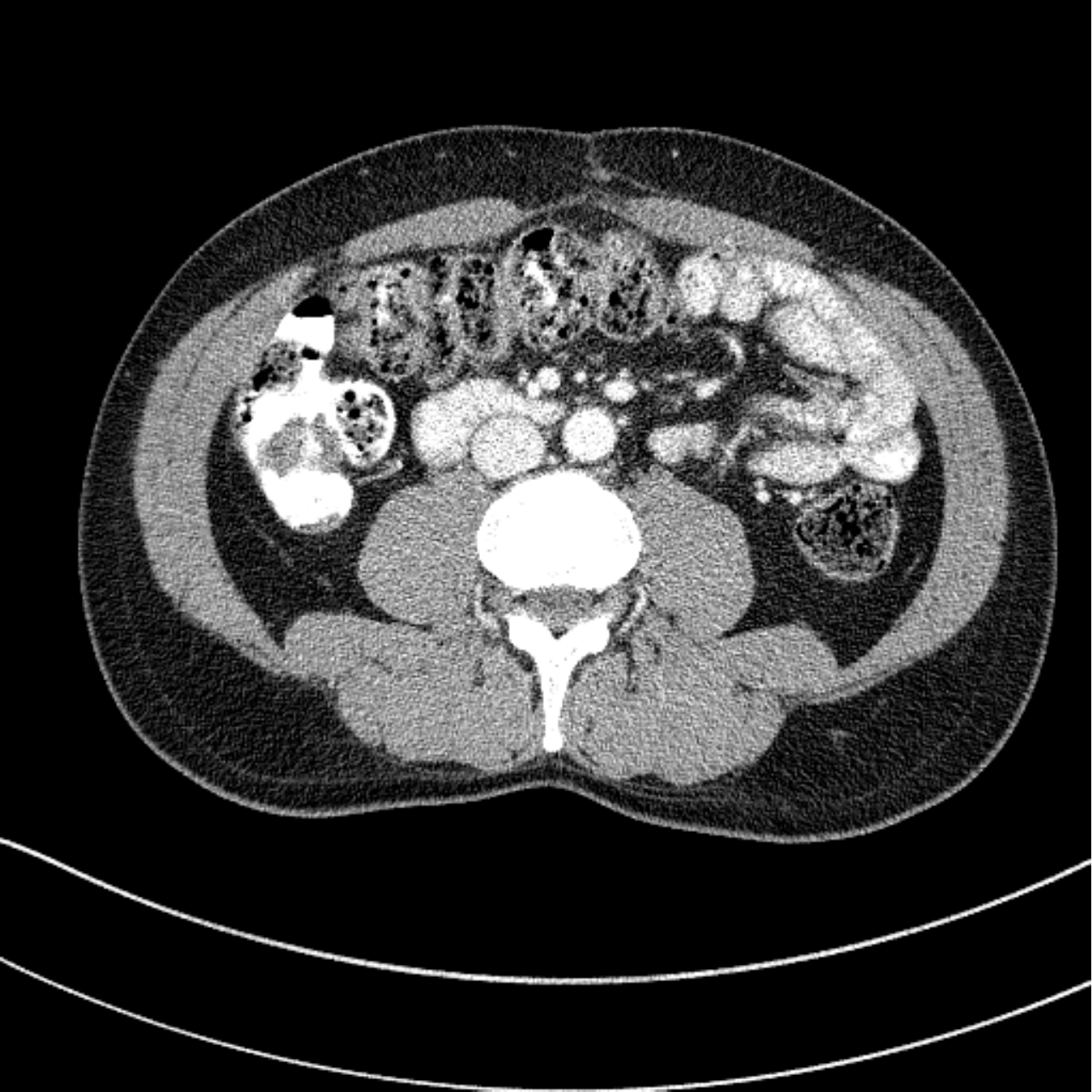}
	\end{subfigure}%
	\hfill
	\begin{subfigure}[t]{0.155\linewidth}
		\centering	
		\includegraphics[width=\linewidth]{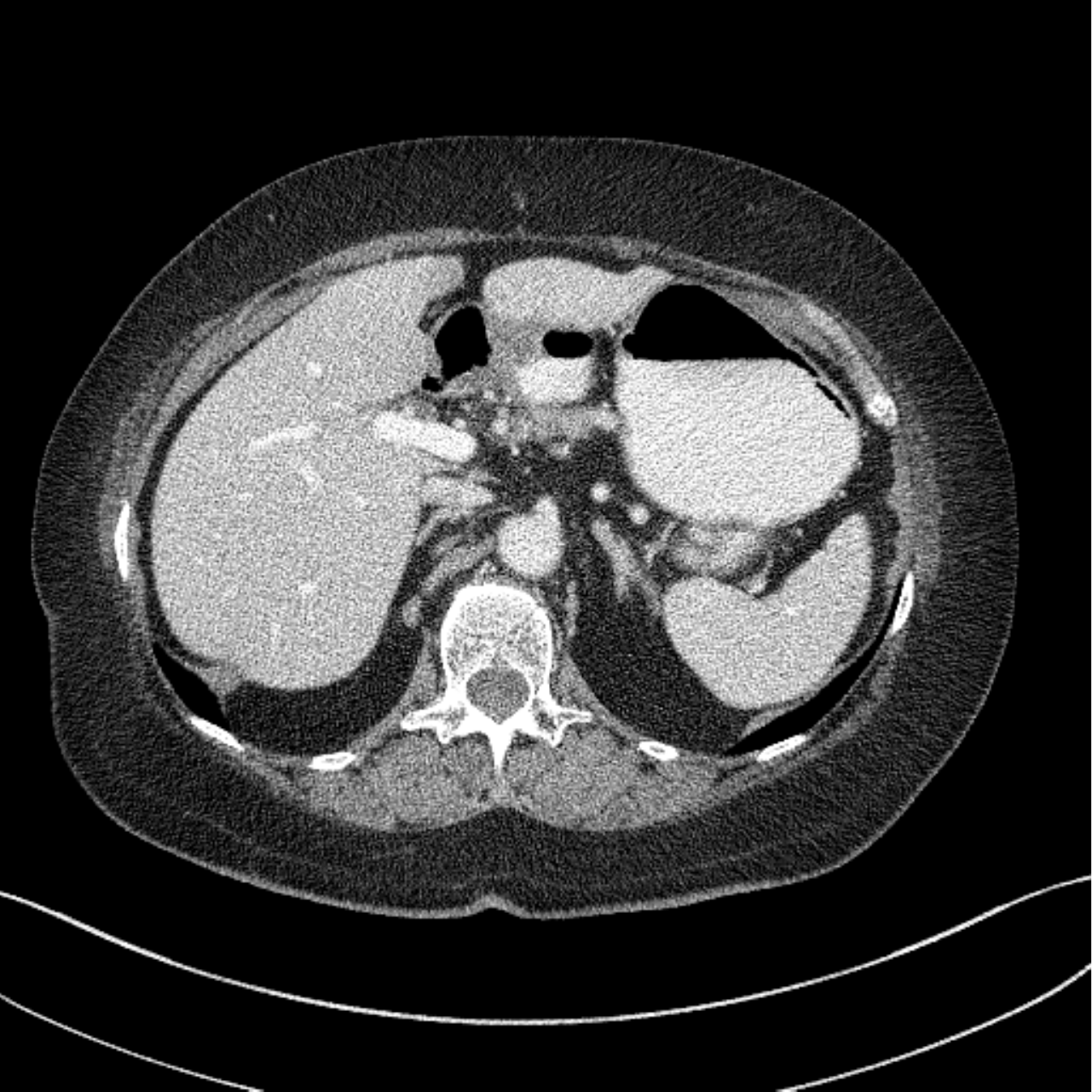}
	\end{subfigure}%
	\hfill
	\begin{subfigure}[t]{0.155\linewidth}
		\centering	
		\includegraphics[width=\linewidth]{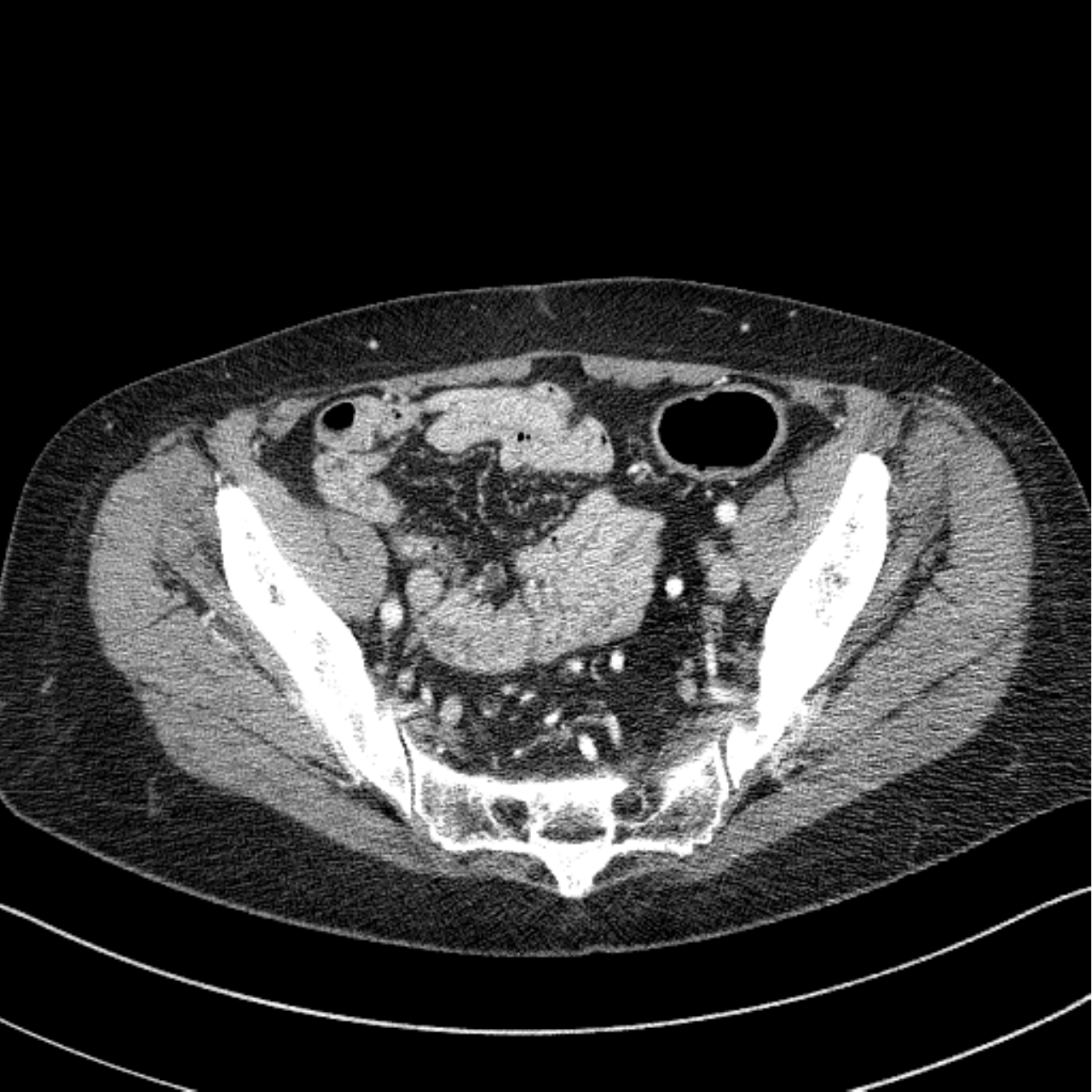}
	\end{subfigure}%
	\caption{Top row shows a single random sample generated by a Gibbs type of roughness priors that are common in inverse problems in imaging (\cref{sec:appendix_analytical_priors}).
	Such a prior is proportional to $e^{-S(\signal)}$ and images show samples for different choices of $S$.
 	Bottom row shows typical samples of normal dose \ac{CT} images 
	  of humans. Ideally, a prior generates samples similar to those in the bottom row.}
	\label{fig:analytic_priors}
\end{figure*}

\subsection{Computational feasibility}\label{sec:CompBayes}
Exploring the posterior for inverse problems in imaging often leads to large scale numerics since this mounts to sampling from a high dimensional probability distribution. 
Most approaches, see \cref{sec:RelWork}, are either not fast enough or rely on simplifying assumptions that does not hold in many applications
For the above reasons, in large scale inverse problems one tends to reconstruct a \emph{single point estimate} of the posterior distribution, the most common being the \ac{MAP} estimator that corresponds to the most likely reconstruction given the data.
A drawback that comes with working with single estimators is that these cannot include all the information present in the posterior distribution. It is clear that knowledge about the full posterior would have dramatic impact upon how solutions to inverse problems are intertwined into decision making. As an example, in medical imaging, practitioners would be able to compute the probability of a tumor being an image artifact, which in turn is necessary for image guided hypothesis testing.


\section{Contribution}\label{sec:Contrib}
Our overall contribution is to suggest two generic, yet adaptable, frameworks for uncertainty quantification in inverse problems that are computationally feasible and where both the prior and probability distribution of data are given implicitly through supervised examples instead of being handcrafted. 
The approach is based on recent advances in \acp{GAN} from deep learning and we demonstrate its performance on ultra low dose 3D helical \ac{CT}.


Our main contribution is \emph{Deep Posterior Sampling} (\cref{sec:WGAN}) where generative models from machine learning are used to sample from a high-dimensional unknown posterior distribution in the context of Bayesian inversion. This is made possible by a novel \emph{conditional \ac{WGAN} discriminator} (\cref{sec:minibatchDiscriminator}). 
The approach is generic and applies in principle to any inverse problem assuming there is relevant training data. 
It can be used for performing statistical analysis of the posterior on $\RecSpace$, e.g., by computing various estimators. 

Independently, we also introduce \emph{Deep Direct Estimation} (\cref{sec:DirectUC}) where one directly computes an estimator using an deep neural network trained using a cleverly chosen loss (\cref{app:sampling_free}). Deep direct estimation is faster than posterior sampling, but it mainly applies to statistical analysis that is based on evaluating a pre-determined estimator. Both approaches should give similar quantitative results when used for evaluating the same estimator.

We demonstrate the performance and computational feasibility for ultra low dose \ac{CT} imaging in a clinical setting by computing some estimators and performing a hypothesis test (\cref{sec:NumericalExp}).

\section{Deep Bayesian Inversion}\label{sec:DeepBayesInverse}
As already stated, in Bayesian inversion both the model parameter $\signal$ and measured data $\data$ are assumed to be generated by random variables $\stsignal$ and $\stdata$, respectively. 
The ultimate goal is to recover the posterior $\posterior{\data}$, which describes all possible solutions $\stsignal=\signal$ along with their probabilities given data $\stdata=\data$.

We here outline two approaches that can be used to perform various statistical analysis on the posterior. 
Deep Posterior Sampling is a technique for learning how to sample from the posterior whereas Deep Direct Estimation learns various estimators directly.

\subsection{Deep Posterior Sampling}\label{sec:WGAN}
The idea is to explore the posterior by sampling from a generator that is defined by a \ac{WGAN}, which has been trained using a conditional \ac{WGAN} discriminator. 

To describe how a \ac{WGAN} can be used for this purpose, let data $\data \in \DataSpace$ be fixed and assume that $\posterior{\data}$, the posterior of $\stsignal$ at $\stdata=\data$, can be approximated by elements in a parametrized family $\{ \GenProb_{\genparam}(\data) \}_{\genparam \in \GenParamSet}$ of probability measures on $\RecSpace$. 
The best such approximation is defined as $\GenProb_{\genparam^*}(\data)$ where $\genparam^* \in \GenParamSet$ solves
\begin{equation}\label{eq:GenericFormulationPointWise}
  \genparam^* \in \argmin_{\genparam \in \GenParamSet} \,
      \ProbDist \bigl(\GenProb_{\genparam}(\data), \posterior{\data} \bigr).
\end{equation}
Here, $\ProbDist$ quantifies the ``distance'' between two probability measures on $\RecSpace$. 
We are however interested in the best approximation for ``all data'', so we extend \cref{eq:GenericFormulationPointWise} by including an averaging over all possible data. 
The next step is to choose a distance notion $\ProbDist$ that desirable from both a theoretical and a computational point of view.
As an example, the distance should be finite and computational feasibility requires using it to be differentiable almost everywhere, since this opens up for using \ac{SGD} type of schemes. 
The Wasserstein 1-distance $\Wasserstein$ (\cref{sec:Wasserstein}) has these properties \cite{Arjovsky:2017aa} and sampling from the posterior $\posterior{\data}$ can then be replaced by sampling from the probability distribution $\GenProb_{\genparam^*}(\data)$ where $\genparam^*$ solves
\begin{equation}\label{eq:WGANFormulation}
  \genparam^* \in \argmin_{\genparam \in \GenParamSet} \,
    \Expect_{\stdata \sim \dataprob} \Bigl[
          \Wasserstein \bigl( \GenProb_{\genparam}(\stdata), \posterior{\stdata} \bigr)
    \Bigr].
\end{equation}
In the above, $\dataprob$ is the probability distribution for data and $\stdata \sim \dataprob$ generates data.

Observe now that evaluating the objective in \cref{eq:WGANFormulation} requires access to the very posterior that we seek to approximate.
Furthermore, the distribution $\dataprob$ of data is often unknown, so an approach based on \cref{eq:WGANFormulation} is essentially useless if the purpose is to sample from an unknown posterior. 
Finally, evaluating the Wasserstein 1-distance directly from its definition is not computationally feasible. 

On the other hand, as shown in \cref{app:WGANMath}, \emph{all} of these drawbacks can be circumvented by rewriting \cref{eq:WGANFormulation} as an expectation over the joint law $(\stsignal, \stdata)\sim \jointlaw$.
This makes use of specific properties of the Wasserstein 1-distance (Kantorovich-Rubenstein duality) and one obtains the following approximate version of \cref{eq:WGANFormulation}:
\begin{equation}\label{eq:FormulationWithZ}
  \genparam^* \in \argmin_{\genparam \in \GenParamSet}
  \Biggl\{
  \sup_{\discrparam \in \DiscrParamSet}
  \Expect_{(\stsignal, \stdata)\sim \jointlaw} 
  \biggl[
      \Discriminator_{\discrparam}(\stsignal, \stdata) - \Expect_{\stgenvar \sim \genvarprob}\bigl[ \Discriminator_{\discrparam}(\Generator_{\genparam}(\stgenvar, \stdata), \stdata) \bigr]
  \biggr]
  \Biggr\}.  
\end{equation}
In the above, $\Generator_{\genparam} \colon \GenSpace \times \DataSpace \to \RecSpace$ (generator) is a deterministic mapping such that $\Generator_{\genparam}(\stgenvar, \data) \sim \GenProb_{\genparam}(\data)$, where $\stgenvar \sim \genvarprob$ is a `simple' $\GenSpace$-valued random variable in the sense that it can be sampled in a computationally feasible manner.
Next, the mapping $\Discriminator_{\discrparam} \colon \RecSpace \times \DataSpace \to \Real$ (discriminator) is a measurable mapping that is $1$-Lipschitz in the $\RecSpace$-variable.

On a first sight, it might be unclear why \cref{eq:FormulationWithZ} is better than \cref{eq:WGANFormulation} if the aim is to sample from the posterior, especially since the joint law $\jointlaw$ in \cref{eq:FormulationWithZ} is unknown. 
The advantage becomes clear when one has access to supervised training data for the inverse problem, i.e., i.i.d. samples $(\signal_1,\data_1), \ldots, (\signal_m,\data_m)$ generated by the random variable $(\stsignal,\stdata) \sim \jointlaw$. 
The $\jointlaw$-expectation in \cref{eq:FormulationWithZ} can then be replaced by an averaging over training data.

To summarize, solving \cref{eq:FormulationWithZ} given supervised training data in $\RecSpace \times \DataSpace$ amounts to learning a generator $\Generator_{\genparam^*}(\stgenvar,\cdot) \colon \DataSpace \to \RecSpace$ such that $\Generator_{\genparam^*}(\stgenvar,\data)$ with $\stgenvar \sim \genvarprob$ is approximately distributed as the posterior $\posterior{\data}$.
In particular, for given $\data \in \DataSpace$ we can sample from $\posterior{\data}$ by generating values of $\genvar \mapsto \Generator_{\genparam^*}(\genvar,\data) \in \RecSpace$ in which $\genvar \in \GenSpace$ is generated by sampling from $\stgenvar \sim \genvarprob$.

An important part of the implementation is the concrete parameterizations of the generator and discriminator:
\[ 
   \Generator_{\genparam} \colon  \GenSpace \times \DataSpace \to \RecSpace 
   \quad\text{and}\quad
   \Discriminator_{\discrparam} \colon \RecSpace \times \DataSpace \to \Real.
\]   
We here use deep neural networks for this purpose and following \cite{WGAN-GP}, we softly enforce the 1-Lipschitz condition on the discriminator by including a gradient penalty term to the training objective in \cref{eq:FormulationWithZ}. Furthermore, if \cref{eq:FormulationWithZ} is implemented as is, then in practice $\stgenvar$ is not used by the generator (so called mode-collapse). To solve this problem, we introduce a novel conditional mini-batch discriminator that can be used with conditional WGAN without impairing upon its analytical properties (\cref{prop:minibatch_discriminator}), see \cref{sec:minibatchDiscriminator} for more details.

\subsection{Deep Direct Estimation}\label{sec:DirectUC}
The idea here is to train a deep neural network to directly approximate an estimator of interest without resorting to generating samples from the posterior as in posterior sampling (\cref{sec:WGAN}).

Deep direct estimation relies on the well known result:
\begin{equation}\label{eq:condexp}
	\Expect_{\stotherdata} \bigl[ \stotherdata \mid \stdata = \cdot \bigr]
	=
	\!\!
	\min_{\DeepVariation \colon \DataSpace \to \OtherSpace}  
	\Expect_{(\stdata, \stotherdata)} 
	\Bigl[ 
		\bigl\|\DeepVariation(\stdata)  - \stotherdata \bigr\|_{\OtherSpace}^2 
	\Bigr].
\end{equation}
In the above, $\stotherdata$ is \emph{any} random variable taking values in some measurable Banach space $\OtherSpace$ and the minimization is over all $\OtherSpace$-valued measurable maps on $\DataSpace$. 
See \cref{prop:condexp} in \cref{app:sampling_free} for a precise statement.
This is useful since many estimators relevant for uncertainty quantification are expressible using terms of this form for appropriate choices of $\stotherdata$.

Specifically, \cref{app:sampling_free} considers two (deep) neural networks $\RecOp_{\recparam^*} \colon \DataSpace \to \RecSpace$ and $\DeepVariation_{\variationparam^*} \colon \DataSpace \to \RecSpace$ with appropriate architectures that are trained according to
\begin{align*}
\recparam^* \in&\
\argmin_{\recparam}
\biggl\{ 
	\Expect_{(\stsignal, \stdata) \sim \jointlaw} 
	\Bigl[\bigl\|
		\stsignal - \RecOp_{\recparam}(\stdata)
	\bigr\|_\RecSpace^2 \Bigr]
\biggr\}
\\
\variationparam^* \in&\
\argmin_{\variationparam}
\biggl\{ 
	\Expect_{(\stsignal, \stdata)\sim \jointlaw} 
	\Bigl[\bigl\|
		\DeepVariation_{\variationparam}(\stdata) - 
		\bigl(
			\stsignal - \RecOp_{\recparam^*}(\stdata)
		\bigr)^2
	\bigr\|_\RecSpace^2 \Bigr]
\biggr\}.
\end{align*}
The resulting networks will then approximate the conditional mean and the conditional point-wise variance, respectively. 
Finally, if one has supervised training data $(\signal_i, \data_i)$, then the joint law $\jointlaw$ above can be replaced by its empirical counterpart and the $\jointlaw$-expectation is replaced by an averaging over training data.

As already indicated, by using \cref{eq:condexp} it is possible to re-write many estimators as minimizers of an expectation.
Such estimators can then be approximated using the direct estimation approach outlined here.
This should coincide with computing the same estimator by posterior sampling (\cref{sec:WGAN}).
Direct estimation is however significantly faster, but not as flexible as posterior sampling since each estimator requires a new neural network that specifically trained for that estimator. 
\Cref{sec:NumericalExp} compares outcome from both approaches.

\section{Numerical Experiments}\label{sec:NumericalExp}
We evaluate the feasibility of posterior sampling (\cref{sec:WGAN}) to sample from the posterior and Direct Estimation (\cref{sec:DirectUC}) to compute mean and point-wise variances for clinical 3D \ac{CT} imaging.

\subsection{Set-up}\label{subsec:setup}
Our supervised data consists of pairs of 3D \ac{CT} images $(\signal_i,\data_i)$ generated by $(\stsignal, \stdata)$ where $\signal_i$ is a normal dose 3D image that serves as the `ground truth' and $\data_i$ is the \ac{FBP} 3D reconstruction computed from ultra low dose \ac{CT} data associated with $\signal_i$.

One could here let $\data_i$ be the ultra low dose \ac{CT} data itself, which results in more complex architectures of the neural networks. 
On the other hand, using \ac{FBP} as a pre-processing step (i.e., $\data_i$ is \ac{FBP} reconstruction from ultra low dose data) simplifies the choice of architectures and poses no limitation in the theoretical setting with infinite data (see \cite[section~8]{Adler:2018ab}).

\begin{figure}[t]
	\centering	
	\includegraphics[height=0.187\textheight]{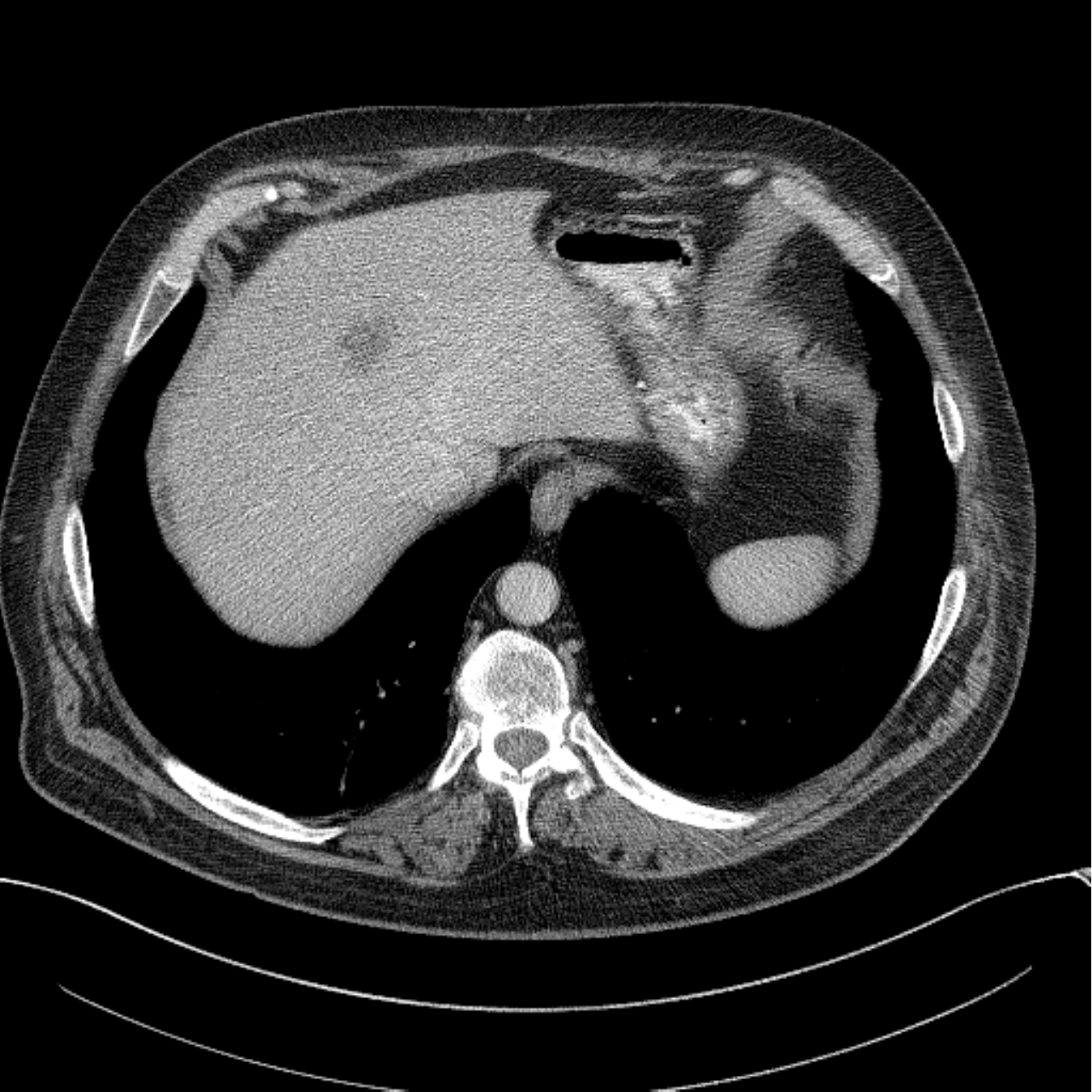}
	\includegraphics[height=0.187\textheight]{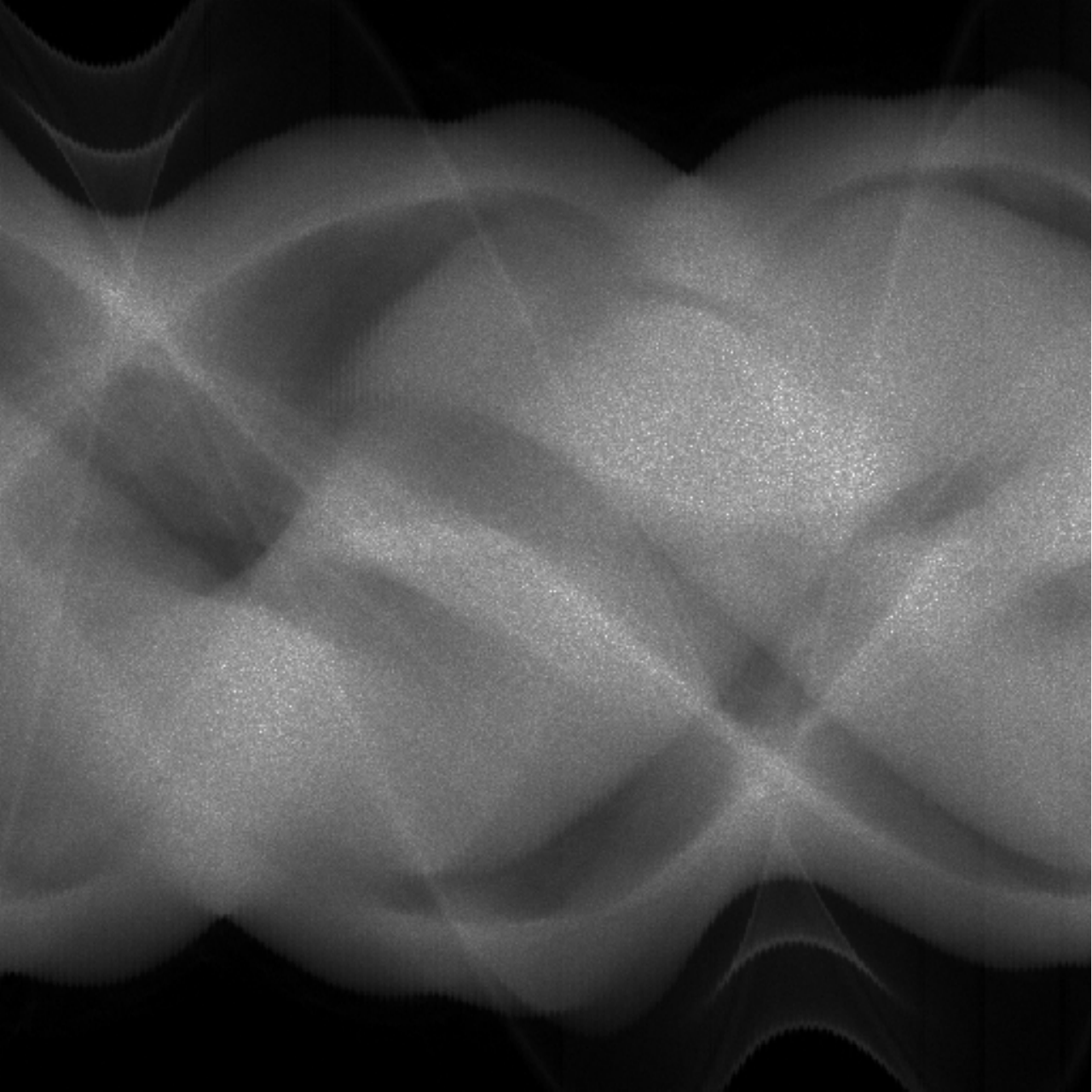}
	\includegraphics[height=0.187\textheight]{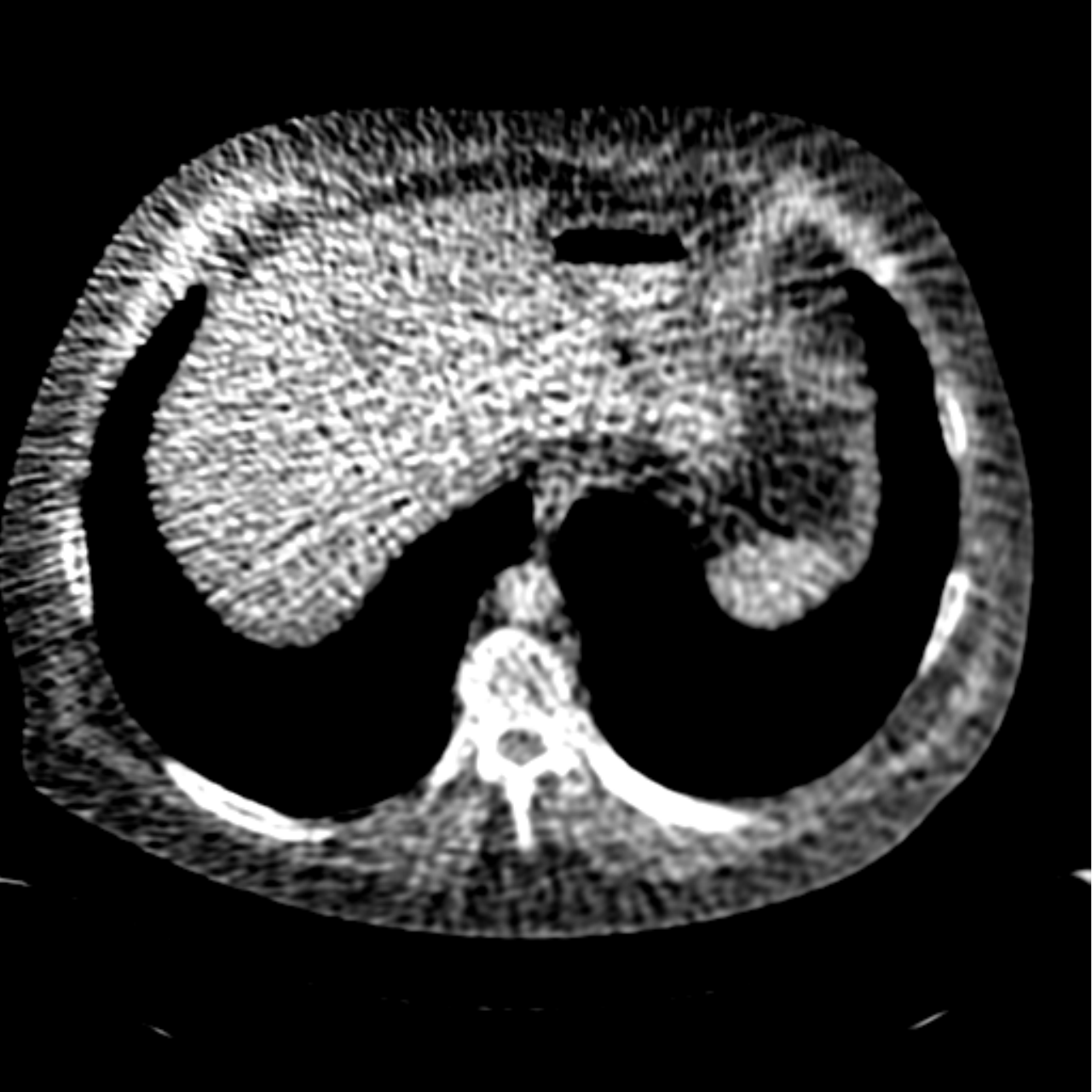}
	\caption{Test data: Normal dose image (left), subset of \ac{CT} data from a ultra-low dose 3D helical scan (middle), and corresponding \ac{FBP} reconstruction (right). Images are shown using a display window set to $[-150, 200]$ \hounsfield.}
	\label{fig:input_data}
\end{figure}

\paragraph{Training data}
We used training data from the Mayo Clinic Low Dose \ac{CT} challenge \cite{AAPMLowDose}. This data consists of ten \ac{CT} scans, of which we use nine for training and one for evaluation. Each 3D image $\signal_i$ has a corresponding ultra low dose data that is generated by using only $\approx 10 \%$ of the full data and adding additional Poisson noise so that the dose corresponds to $2\%$ of normal dose.
Applying \ac{FBP} on this data yields the ultra low dose \ac{CT} images, see \cref{sec:training_data_details} for a detailed description.

An example of normal dose \ac{CT} reconstruction, tomographic data, and the ultra low dose \ac{FBP} reconstruction is shown in \cref{fig:input_data}.

\paragraph{Network architecture and training}
The operators 
\begin{alignat*}{2} 
   \Generator_{\genparam} &\colon  \GenSpace \times \DataSpace \to \RecSpace &\qquad\qquad
   \RecOp_{\recparam^*} &\colon \DataSpace \to \RecSpace 
\\   
   \Discriminator_{\discrparam} &\colon \RecSpace \times \DataSpace \to \Real &\qquad\qquad
   \DeepVariation_{\variationparam^*} &\colon \DataSpace \to \RecSpace
\end{alignat*}
are represented by multi-scale residual neural networks. For computational reasons, we applied the method slice-wise, see \cref{sec:architecture} for details regarding the exact choice of architecture and training procedure.

The parts related to the inverse problem (tomography) were implemented using the ODL framework \cite{Adler:2017ab} with ASTRA \cite{Aarle:2016aa} as back-end for computing the ray-transform and its adjoint. The learning components were implemented in TensorFlow \cite{Abadi:2016aa}.

\subsection{Results}

\paragraph{Estimators}
A typical use-case of Bayesian inversion is to compute estimators from the posterior. In our case, we are interested in the conditional mean and point-wise standard deviation (square root of variance).

When using posterior sampling, we compute the conditional mean and point-wise standard deviations based on 1\,000 images sampled from the posterior, see \cref{sec:individual_samples} for some examples of such images. For direct estimation we simply evaluated the associated trained networks.
Both approaches are computationally feasible, the time needed per slice to compute these estimators is $\unit{40}{\second}$ using posterior sampling based on 1\,000 samples and $\unit{80}{\milli\second}$ for direct estimation.

The mean and standard-deviations that were computed using both methods are shown in \cref{fig:bayesian_results}.
We note that results from the methods agree very well with each other, indicating that the posterior samples follow the posterior quite well, or at least that the methods have similar bias. 
The posterior mean looks as one would expect, with highly smoothed features due to the high noise level. 
Likewise, the standard deviation is also as one would expect, with high uncertainties around the boundaries of the high contrast objects.
We also note that the standard deviation at the white ``blobs'' that appear in some samples (see \cref{sec:individual_samples}) is quite high, indicating that the model is uncertain about their presence.
There is also a background uncertainty at $\approx \unit{20}{\hounsfield}$ due to point-wise noise in the reference normal-dose scans that we take as ground truth.

\begin{figure}[t]
	\centering	
	\begin{tabular}{r c c l l }
		& Posterior sampling & Direct estimation &
		\\
		\raisebox{5em}[0pt][0pt]{\rotatebox[origin=c]{90}{Mean}} \hspace{-4mm} &
		\includegraphics[height=0.187\textheight]{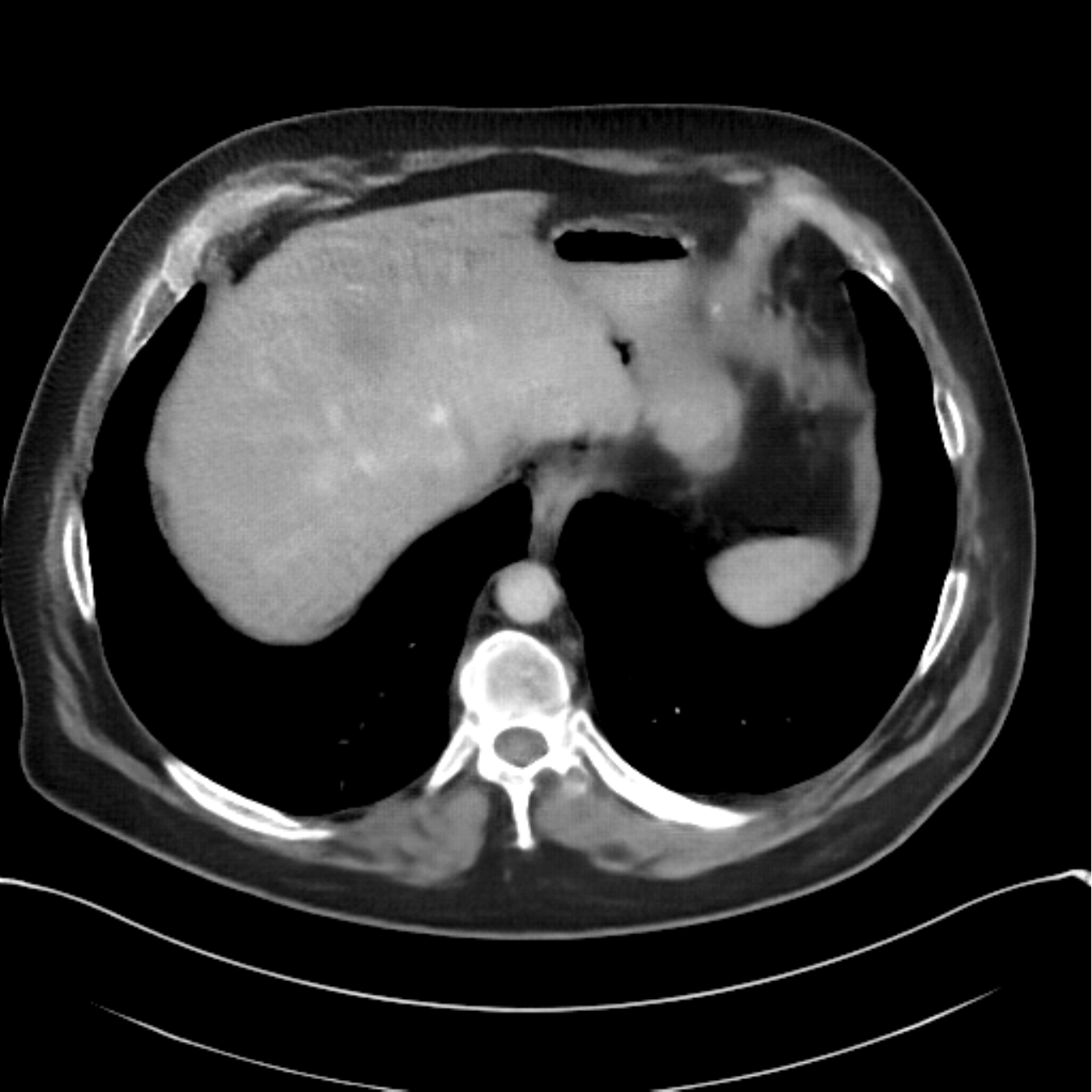}
		\hspace{-4mm} &
		\includegraphics[height=0.187\textheight]{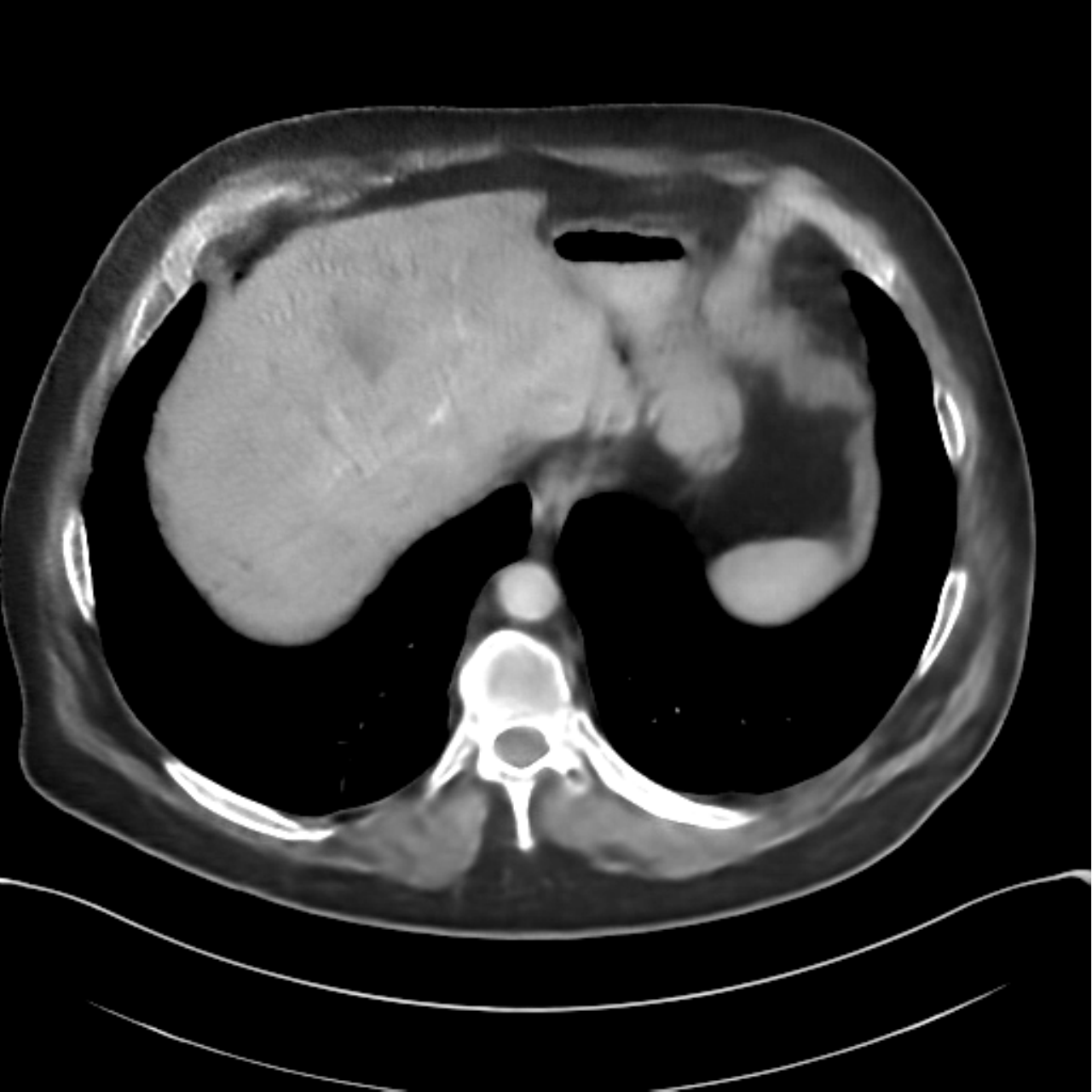}
		\hspace{-4mm} &
		\includegraphics[width=0.02\textwidth, height=0.187\textheight]{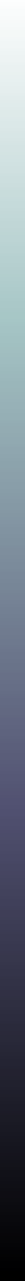} &
		\hspace{-4mm}
		$\overset{\raise0.16\textheight\hbox{\footnotesize\text{200 HU}}}{\footnotesize \text{-150 HU}}$
		\\
		\raisebox{5em}[0pt][0pt]{\rotatebox[origin=c]{90}{pStd}} \hspace{-4mm} &
		\includegraphics[height=0.187\textheight]{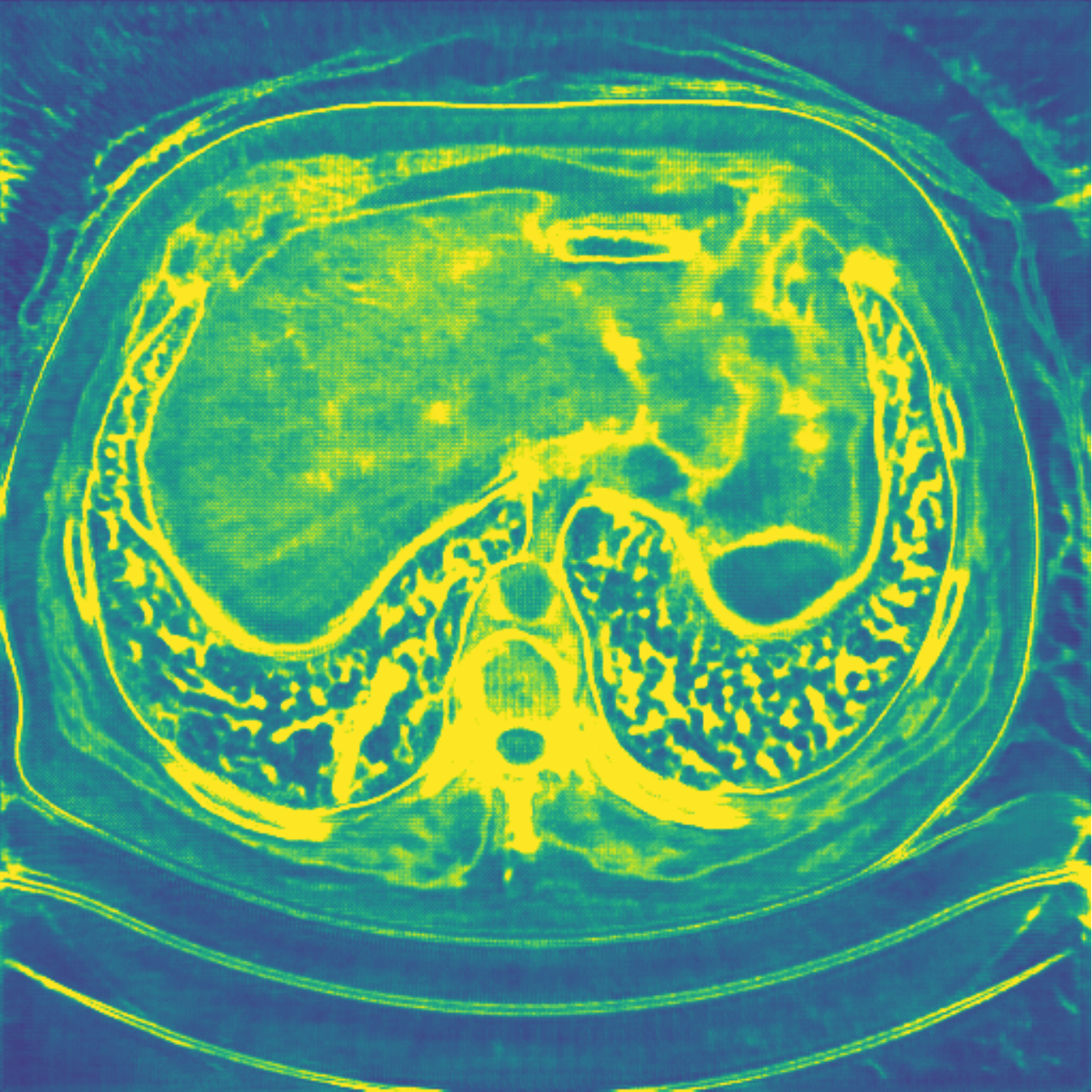}
		\hspace{-4mm} &
		\includegraphics[height=0.187\textheight]{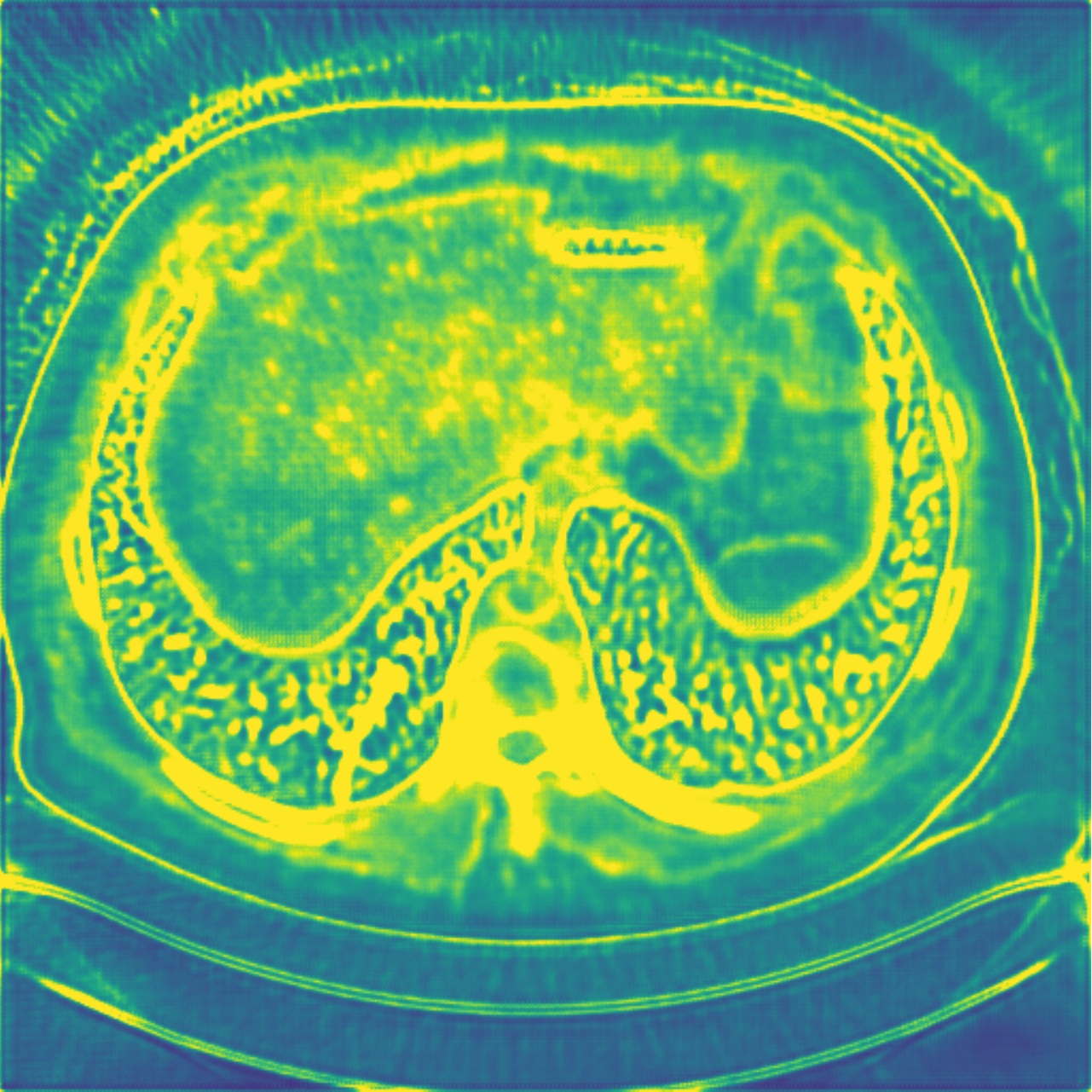}
		\hspace{-4mm} &
		\includegraphics[width=0.02\textwidth, height=0.187\textheight]{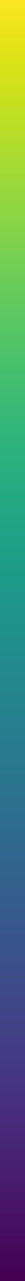} &
		\hspace{-4mm}
		$\overset{\raise0.157\textheight\hbox{\footnotesize\text{50 HU}}}{\footnotesize\text{0 HU}}$
	\end{tabular}
	\caption{Conditional mean and point-wise standard deviation (pStd) computed from test data (\cref{fig:input_data}) using posterior sampling (\cref{sec:WGAN}) and direct estimation (\cref{sec:DirectUC}).}
	\label{fig:bayesian_results}
\end{figure}

\paragraph{Uncertainty quantification}
We here show how to use Bayesian credible sets for clinical image guided decision making.
One computes a reconstruction from ultra low dose data (middle image in \cref{fig:input_data}), identifies one or more features, and then seeks to estimate the likelihood for the presence of these features.

Formalizing the above, let $\Delta$ denote the difference in mean intensity in the reconstructed image between a region encircling the feature and the surrounding organ, which in our example is the liver. 
The feature is said to ``exist'' whenever $\Delta$ is bigger than a certain threshold, say $\unit{10}{\hounsfield}$.

To use posterior sampling, start by computing the conditional mean image (top left in \cref{fig:bayesian_results}) by sampling from the posterior using the conditional \ac{WGAN} approach in \cref{sec:WGAN}. 
There is a dark ``spot'' in the liver (possible tumor) and a natural clinical question is to statistically test for the presence of this feature.
To do this, compute $\Delta$ for a number of samples generated by posterior sampling, which here is the same 1\,000 samples used for computing the conditional mean.
We estimate $p := \Prob(\Delta > \unit{10}{\hounsfield})$ from the resulting histogram in \cref{fig:hypothesis_results} and clearly $p > 0.95$, indicating that the ``dark spot'' feature exists with at least 95\% significance. 
This is confirmed by the ground truth image (left image in \cref{fig:input_data}).
The conditional mean image also under-estimates $\Delta$, whose true value is the vertical line in \cref{fig:hypothesis_results}. 
This is to be expected since the prior introduces a bias towards homogeneous regions, a bias that decreases as noise level decreases. 

To perform the above analysis using direct estimation, start with computing the conditional mean image from the same ultra-low dose data using direct estimation. 
As expected, the resulting image (top right in \cref{fig:bayesian_results}) shows a dark ``spot'' in the liver. 
Now, designing and training a neural network that directly estimates the distribution of $\Delta$ is unfeasible in a general setting.
However, as shown in \cref{sec:DirectUC}, this is possible if one assumes pixels are independent of each other. 
The estimated distribution of $\Delta$ is the curve in \cref{fig:hypothesis_results} and we get $p > 0.95$, which is consistent with the result obtained using posterior sampling. 
The direct estimation approach is based on assuming independent pixels, so it will significantly underestimate the variance. 
In contrast, the approach based on posterior sampling seems to give a more realistic estimate of the variance.
\begin{figure}[t]
	\centering
	\includegraphics[height=0.20\textheight]{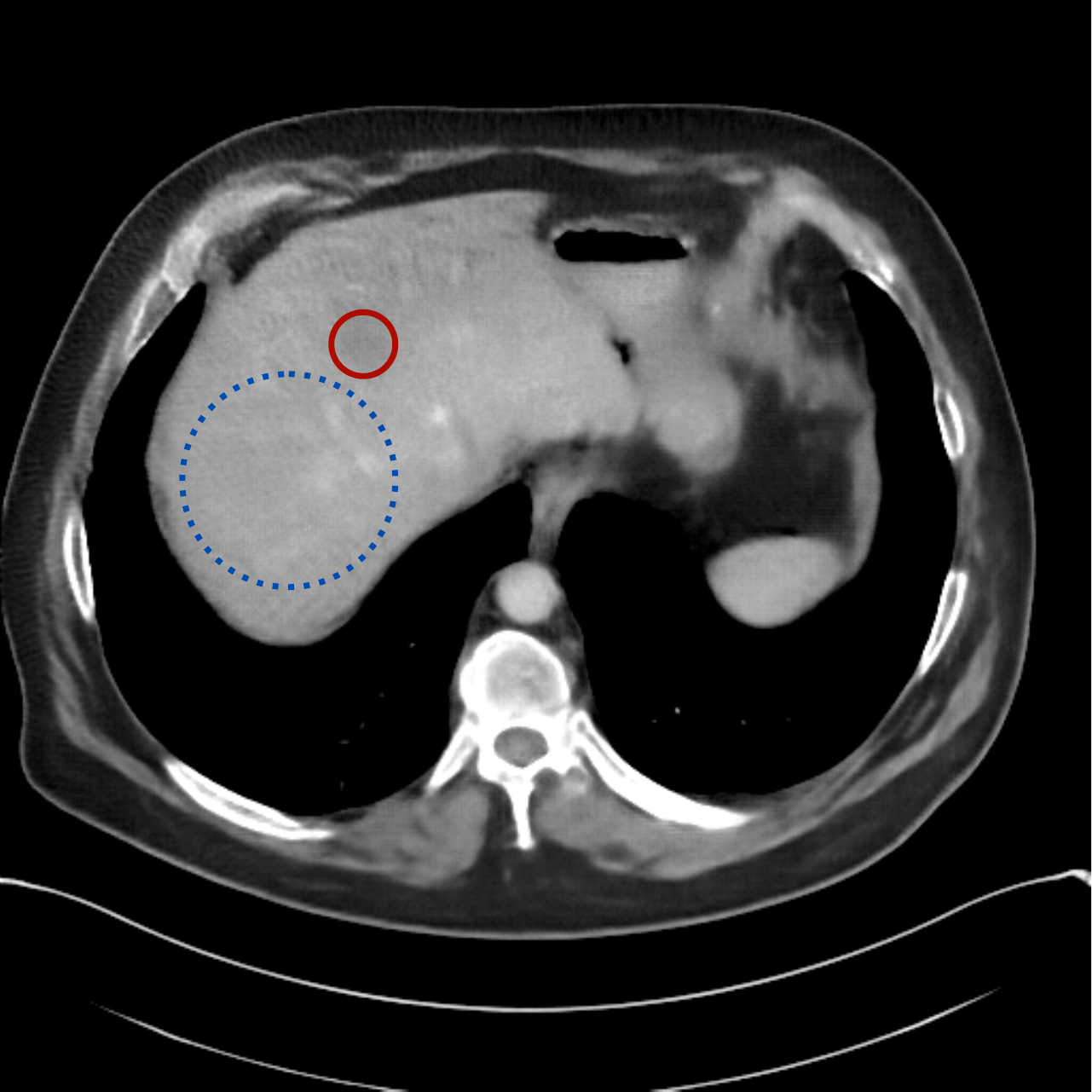}
	\raisebox{-0.7em}{\includegraphics[height=0.22\textheight]{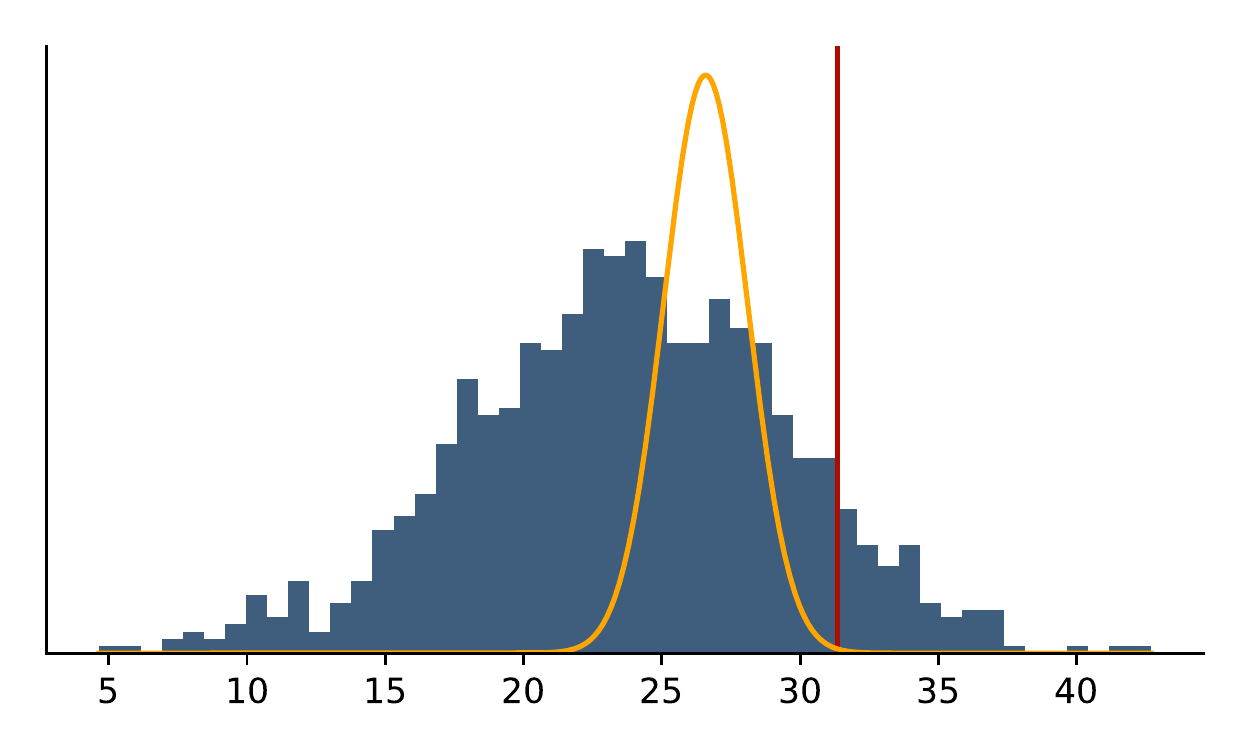}}
	\caption{The suspected tumor (red) and the reference region (blue) shown in the sample posterior mean image. Right plot shows average contrast differences between the tumor and reference region. The histogram is computed by posterior sampling applied to test data (\cref{fig:input_data}), the yellow curve is from direct estimation, and the true value is the red threshold.}
	\label{fig:hypothesis_results}
\end{figure}

\section{Related Work}\label{sec:RelWork}
Deep learning based methods are increasingly used for medical image reconstruction, either by using deep learning for post-processing \cite{DirectionalWavelet, UnserInverse} or by integrating deep learning into the image reconstruction \cite{DeepADMMNET,Adler:2017aa, Adler:2018aa,PockMRI,LEARN,DeepDBar,GradientBoostedPAT}. 
These papers start by specifying the loss and then use a deep neural network to minimize the expected loss.
This essentially amounts to directly computing a Bayes estimator with a risk is given by the loss. 
The loss is often the squared $L_2$-distance, which implicitly implies that one approximates the conditional mean. 
Hence, the above approaches could be seen as examples of deep direct estimation. 
There is however an important difference, in deep direct estimation one starts by explicitly specify the estimator, which then implies the appropriate loss function. 

There has also been intense research in selecting a loss function different from the $L_2$-loss \cite{Johnson:2016aa, Adler:2017ac} and specifically \ac{GAN}-like methods have been applied to image post-processing in \ac{CT} \cite{Wolterink:2017aa,cGANCT} and image reconstruction in \ac{MRI} (Fourier inversion) \cite{Yang:2018aa,Mardani:2017aa}. However, in these papers the authors discard providing any randomness to the generator, instead only giving it the prior. They have thus not fully realized the potential of using \acp{GAN} for sampling from the posterior in Bayesian inversion.

Regarding sampling from a posterior, conditional generative models \cite{ConditionalGAN, Nguyen:2016aa} have been widely used in the machine learning literature for this purpose. 
Typical use cases is to sample from a posterior where an image is conditioned on a text, like ``the bird is yellow'' \cite{Hu:2017aa,Deng:2017aa}, but also for simple applications in imaging, including image super-resolution and in-painting \cite{Ledig:2016aa,Parmar:2018aa,ImageTransformer}. 
These approaches do not consider sampling from the posterior for more elaborate inverse problems that involve a physics driven data likelihood.
An approach in this direction is presented in  \cite{Kohl:2018aa} where variational auto-encoders are used to sample from the posterior of possible segmentations (model parameter) given \ac{CT} images (data).

An entirely different class of methods for exploring the posterior are based on \ac{MCMC} techniques, which have revolutionized mathematical computation and enabled statistical inference within many previously intractable models.
Most of the techniques are rooted in solid mathematical theory, but they are limited to cases where the prior model is known in closed form, see surveys in \cite{Dashti:2016aa,Calvetti:2017aa,Barp:2018aa}.
Furthermore, these \ac{MCMC} techniques are still computationally unfeasible for large-scale inverse problems, like 3D clinical \ac{CT}.

A computationally feasible alternative to \ac{MCMC} for uncertainty quantification is to consider asymptotic characterizations of the posterior. 
For many inverse problems, it is possible to prove Bernstein--von Mises type of theorems that characterizes the posterior using analytic expressions assuming the prior is asymptotically uninformative \cite{Nickl:2017aa}.
Such characterizations do not hold for finite data, but assuming a Gaussian process model (data likelihood and prior are both Gaussian) allows for using numerical methods for linear inverse problems \cite{Purisha:2018aa}.
Gaussian process models are however still computationally demanding and it can be hard to design appropriate priors, so \cite{Garnelo:2018aa,Garnelo:2018ab} introduces (conditional) neural processes that incorporate deep neural networks into Gaussian process models for learning more general priors. 

Finally, another computationally feasible approach for uncertainty quantification is to approximate Bayesian credible sets for the \ac{MAP} estimator by solving a convex optimization problem \cite{Repetti:1803aa,Pereyra:2017aa}. 
The approach is however restricted to the \ac{MAP} estimator and furthermore, it requires access to a handcrafted prior.

\section{Conclusions}
Bayesian inversion is an elegant framework for recovering model parameters along with uncertainties that applies to a wide range of inverse problems.
The traditional approach requires specifying a prior and, depending on the choice of estimator, also the probability of data. 
Furthermore, exploring the posterior remains a computational challenge. 
Hence, despite significant progress in theory and algorithms, Bayesian inversion remains unfeasible for most large scale inverse problems, like those arising in imaging.

This paper addresses \emph{all} these issues, thereby opening up for the possibility to perform Bayesian inversion on large scale inverse problems. 
Capitalizing on recent advances in deep learning, we present two approaches for performing Bayesian inversion: \emph{Deep Posterior Sampling} (\cref{sec:WGAN}), which uses a \ac{GAN} to sample from the posterior, and \emph{Deep Direct Estimation} (\cref{sec:DirectUC}) that computes an estimator directly using a deep neural network.

The performance of both approaches is demonstrated in the context of ultra low dose ($2\%$ of normal dose) clinical 3D helical \ac{CT} imaging (\cref{sec:NumericalExp}).
We show how to compute basic Bayesian estimators, like the posterior mean and point-wise standard deviation. 
We also compute Bayesian credible sets and use this for testing whether a suspected ``dark spot'' in the liver, which is visible in the posterior mean image, is real. 
The quality of the posterior mean reconstruction is also quite promising, especially bearing in mind that it is computed from \ac{CT} data that corresponds to $2\%$ of normal dose. 

To the best of our knowledge, \emph{this is the first time one can perform such computations on large scale inverse problems in a timely manner, like clinical 3D helical \ac{CT} image reconstruction}.
On the other hand, using such a radically different way to perform image reconstruction in clinical practice quickly gets complicated since it must be preceded by clinical trials in the context of image guided decision making.
However, there are many advantages that comes with using our proposed approach, which for medical imaging means integrating imaging with clinical decision making while accounting for the uncertainty.

To conclude, the posterior sampling approach allows one to perform Bayesian inversion on large scale inverse problems that goes beyond computing specific estimators, such as the \ac{MAP} or conditional mean. 
The framework is not specific to tomography, it applies to essentially any inverse problem assuming access to sufficient amount of ``good'' supervised training data.
Furthermore, the possibility to efficiently sample from the posterior opens up for new ways to integrate decision making with reconstruction. 

\section{Discussion and Outlook}
There are several open research topics related to using \acp{GAN} as generative models for the posterior in inverse problems. 

One natural topic is to have a precise notion of ``good'' supervised training data. 
Specifically, it is desirable to estimate the amount of supervised training data necessary for ``resolving'' the posterior/estimator up to some accuracy.
Unfortunately, most of the current theory for Bayesian inference does not apply directly to this setting. 
Its emphasis is on characterizing the posterior in the asymptotic regime where information content in data increases indefinitely and the prior is asymptotically non-informative, like when a Gaussian prior is used. 

Another research topic is to study whether there are theoretical guarantees that ensure the conditional \ac{WGAN} generator given by \cref{eq:FormulationWithZ} converges towards the posterior. 
In \cite{Goodfellow:2014aa} one proves that given infinite capacity of the discriminator, the optimal generator minimizes the Jensen--Shannon divergence w.r.t. the target distribution.
For the case with \ac{WGAN}, \cite[Lemma~6]{Qi:2017aa} shows that one can learn the posterior in the sense of \cite[Definition 3.1]{Shalev-Shwartz:2014aa}, i.e. solving \cref{eq:GenericWGANFormulationNonParam} arbitrarily well, given enough training data.
But this does not settle the question of what happens with realistic sample and model capacities.
This is part of a more general research theme for investigating the theoretical basis for using \acp{GAN} trained on supervised data to sample from high dimensional probability distributions \cite{Arora:2018aa}.

Yet another topic relates to including explicit knowledge about the data likelihood, which in contrast to the prior, can be successfully handcrafted for many inverse problems. 
This is essential for large-scale inverse problems where the amount of supervised training data is little and there are few opportunities for re-training when data the acquisition protocol changes. 
In this work, this knowledge was implicitly accounted for by our choice to use a \ac{FBP} reconstruction as the data.
While it can be proven that this is in theory sufficient for generating samples from the posterior \cite[section~8]{Adler:2018ab}, \cite{Adler:2017aa,Adler:2018aa} clearly shows that working directly from measured data gives better results. 
We therefore expect further improvements to our results by using a conditional \acp{WGAN} based on \ac{CNN} architectures that integrate a handcrafted data likelihood, such as those provided by learned iterative reconstruction.

Finally, our deep direct estimators were very easy to train with no major complications, but training the generative models for posterior sampling is still complicated and involves quite a bit of fine tuning and ``tricks''. We hope that future research in generative models will improve upon this situation.

\section*{Acknowledgments*}
The work was supported by the Swedish Foundation of Strategic Research grant AM13-0049, Industrial PhD grant ID14-0055 and by Elekta.
The authors also thank Dr. Cynthia McCollough, the Mayo Clinic, and the American Association of Physicists in Medicine, and acknowledge funding from grants EB017095 and EB017185 from the National Institute of Biomedical Imaging and Bioengineering, for providing the data.

\FloatBarrier

\clearpage
\thispagestyle{empty}
\quad\par
\vskip0.1\textheight
\begin{center}
  \textbf{\Huge Appendices}
\end{center}
\clearpage

\appendix
\acresetall

\section{The Wasserstein 1-distance}
\label{sec:Wasserstein}
Let $\RecSpace$ be a measurable separable Banach Space and $\PClass_{\RecSpace}$ the space of probability measures on $\RecSpace$. 
The Wasserstein 1-distance $\Wasserstein \colon \PClass_{\RecSpace} \times \PClass_{\RecSpace} \to \Real$ is a metric on $\PClass_{\RecSpace}$ that can be defined as \cite[Definition~6.1]{Villani:2009aa} 
\begin{equation}\label{eq:WasserSteinMetricOrig}
   \Wasserstein(p,q) 
   := 
   \!\!
   \inf_{\mu \in \Pi(p,q)} \Expect_{(\stsignal,\stsignalother) \sim \mu}\bigl[ \Vert \stsignal-\stsignalother \Vert_{\RecSpace} \bigr]
   \quad\text{for $p,q \in \PClass_{\RecSpace}$.}
\end{equation}
In the above, $\Pi(p,q) \subset \PClass_{\RecSpace \times \RecSpace}$ denotes the family of joint probability measures on $\RecSpace \times \RecSpace$ that has $p$ and $q$ as marginals.
Note also that we assume $\PClass_{\RecSpace}$ only contains measures where the Wasserstein distance takes finite values (Wasserstein space), see \cite[Definition 6.4]{Villani:2009aa} for the formal definition.

The Wasserstein 1-distance in \cref{eq:WasserSteinMetricOrig} can be rewritten using the Kantorovich-Rubinstein dual characterization \cite[Remark~6.5 on p.~95]{Villani:2009aa}, resulting in 
\begin{equation}\label{eq:WassersteinKL}
	\Wasserstein(p,q) 
    = 
    \!\!
    \sup_{\subalign{& \Discriminator \colon \RecSpace \to \Real \,\,\, \\ & \LipClass{\Discriminator}{\RecSpace}}}
     \Bigl\{
       \Expect_{\stsignal \sim q}\bigl[ \Discriminator(\stsignal) \bigr] 
	-
       \Expect_{\stsignalother \sim p}\bigl[ \Discriminator(\stsignalother) \bigr] 
     \Bigr\}
     \quad\text{for $p,q \in \PClass_{\RecSpace}$.}
\end{equation}
Here, $\Lip(\RecSpace)$ denotes real-valued 1-Lipschitz maps on $\RecSpace$, i.e., 
\[
	\LipClass{\Discriminator}{\RecSpace}
	\quad
	\Longleftrightarrow
	\quad
	|
	 	\Discriminator(\signal_1) - \Discriminator(\signal_2)
 	|
 	\leq
 	\|\signal_1 - \signal_2\|_{\RecSpace}
	\quad\text{for all $\signal_1, \signal_2 \in \RecSpace$.}
\]
The above constraint can be hard to enforce in \cref{eq:WassersteinKL} as is, so following \cite{WGAN-GP, BWGAN} we prefer the \emph{gradient characterization}:
\[
	\LipClass{\Discriminator}{\RecSpace}
	\quad
	\Longleftrightarrow
	\quad
	\bigl\|\partial\! \Discriminator(\signal) \bigr\|_{\RecSpace^*}
	\leq
	1
	\quad\text{for all $\signal \in \RecSpace$,}	
\]
where $\partial$ indicates the Fr\'echet derivative and $\RecSpace^*$ is the dual space of $\RecSpace$. 
In our setting, $\RecSpace$ is an $L_2$ space, which is a Hilbert space so $\RecSpace^* = \RecSpace$ and the Fr\'echet derivative becomes the (Hilbert space) gradient of $\Discriminator$.

\section{Individual Posterior Samples}\label{sec:individual_samples}
It is instructive to visually inspect individual random samples of the posterior obtained from the conditional \ac{WGAN} generator.

Generating one such sample is fast, taking approximately $\unit{40}{\milli\second}$ on a desktop ``gaming'' PC.
Furthermore, as seen in \cref{fig:posterior_samples}, the generated samples look realistic, practically indistinguishable from the ground truth to the \emph{untrained} observer. 
With that said, some anatomical features are clearly misplaced, e.g., there are white ``blobs'' (blood vessels) in the liver. 
These are present because the supervised training set contained images from patients that were given contrast (see bottom row in \cref{fig:analytic_priors}), which has influenced the anatomical prior that is learned from the supervised data. 
\begin{figure}[t]
	\includegraphics[height=0.187\textheight]{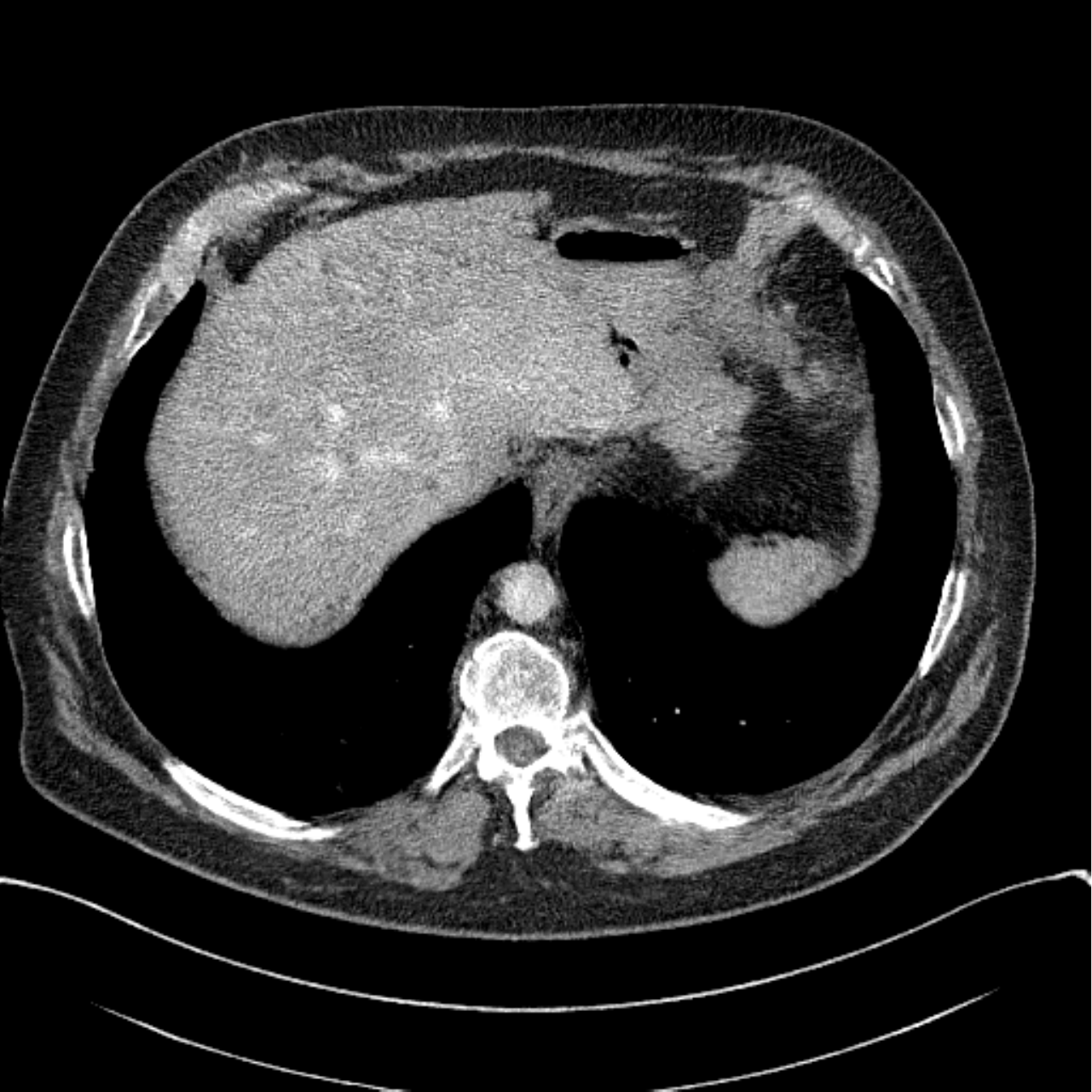}
	\includegraphics[height=0.187\textheight]{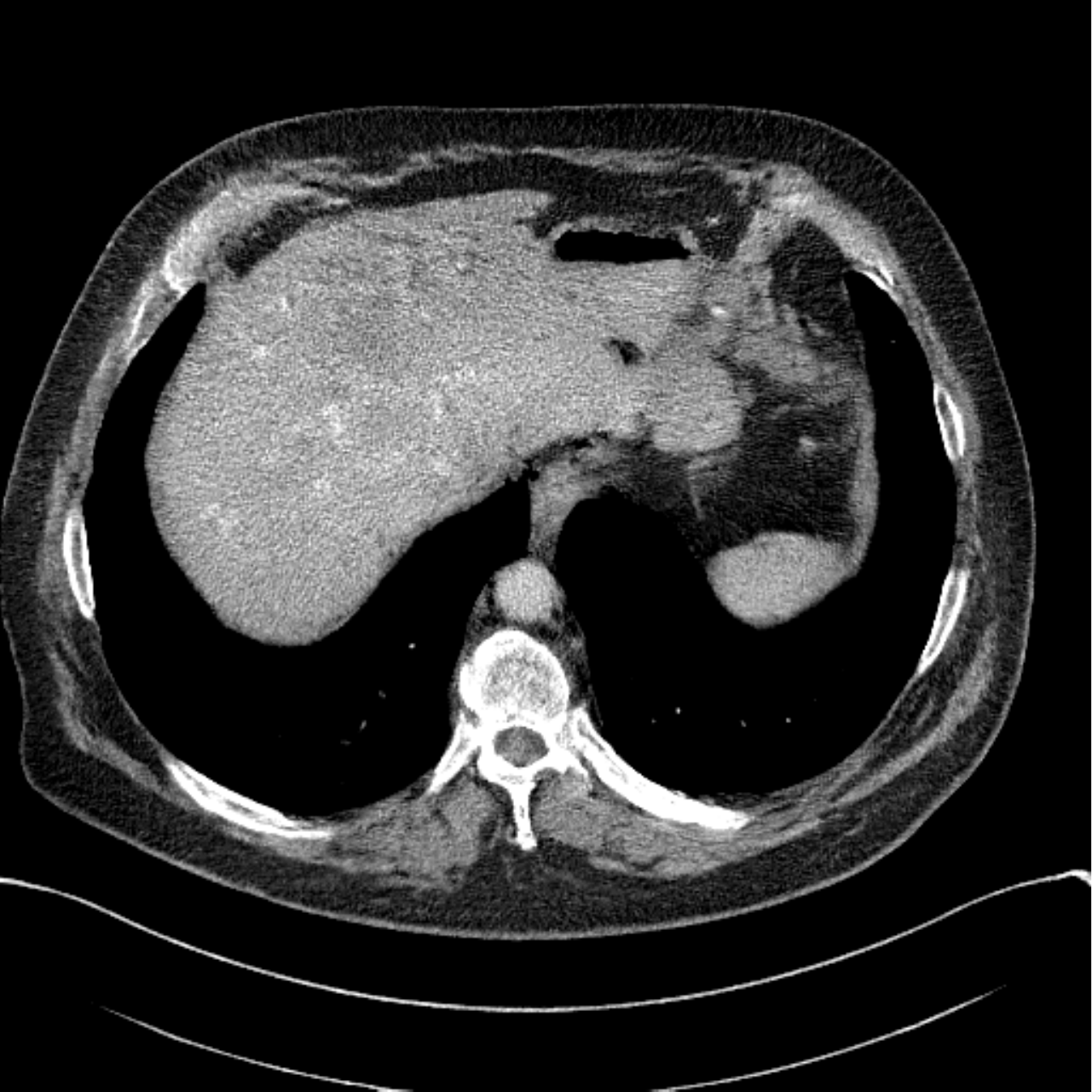}
	\includegraphics[height=0.187\textheight]{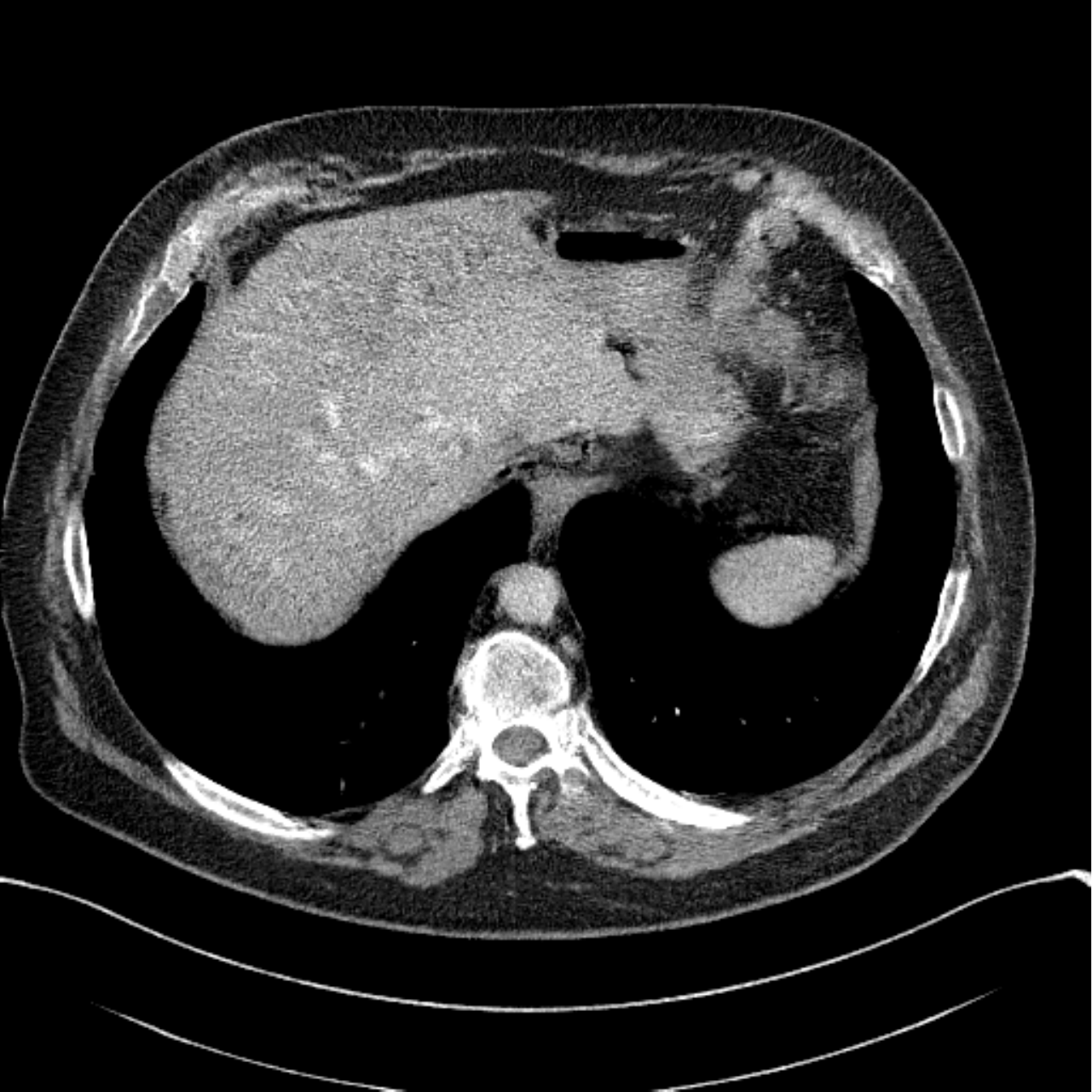}
	\\[0.25em]
	\includegraphics[height=0.187\textheight]{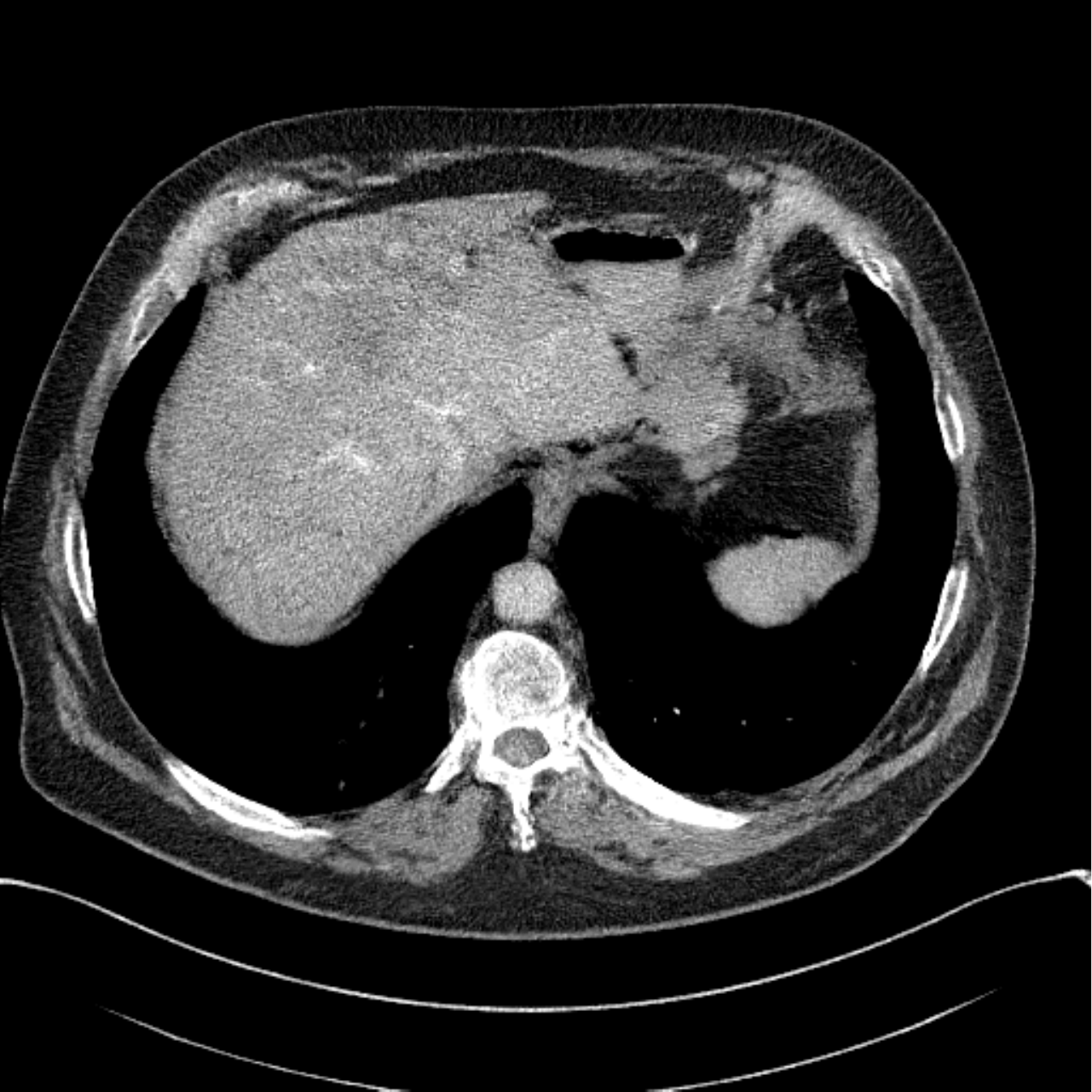}
	\includegraphics[height=0.187\textheight]{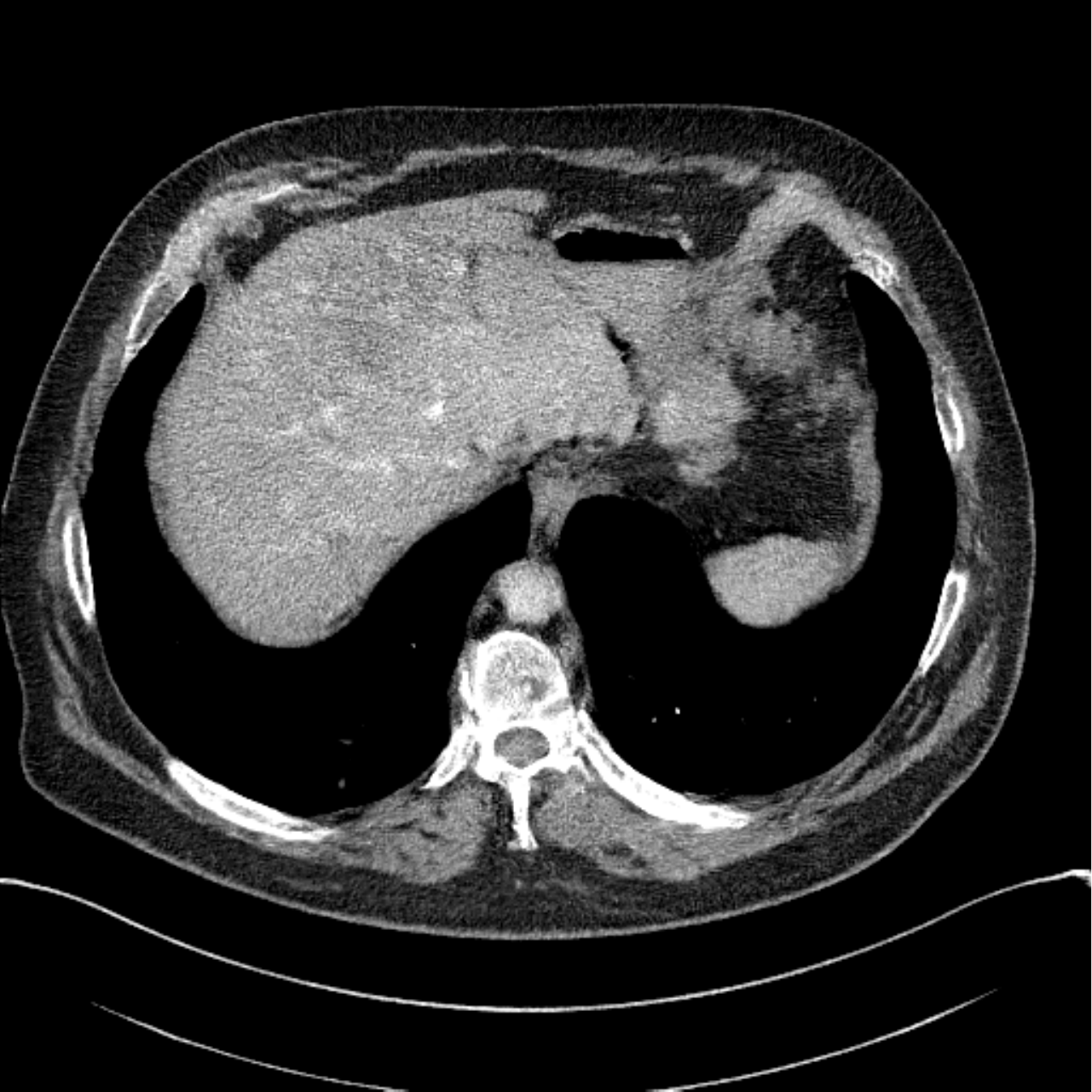}
	\includegraphics[height=0.187\textheight]{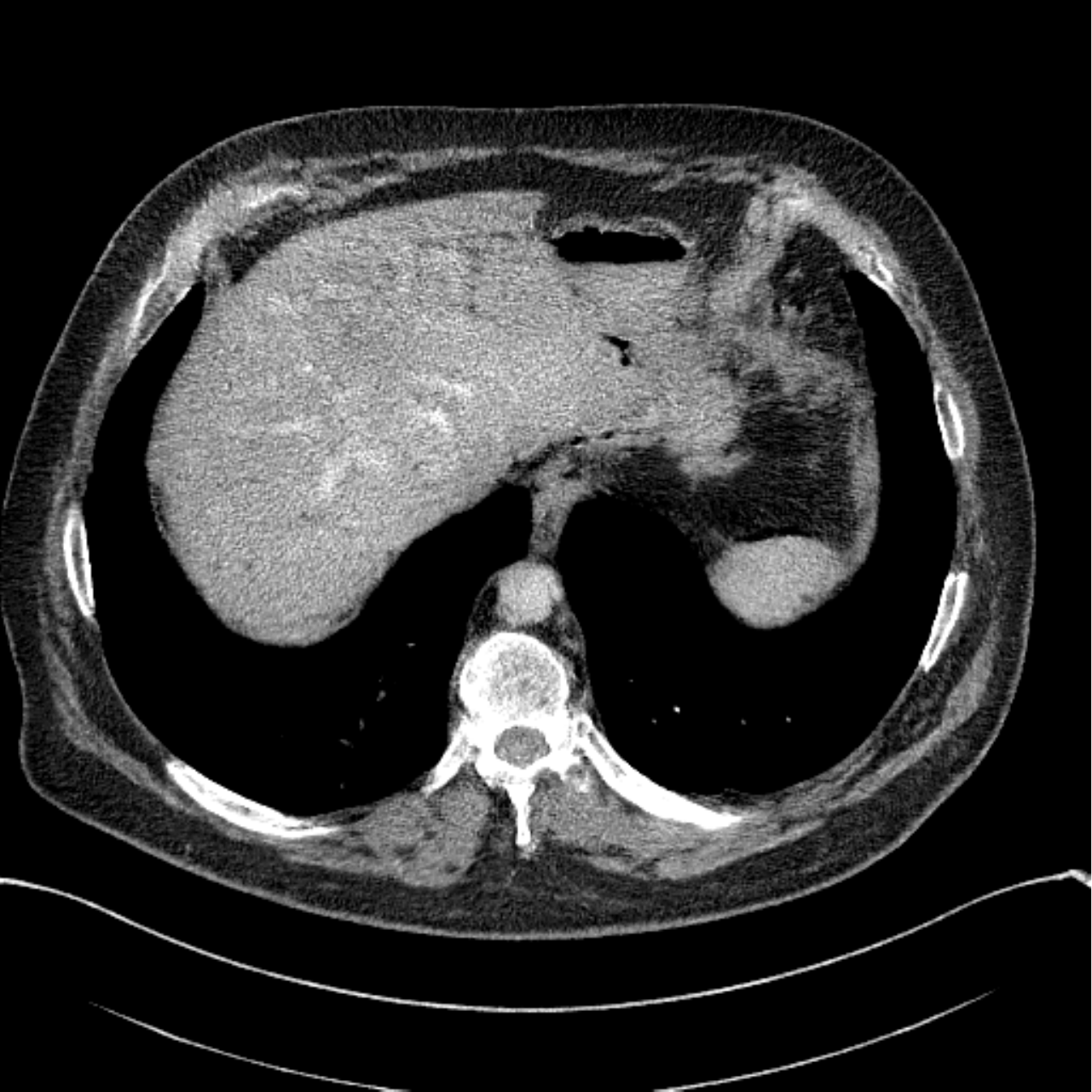}
	\caption{Deep posterior samples (\cref{sec:WGAN}) on test data (\cref{fig:input_data}) shown using a display window set to [-150, 200] \hounsfield.}
	\label{fig:posterior_samples}
\end{figure}

\section{Theory of Deep Bayesian Inversion}
This section contains the theoretical foundations needed for Deep Bayesian Inversion and derivations of the expressions used in the main article.

\subsection{Derivation of conditional \ac{WGAN}}\label{app:WGANMath}
This section provides the mathematical details for deriving \cref{eq:FormulationWithZ} from \cref{eq:WGANFormulation}.
Such a reformulation is well-known, but the derivation given here seems to be novel.

The overall aim with \ac{WGAN} is to approximate the posterior $\data \mapsto \posterior{\data}$ that is inaccessible. 
The approach is to construct a mapping that associates a probability measure on $\RecSpace$ to each data $\data \in \DataSpace$.
A family of such mappings can be explicitly constructed and trained against supervised training data in order to approximate the posterior.
To proceed we need to specify the statistical setting and our starting point is to introduce a ``distance'' between two probability measures on $\RecSpace$:
\begin{equation}\label{eq:ProbDist}
  \ProbDist \colon \PClass_{\RecSpace} \times \PClass_{\RecSpace} \to \Real_+. 
\end{equation}  
In the above, $\PClass_{\RecSpace}$ denotes the class of all probability measures on $\RecSpace$, so $\posterior{\data} \in \PClass_{\RecSpace}$ whenever the posterior exists, which holds under fairly general assumptions where both $\RecSpace$ and $\DataSpace$ can be infinite-dimensional separable Banach spaces \cite[Theorem~14]{Dashti:2016aa}. 
Next, let $\GenProbMapSpace$ denote a fixed family of \emph{generators}, which are mappings 
\begin{equation}\label{eq:GenProb}
  \GenProb \colon \DataSpace \to \PClass_{\RecSpace} 
\end{equation}  
that associate each $\data \in \DataSpace$ to a probability measure on $\RecSpace$.
Note here that $\data \mapsto \posterior{\data}$ is not necessarily contained in $\GenProbMapSpace$.
A generator in $\GenProbMapSpace$ is an ``optimal'' approximation of the posterior $\data \mapsto \posterior{\data}$ if it minimizes the expected $\ProbDist$-distance, i.e., it solves 
\begin{equation}\label{eq:GenericFormulationNonParam}
  \inf_{\GenProb \in \GenProbMapSpace}\,
    \Expect_{\stdata \sim \dataprob} \Bigl[
      \ProbDist \bigl( \GenProb(\stdata), \posterior{\stdata} \bigr)
    \Bigr].
\end{equation}
Here, $\stdata \sim \dataprob$ is the $\DataSpace$-valued random variable generating data. 

There are three issues that arise if a solution to \cref{eq:GenericFormulationNonParam} is to be used as a proxy for the posterior:
\begin{inparaenum}[(i)]
\item Evaluating the objective requires access to the very posterior, which we assumed was inaccessible, 
\item the distribution $\dataprob$ of data is almost always unknown, so the expectation cannot be computed, and finally 
\item the computational feasibility requires access to an explicit finite dimensional parametrization for constructing generators in $\GenProbMapSpace$ that one searches over in \cref{eq:GenericFormulationNonParam}.
\end{inparaenum}

As we shall see next, choosing \cref{eq:ProbDist} as the Wasserstein 1-distance allows us to addresses the first two issues. With this choice one can re-write the objective in \cref{eq:GenericFormulationNonParam} as an expectation over the joint law $(\stsignal,\stdata) \sim \jointlaw$, thereby avoiding expressions that explicitly depend on the unknown posterior $\data \mapsto \posterior{\data}$ and distribution of data $\dataprob$.
This joint law is also unknown, but it can often be replaced by its empirical counterpart derived from a suitable supervised training data set.

More precisely, choosing the Wasserstein 1-distance $\Wasserstein \colon \PClass_{\RecSpace} \times \PClass_{\RecSpace} \to \Real_+$ as $\ProbDist$ in \cref{eq:GenericFormulationNonParam} yields 
\begin{equation}\label{eq:GenericWGANFormulationNonParam}
  \inf_{\GenProb \in \GenProbMapSpace}\,
    \Expect_{\stdata \sim \dataprob} \Bigl[
      \Wasserstein \bigl( \GenProb(\stdata), \posterior{\stdata} \bigr)
    \Bigr].
\end{equation}
Note here that $\data \mapsto \Wasserstein \bigl( \GenProb(\data), \posterior{\data} \bigr)$ is assumed to be a measurable real-valued function on $\DataSpace$.
The Kantorovich-Rubinstein dual characterization in \cref{eq:WassersteinKL} yields
\begin{equation}\label{eq:WasserSteinMetric}
  \Wasserstein\bigl( \GenProb(\data), \posterior{\data} \bigr)
  = \!\!\!\sup_{\Discriminator_{\data} \in \Lip(\RecSpace)}     
     \biggl\{
       \Expect_{\subalign{ & \stsignal \sim \posterior{\data} \\ 
         & \stsignalother \sim \GenProb(\data)}}
       \bigl[ 
       \Discriminator_{\data}(\stsignal)  
       -
       \Discriminator_{\data}(\stsignalother) \bigr] 
    \biggr\}
    \quad\text{for $\data \in \DataSpace$.}
\end{equation}
Here, $\Lip(\RecSpace)$ denotes the set of real-valued mappings on $\RecSpace$ that are 1-Lipschitz.
Hence, \cref{eq:GenericWGANFormulationNonParam} can be written as  
\begin{equation}\label{eq:WassersteinFormulationNonParam}
\inf_{\GenProb \in \GenProbMapSpace}\,
  \Expect_{\stdata \sim \dataprob} \Biggl[ 
     \sup_{\Discriminator_{\stdata} \in \Lip(\RecSpace)}
     \biggl\{
       \Expect_{\subalign{ & \stsignal \sim \posterior{\stdata} \\ 
         & \stsignalother \sim \GenProb(\stdata)}}
       \bigl[ 
       \Discriminator_{\stdata}(\stsignal)  
       -
       \Discriminator_{\stdata}(\stsignalother) \bigr] 
    \biggr\}
  \Biggr].
\end{equation}
Next, in this case the supremum commutes with the $\dataprob$-expectation, i.e., 
\begin{multline}\label{eq12}
\Expect_{\stdata \sim \dataprob} \Biggl[ 
\sup_{\Discriminator_{\stdata} \in \Lip(\RecSpace)}
\biggl\{
\Expect_{\subalign{ & \stsignal \sim \posterior{\stdata} \\ 
		& \stsignalother \sim \GenProb(\stdata)}}
\bigl[ 
\Discriminator_{\stdata}(\stsignal)  
-
\Discriminator_{\stdata}(\stsignalother) \bigr] 
\biggr\}
\Biggr]
\\
=
\sup_{\Discriminator \in \DiscrSpace(\RecSpace \times \DataSpace)} \biggl\{
  \Expect_{\subalign{& (\stsignal,\stdata) \sim \jointlaw \\ & \stsignalother \sim \GenProb(\stdata)}} \Bigl[ 
    \Discriminator(\stsignal,\stdata) - \Discriminator(\stsignalother,\stdata)
 \Bigr] \biggr\},
\end{multline}
where $\DiscrSpace(\RecSpace \times \DataSpace)$ is the space of measurable real-valued mappings on $\RecSpace \times \DataSpace$ that are 1-Lipschitz in the $\RecSpace$-variable for every $\data \in \DataSpace$.
The proof of \cref{eq12} is given on p.~\pageref{proof12} and combining it with \cref{eq:WassersteinFormulationNonParam} gives 
\begin{equation}\label{eq:FinalNonParamFormulation}
   \inf_{\GenProb \in \GenProbMapSpace}\,
   \Biggl\{
     \sup_{\Discriminator \in \DiscrSpace(\RecSpace \times \DataSpace)}
          \Expect_{\subalign{& (\stsignal,\stdata) \sim \jointlaw \\ & \stsignalother \sim \GenProb(\stdata)}} \Bigl[ 
            \Discriminator(\stsignal,\stdata) - \Discriminator(\stsignalother,\stdata)
         \Bigr]
    \Biggr\}.
\end{equation}

Note that there are no approximations involved in going from \cref{eq:GenericWGANFormulationNonParam} to \cref{eq:FinalNonParamFormulation}, the derivation is solely based on properties of the Wasserstein 1-distance.
Furthermore, the advantage of \cref{eq:FinalNonParamFormulation} over \cref{eq:GenericWGANFormulationNonParam} is that the latter neither involves the posterior nor $\dataprob$ (the probability measure of data). 
It does involve the joint law $(\stsignal,\stdata) \sim \jointlaw$, which is of course unknown. On the other hand, if we have access to supervised training data:
\begin{equation}\label{eq:TData}
  (\signal_1,\data_1), \ldots, (\signal_m,\data_m) \in \RecSpace \times \DataSpace
  \quad\text{are i.i.d. samples of $(\stsignal,\stdata) \sim \jointlaw$,}
\end{equation}
then we can replace the joint law $\jointlaw$ in \cref{eq:FinalNonParamFormulation} with its empirical counterpart and the $\jointlaw$-expectation is replaced by an averaging over training data.

The final steps concern computational feasibility. We start by considering parameterizations of the generators in $\GenProbMapSpace$ that enables one to solve \cref{eq:FinalNonParamFormulation} in a computational feasible manner.
A key aspect is to evaluate the $\GenProb(\data)$-expectation for any $\data \in \DataSpace$ without impairing upon the ability to approximate the posterior with elements from $\GenProbMapSpace$.
We will assume that each generator $\GenProb \in \GenProbMapSpace$ corresponds to a measurable map $\Generator \colon \GenSpace \times \DataSpace \to \RecSpace$ such that the following holds: 
\begin{equation}\label{eq:GenRel}
   \stsignalother \sim \GenProb(\data)
   \iff  
   \stsignalother = \Generator(\stgenvar,\data) \quad\text{for some $\GenSpace$-valued random variable $\stgenvar \sim \genvarprob$.}
\end{equation}
In the above, $\GenSpace$ is some fixed set and $\genvarprob$ is a ``simple'' probability measure on $\GenSpace$ meaning that there are computationally efficient means for generating samples of $\stgenvar \sim \genvarprob$.
It is then possible to express \cref{eq:FinalNonParamFormulation} as 
\begin{equation}\label{eq:FinalNonParamFormulation2}
\inf_{\Generator \in \GeneratorSpace}\,
  \biggl\{
    \sup_{\Discriminator \in \DiscrSpace(\RecSpace \times \DataSpace)}
    \Expect_{\subalign{& (\stsignal,\stdata) \sim \jointlaw \\ & \stgenvar \sim \genvarprob}}
    \Bigl[ 
      \Discriminator(\stsignal,\stdata)
      -
      \Discriminator\bigl(\Generator(\stgenvar,\stdata),\stdata\bigr)
    \Bigr]
  \biggr\}
\end{equation}
where $\GeneratorSpace$ is the class of $\RecSpace$-valued measurable maps on $\GenSpace \times \DataSpace$ that corresponds to $\GenProbMapSpace$ by \cref{eq:GenRel}.

The formulation in \cref{eq:FinalNonParamFormulation2} involves taking the infimum over $\GeneratorSpace$ and supremum over $\DiscrSpace$, which is clearly computationally unfeasible. 
Hence, one option is to consider a parametrization of these spaces using deep neural networks with appropriately chosen architectures:
\begin{alignat}{2}
  \GeneratorSpace := \{ \Generator_{\genparam} \}_{\genparam \in \GenParamSet}
  &\qquad \text{where}\quad &\Generator_{\genparam} \colon \GenSpace \times \DataSpace \to \RecSpace \phantom{.}
  \label{eq:GenArchitecture}
  \\[0.25em]
  \DiscrSpace := \{ \Discriminator_{\discrparam} \}_{\discrparam \in \DiscrParamSet}
  &\qquad \text{where}\quad &\Discriminator_{\discrparam} \colon \RecSpace \times \DataSpace \to \Real.
  \label{eq:DiscArchitecture}
\end{alignat}
Inserting the above parametrizations into \cref{eq:FinalNonParamFormulation2} results in 
\begin{equation}\label{eq:FinalFormulation}
\genparam^* \in \argmin_{\genparam \in \GenParamSet}\,
  \biggl\{
    \sup_{\discrparam \in \DiscrParamSet}\,\,
    \Expect_{\subalign{& (\stsignal,\stdata) \sim \jointlaw \\ & \stgenvar \sim \genvarprob}}
    \Bigl[
      \Discriminator_{\discrparam}(\stsignal,\stdata)
      -
     \Discriminator_{\discrparam}\bigl(\Generator_{\genparam}(\stgenvar,\stdata),\stdata\bigr) 
    \Bigr]
  \biggr\}.    
\end{equation}
Note again that the unknown joint law $\jointlaw$ in \cref{eq:FinalFormulation} is replaced by its empirical counterpart given from the training data in \cref{eq:TData}.

To summarize, solving the training problem in \cref{eq:FinalFormulation} given training data \cref{eq:TData} and the parametrizations in \cref{eq:GenArchitecture,eq:DiscArchitecture} yields a mapping $\Generator_{\genparam^*} \colon \GenSpace \times \DataSpace \to \RecSpace$ that approximates the posterior in the sense that the distribution of $\Generator_{\genparam^*}(\stgenvar, \data)$ with $\stgenvar \sim \genvarprob$ is closest to $\posterior{\data}$ in  expected Wasserstein 1-distance.
Hence, we can sample $\genvar \in \GenSpace$ from $\stgenvar \sim \genvarprob$ and $\Generator_{\genparam^*}(\genvar,\data) \in \RecSpace$ will approximate a sample of the conditional random variable $(\stsignal \mid \stdata = \data) \sim \posterior{\data}$.
The formulation in \cref{eq:FinalFormulation} is also suitable for \ac{SGD}, so computational techniques from deep neural networks can be used for solving the empirical expected minimization problem. We conclude with providing a proof of \cref{eq12}.

\paragraph{Proof of \cref{eq12}}\label{proof12}
To simplify the notational burden, define $f_{\Discriminator} \colon \DataSpace \to \Real$ as
\[
    f_{\Discriminator}(\data)
    :=
    \Expect_{\subalign{ & \stsignal \sim \posterior{\data} \\
     		& \stsignalother \sim \GenProb(\data)}}
     \bigl[ 
     \Discriminator(\stsignal)  
     -
     \Discriminator(\stsignalother) \bigr]
\quad\text{for $\Discriminator \in \Lip(\RecSpace)$.}
\]
Next, $(\stsignal,\stdata) \sim \jointlaw$ with $\jointlaw = \posterior{\stdata} \otimes \dataprob$, so by the law of total expectation we can re-write the objective in the right-hand side of \cref{eq12} as
\[
	\Expect_{\subalign{& (\stsignal,\stdata) \sim \jointlaw \\ & \stsignalother \sim \GenProb(\data)}} \Bigl[ 
	\Discriminator(\stsignal,\stdata) - \Discriminator(\stsignalother,\stdata)
	\Bigr]
	=
	\Expect_{\stdata \sim \dataprob}
	\bigl[ f_{\Discriminator(\cdot,\stdata)}(\stdata) \bigr].
\]
Hence, proving \cref{eq12} is equivalent to proving 
\begin{equation}\label{eq12simple}
\Expect_{\stdata \sim \dataprob}
\Bigl[
\sup_{\Discriminator_{\stdata} \in \Lip(\RecSpace)}
f_{\Discriminator_{\stdata}}(\stdata)
\Bigr]
=
\sup_{\Discriminator \in \DiscrSpace(\RecSpace \times \DataSpace)}\!\!\!\!
\Expect_{\stdata \sim \dataprob} \bigl[ f_{\Discriminator(\cdot,\stdata)}(\stdata) \bigr].
\end{equation}
To prove \cref{eq12simple}, note first that the claim clearly holds when equality is replaced with ``$\geq$'' since $\Discriminator(\cdot, \data) \in \Lip(\RecSpace)$ for any $\Discriminator \in \DiscrSpace(\RecSpace \times \DataSpace)$. It remains to prove that strict inequality in \cref{eq12simple} cannot hold. 
In the following, we use a proof by contradiction approach, so assume strict inequality holds: 
\begin{equation}\label{eq:ContradictionAssumption}
\Expect_{\stdata \sim \dataprob}
\Bigl[
\sup_{\Discriminator_{\stdata} \in \Lip(\RecSpace)}
f_{\Discriminator_{\stdata}}(\stdata)
\Bigr]
>
\sup_{\Discriminator \in \DiscrSpace(\RecSpace \times \DataSpace)}\!\!\!\!
\Expect_{\stdata \sim \dataprob} \bigl[ f_{\Discriminator(\cdot,\stdata)}(\stdata) \bigr].
\end{equation}
From \cref{eq:ContradictionAssumption}, there exists $\varepsilon > 0$ such that 
\begin{equation}
	\label{eq:archimedian}
	\Expect_{\stdata \sim \dataprob}
	\Bigl[
	\sup_{\Discriminator_{\stdata} \in \Lip(\RecSpace)}
	f_{\Discriminator_{\stdata}}(\stdata)
	\Bigr]
	-
	\varepsilon
	>
	\sup_{\Discriminator \in \DiscrSpace(\RecSpace \times \DataSpace)}\!\!\!\!
	\Expect_{\stdata \sim \dataprob}
	\bigl[
	f_{\Discriminator(\cdot,\stdata)}(\stdata)
	\bigr].
\end{equation}
Next, for any $\data \in \DataSpace$ and $\varepsilon > 0$, there exists $\widehat{\Discriminator}_\data \in \Lip(\RecSpace)$ such that
\begin{equation}\label{eq24}
	f_{\widehat{\Discriminator}_\data}(\data)
	>
	\sup_{\Discriminator_\data \in \Lip(\RecSpace)}
	f_{\Discriminator_\data}(\data)
	-
	\varepsilon
	\quad\text{holds for any $\data \in \DataSpace$.}
\end{equation}
\emph{Assume} next that it is possible to \emph{choose} $\widehat{\Discriminator}_\data$ so that $(\signal,\data) \mapsto \widehat{\Discriminator}_\data(\signal)$ is measurable on $\RecSpace \times \DataSpace$.
This implies that $(\signal,\data) \mapsto \widehat{\Discriminator}_\data(\signal) \in \DiscrSpace(\RecSpace \times \DataSpace)$ since $\widehat{\Discriminator}_\data \in \Lip(\RecSpace)$ for all $\data \in \DataSpace$.
Hence, the $\dataprob$-expectation of $f_{\widehat{\Discriminator}_\stdata}(\stdata)$ exists and \cref{eq24} combined with the monotonicity of the expectation gives 
\begin{align*}
\Expect_{\stdata \sim \dataprob}
\Bigl[
f_{\widehat{\Discriminator}_\stdata}(\stdata)
\Bigr]
&>
\Expect_{\stdata \sim \dataprob}
\Bigl[
\sup_{\Discriminator_\stdata \in \Lip(\RecSpace)}
f_{\Discriminator_\stdata}(\stdata)
-
\varepsilon.
\Bigr]
\\[0.25em]
&=
\Expect_{\stdata \sim \dataprob}
\Bigl[
\sup_{\Discriminator_\stdata \in \Lip(\RecSpace)}
f_{\Discriminator_\stdata}(\stdata)
\Bigr]
-
\varepsilon.
\end{align*}
Insert the above into \cref{eq:archimedian} gives 
\begin{equation}\label{eq:Contradiction}
	\Expect_{\stdata \sim \dataprob}
	\Bigl[
	f_{\widehat{\Discriminator}_\stdata}(\stdata)
	\Bigr]
	>
	\sup_{\Discriminator \in \DiscrSpace(\RecSpace \times \DataSpace)}\!\!\!
	\Expect_{\stdata \sim \dataprob}
	\bigl[
	f_{\Discriminator(\cdot,\stdata)}(\stdata)
	\bigr].
\end{equation}
Since $(\signal,\data) \mapsto \widehat{\Discriminator}_\data(\signal) \in \DiscrSpace(\RecSpace \times \DataSpace)$, the statement in \cref{eq:Contradiction} contradicts the definition of the supremum, i.e., \cref{eq:ContradictionAssumption} leads to a contradiction implying that \cref{eq12simple} is true. This concludes the proof.

\subsection{A novel discriminator for conditional \ac{WGAN}}\label{sec:minibatchDiscriminator}
A generator trained using the formulation in \cref{eq:FinalFormulation} \emph{as is} will typically learn to ignore the randomness from $\stgenvar \sim \genvarprob$. 
This can be seen in \cref{fig:posterior_samples_no_minibatch,fig:bayesian_results_no_minibatch} that replicate the tests performed in \cref{fig:posterior_samples,fig:bayesian_results} but with a generator trained using \cref{eq:FinalFormulation} \emph{as is}. 
Observe that the inter-sample variance is very low, e.g., the conditional mean image in \cref{fig:posterior_samples_no_minibatch} is still very noisy as compared to corresponding images in \cref{fig:bayesian_results}.

An explanation to this phenomena can be found in statistical learning theory.
Note that, regardless of the number of supervised training data points $(\signal_i, \data_i)$, the training data only provides a single $\RecSpace$-sample $\signal_i$ of the probability measure $\posterior{\data_i}$, which is the posterior at $\data_i$.
Since training data only provides a \emph{single sample} from $\posterior{\data_i}$, training by \cref{eq:FinalFormulation} will result in a generator that only learns how to generate the corresponding single sample thereby generating the same sample repeatedly (mode collapse)  \cite{Zaheer:2017aa}. 

The importance of addressing mode collapse is clearly illustrated in \cref{fig:posterior_samples_no_minibatch,fig:bayesian_results_no_minibatch}.
One approach to avoid mode collapse is to let the discriminator in \cref{eq:WassersteinFormulationNonParam} see multiple samples from $\posterior{\data_i}$, which leads to the idea of mini-batch discriminators \cite{NIPS2016_6125,karras2018progressive}.
Such an approach is not possible in Bayes inversion since training data only provides access to a \emph{single} model parameter $\signal$ for each data $\data$.
In the following we describe a new \emph{conditional mini-batch discriminator} that is better at avoiding mode collapse in the Bayesian inversion setting. 

\paragraph{Conditional \ac{WGAN} discriminator}
The idea is to let the discriminator distinguish between unordered pairs in $\RecSpace$ containing either the model parameter or random samples generated by the generative model. 
To formalize this, the generative model is trained using the following generalization of \cref{eq:WGANFormulation}:
\begin{equation}\label{eq:WGANFormulationMiniBatch}
  \inf_{\GenProb \in \GenProbMapSpace}\,
    \Expect_{\stdata \sim \dataprob} \biggl[
        \Wasserstein
	\Bigl(
	\GenProb(\stdata) \otimes \GenProb(\stdata)
	,
	\frac{1}{2}\bigl( \posterior{\data} \otimes \GenProb(\stdata) \bigr) \oplus \frac{1}{2}\bigl( \GenProb(\stdata) \otimes \posterior{\data} \bigr)
	\Bigr)
\biggr]	
\end{equation}
where $\oplus$ denotes usual summation of measures. 
Next, we show that one may train a generative model based on \cref{eq:WGANFormulationMiniBatch} instead of \cref{eq:WGANFormulation}.
The former lets the discriminator see more than a single sample from the posterior, so the resulting learned generator is much less likely to suffer from mode collapse, see \cref{fig:bayesian_results_no_minibatch} for an empirically confirmation of this. 
\begin{claim}\label{prop:minibatch_discriminator}
A generative model $\GenProb \colon \DataSpace \to \PClass_{\RecSpace}$ solves \cref{eq:WGANFormulationMiniBatch} iff it solves \cref{eq:WGANFormulation}.
\end{claim}
\begin{proof}
Let $\data \in \DataSpace$ be fixed and consider the objective in \cref{eq:WGANFormulationMiniBatch}:
\begin{multline*}
\Wasserstein\Bigl(
  \GenProb(\data) \otimes \GenProb(\data),
  \frac{1}{2}\bigl( \posterior{\data} \otimes \GenProb(\data) \bigr) 
  \oplus 
  \frac{1}{2}\bigl( \GenProb(\data) \otimes \posterior{\data} \bigr)
\Bigr)  
\\
=
\Wasserstein\Bigl(
  \frac{1}{2} \bigl(
	  \GenProb(\data) \otimes \GenProb(\data),
	  \posterior{\data} \oplus \GenProb(\data) \bigr) \otimes \frac{1}{2} \bigl( \posterior{\data} \oplus \GenProb(\data) \bigr)
\Bigr)  
\\
\propto
\Wasserstein\Bigl(
  \frac{1}{2} \bigl( \GenProb(\data) \otimes \GenProb(\data) \bigr),
  \frac{1}{2} \bigl( \posterior{\data} \otimes \posterior{\data} \bigr)
\Bigr).  
\end{multline*}
The last equality above follows from subtracting the measure $\frac{1}{2} \bigl( \GenProb(\data) \otimes \GenProb(\data) \bigr)$ from both arguments in the Wasserstein metric and utilizing its translation invariance (which is easiest to see in the Kantorovich-Rubenstein characterization).
Next, 
\[ \Wasserstein\Bigl(
  \frac{1}{2} \bigl( \GenProb(\data) \otimes \GenProb(\data) \bigr),
  \frac{1}{2} \bigl( \posterior{\data} \otimes \posterior{\data} \bigr)
\Bigr) 
\propto 
\Wasserstein\bigl(\GenProb(\data), \posterior{\data} \bigr),
\]
so a generative model solves \cref{eq:WGANFormulationMiniBatch} if and only if it solves \cref{eq:WGANFormulation}.
\end{proof}
Note that in the proof of \cref{prop:minibatch_discriminator}, we implicitly assume the Wasserstein distance can be defined on any pair of positive Radon measures with equal mass.
This is a trivial extension of the original definition of the Wasserstein distance, which assumes the domain is a pair of probability measures. It is worth noting that one can define ``optimal transportation''-like distances between arbitrary positive Radon measures \cite{Chizat:2017aa,Chizat:2015aa}.

To proceed, we need to rewrite the training in \cref{eq:WGANFormulationMiniBatch} so that it becomes more tractable, e.g., by removing the explicit appearance of the unknown posterior and probability measure for data.
To do that, we yet again resort to the Kantorovich-Rubenstein duality \cref{eq:WassersteinKL}.
When applied to \cref{eq:WGANFormulationMiniBatch}, it yields 
\begin{equation}\label{eq:WassersteinFormulationNonParamMiniBatch}
\inf_{\GenProb \in \GenProbMapSpace}\,
  \Expect_{\stdata \sim \dataprob} \Biggl[  
     \sup_{\Discriminator \in \DiscrSpace}\,\,
       \Expect_{
       	\subalign{& (\stsignal_1,\stsignal_2) \sim \rho(\stdata) \\
       	  & (\stsignalother_1,\stsignalother_2) \sim 
       	\GenProb(\stdata) \otimes \GenProb(\stdata)}
       	}\!
       	\Bigl[ 
        \Discriminator\bigl((\stsignal_1,\stsignal_2),\stdata \bigr)  
        -
        \Discriminator\bigl((\stsignalother_1,\stsignalother_2),\stdata\bigr) \Bigr] 
  \Biggr].
\end{equation}
Here, $\rho(\data) := \frac{1}{2}\bigl( \posterior{\data} \otimes \GenProb(\stdata) \bigr) \oplus \frac{1}{2}\bigl( \GenProb(\stdata) \otimes \posterior{\data} \bigr)$ is a probability measure on $\RecSpace \times \RecSpace$ and $\DiscrSpace$ are measurable maps $\Discriminator \colon (\RecSpace \times \RecSpace) \times \DataSpace \to \Real$ that are 1-Lipschitz w.r.t. its $(\RecSpace \times \RecSpace)$-variable. 
Next, the same arguments used to rewrite \cref{eq:WassersteinFormulationNonParam} as \cref{eq:FinalNonParamFormulation} can also be used to rewrite \cref{eq:WassersteinFormulationNonParamMiniBatch} as 
\begin{equation}\label{eq:FinalNonParamFormulationMiniBatch}
\inf_{\GenProb \in \GenProbMapSpace}\,
\Biggl\{
    \sup_{\Discriminator \in \DiscrSpace} 
    \Expect_{\subalign{ &(\stsignal,\stdata) \sim \jointlaw \\ & \stsignalother_1, \stsignalother_2 \sim \GenProb(\stdata)}} 
    \biggl[
      \frac{1}{2} \Bigl( 
      \Discriminator\bigl((\stsignal, \stsignalother_2),\stdata \bigr) + \Discriminator\bigl((\stsignalother_1,\stsignal),\stdata \bigr) \Bigr)
      - 
      \Discriminator\bigl( (\stsignalother_1,\stsignalother_2), \stdata \bigr) 
    \biggr]    
\Biggr\}.
\end{equation}
In contrast to \cref{eq:WassersteinFormulationNonParamMiniBatch}, the formulation in \cref{eq:FinalNonParamFormulationMiniBatch} makes no reference to the posterior $\posterior{\data}$ nor the probability measure $\dataprob$ for data.
Instead, it involves an expectation w.r.t. the joint law $\jointlaw$, which in a practical setting can be replaced by its empirical counterpart given from supervised training data in \cref{eq:TData}.

The final step is to introduce parameterizations for the generator and discriminator. 
The generator is parametrized as in \cref{eq:GenArchitecture}, whereas the parametrized family $\DiscrSpace := \{ \Discriminator_{\discrparam} \}_{\discrparam \in \DiscrParamSet}$ of discriminators are measurable mappings of the type $\Discriminator_{\discrparam} \colon (\RecSpace \times \RecSpace) \times \DataSpace \to \Real$ that are 1-Lipschitz in the $(\RecSpace \times \RecSpace)$-variable.
Inserting these parametrizations into \cref{eq:FinalNonParamFormulationMiniBatch} results in 
\begin{multline}\label{eq:FinalFormulationMiniBatch}
(\genparam^*,\discrparam^*) \in \argmin_{\genparam \in \GenParamSet}\,
  \Biggl\{
    \sup_{\discrparam \in \DiscrParamSet}\,\,
    \Expect_{\subalign{& (\stsignal,\stdata) \sim \jointlaw \\ & \stgenvar_1, \stgenvar_2 \sim \genvarprob}} 
    \biggl[
      \frac{1}{2} \Bigl( \Discriminator_{\discrparam}\bigl( \bigl(\stsignal, \Generator_{\genparam}(\stgenvar_2,\stdata) \bigr),\stdata \bigr)  
      +
      \Discriminator_{\discrparam}\bigl( \bigl( \Generator_{\genparam}(\stgenvar_1,\stdata), \stsignal \bigr),\stdata \bigr)
       \Bigr)
       \\
      -
      \Discriminator_{\discrparam}\bigl( \bigl(\Generator_{\genparam}(\stgenvar_1,\stdata),\Generator_{\genparam}(\stgenvar_2,\stdata) \bigr), \stdata\bigr) 
    \biggr]
  \Biggr\}.    
\end{multline}
Note again that the unknown joint law $\jointlaw$ in \cref{eq:FinalFormulationMiniBatch} is replaced by its empirical counterpart given from the training data in \cref{eq:TData}.

\begin{figure}[t]
	\centering
	\includegraphics[height=0.187\textheight]{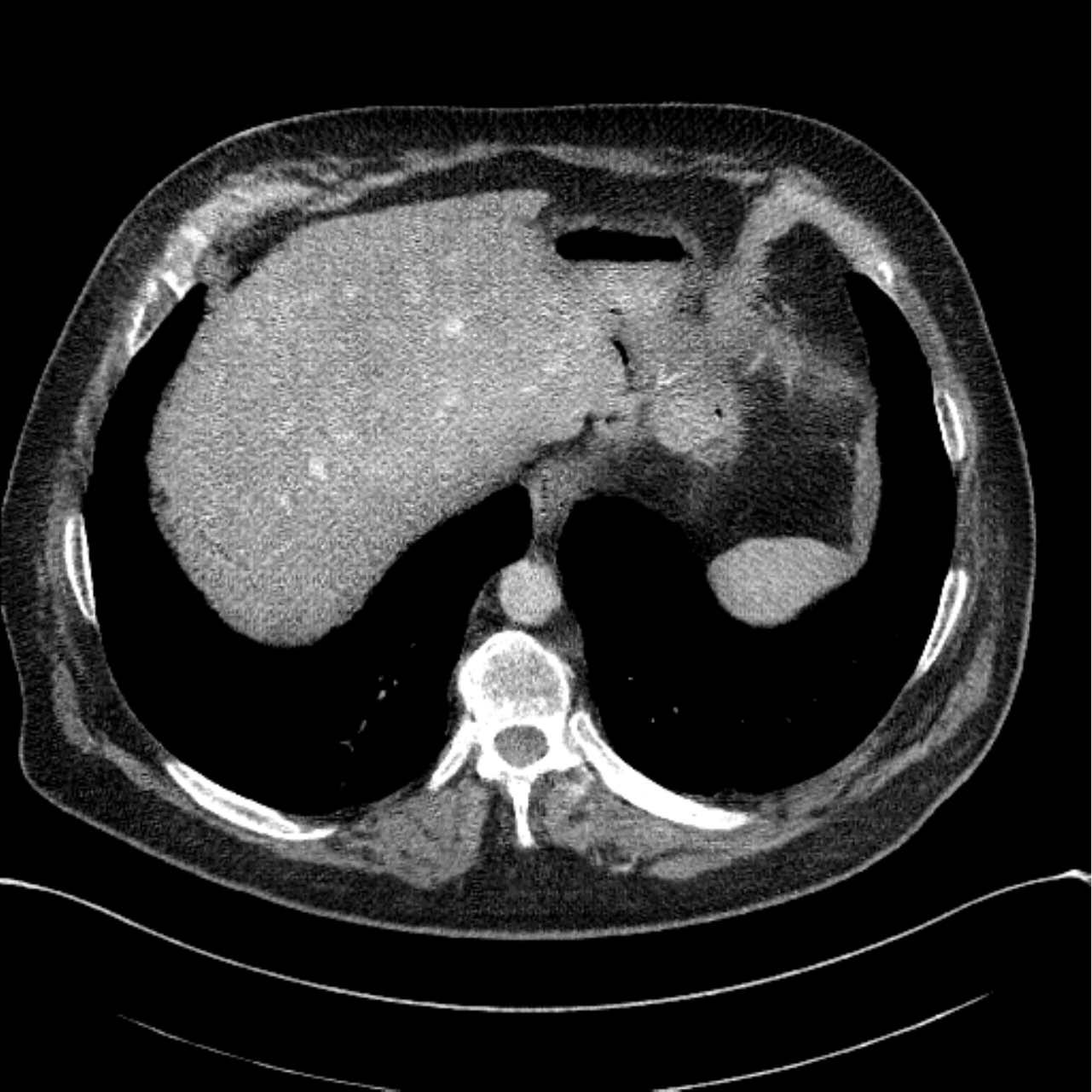}
	\includegraphics[height=0.187\textheight]{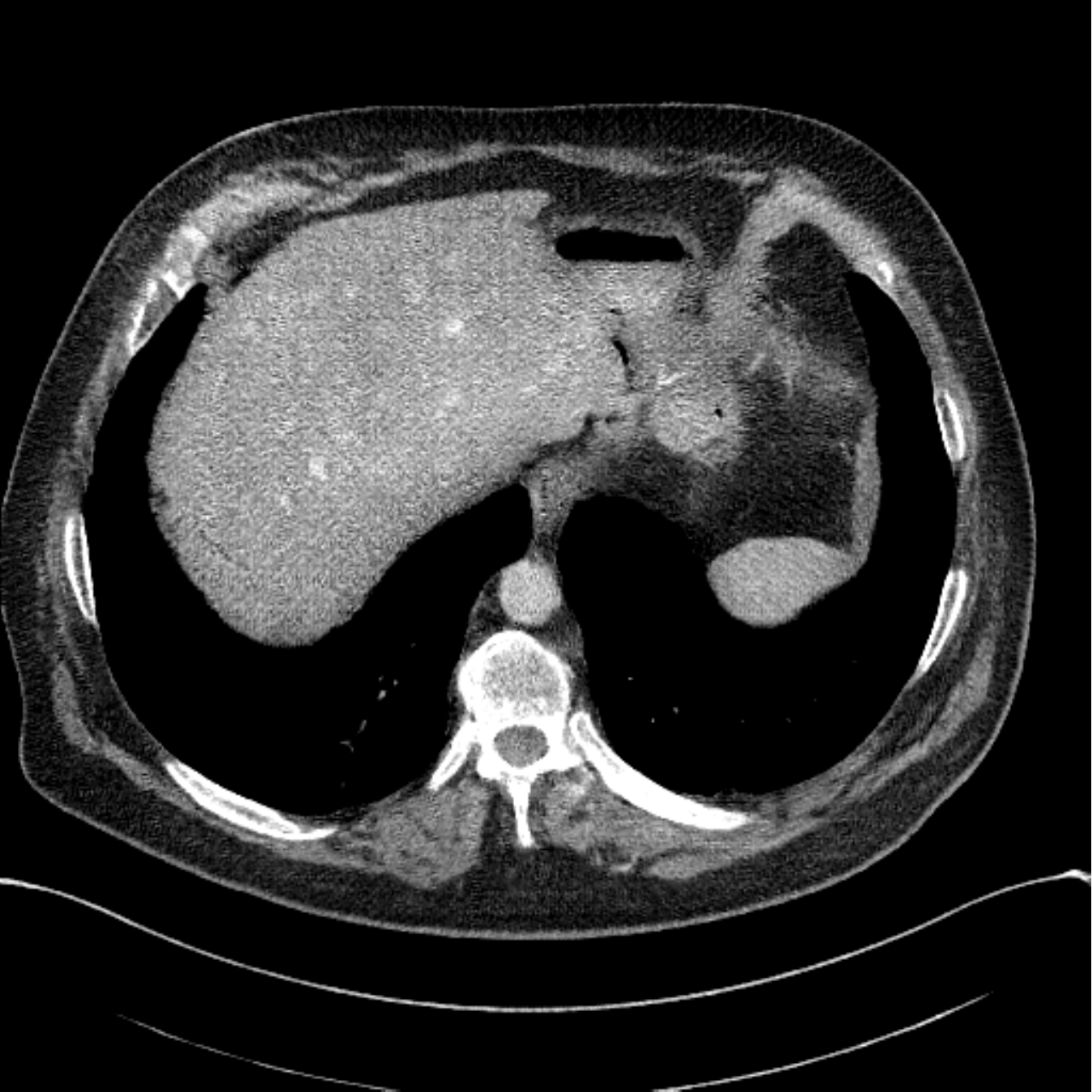}
	\includegraphics[height=0.187\textheight]{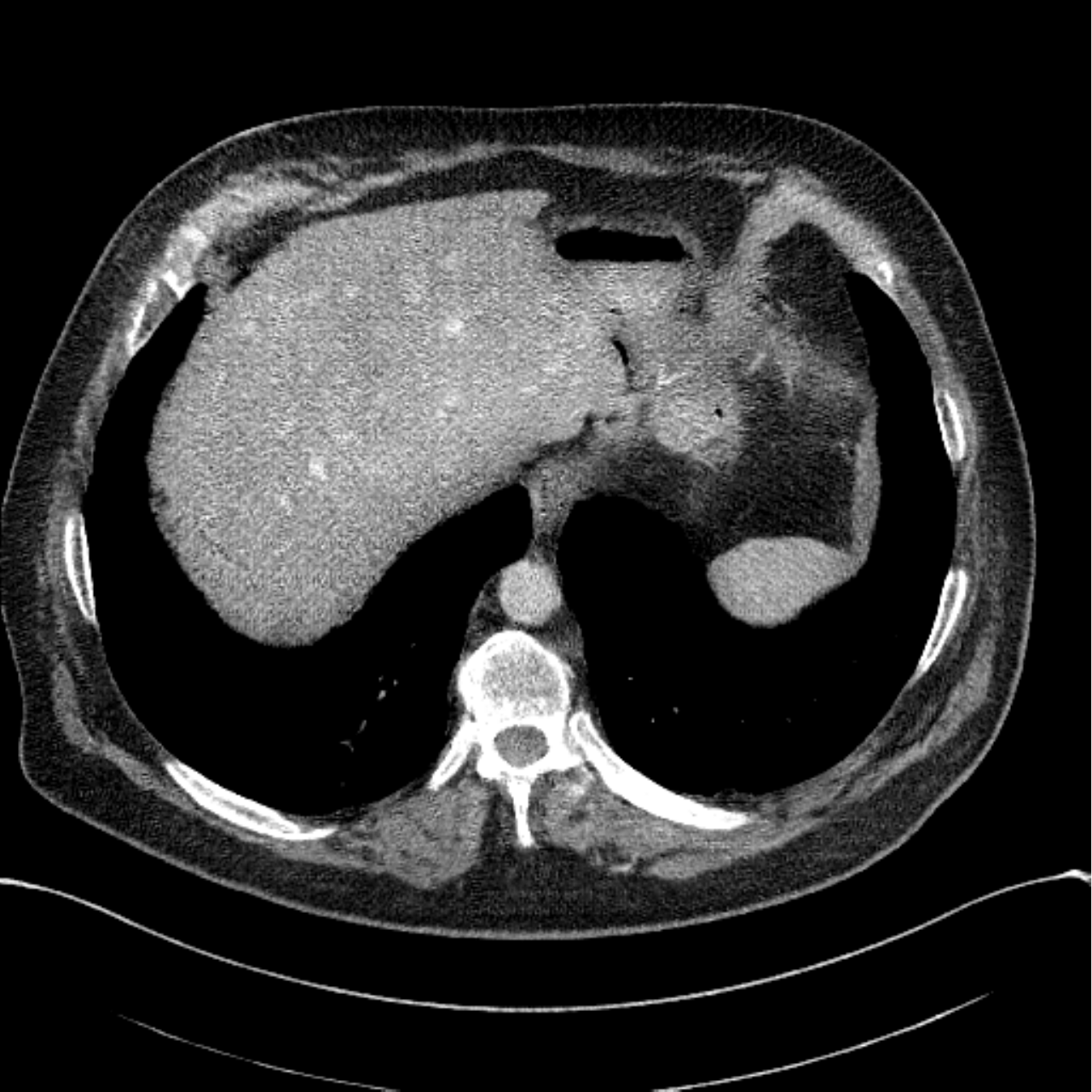}
	\\[0.25em]
	\includegraphics[height=0.187\textheight]{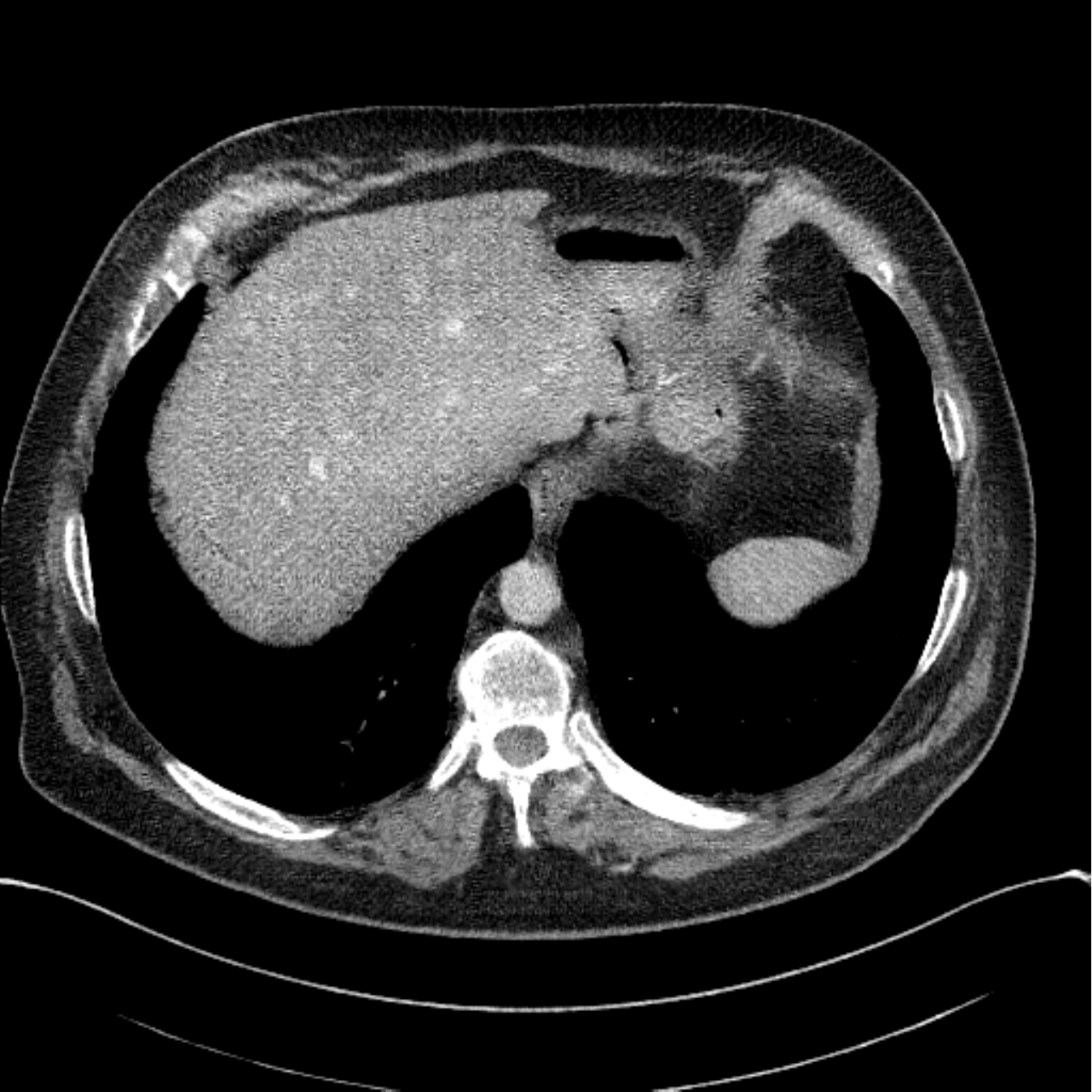}
	\includegraphics[height=0.187\textheight]{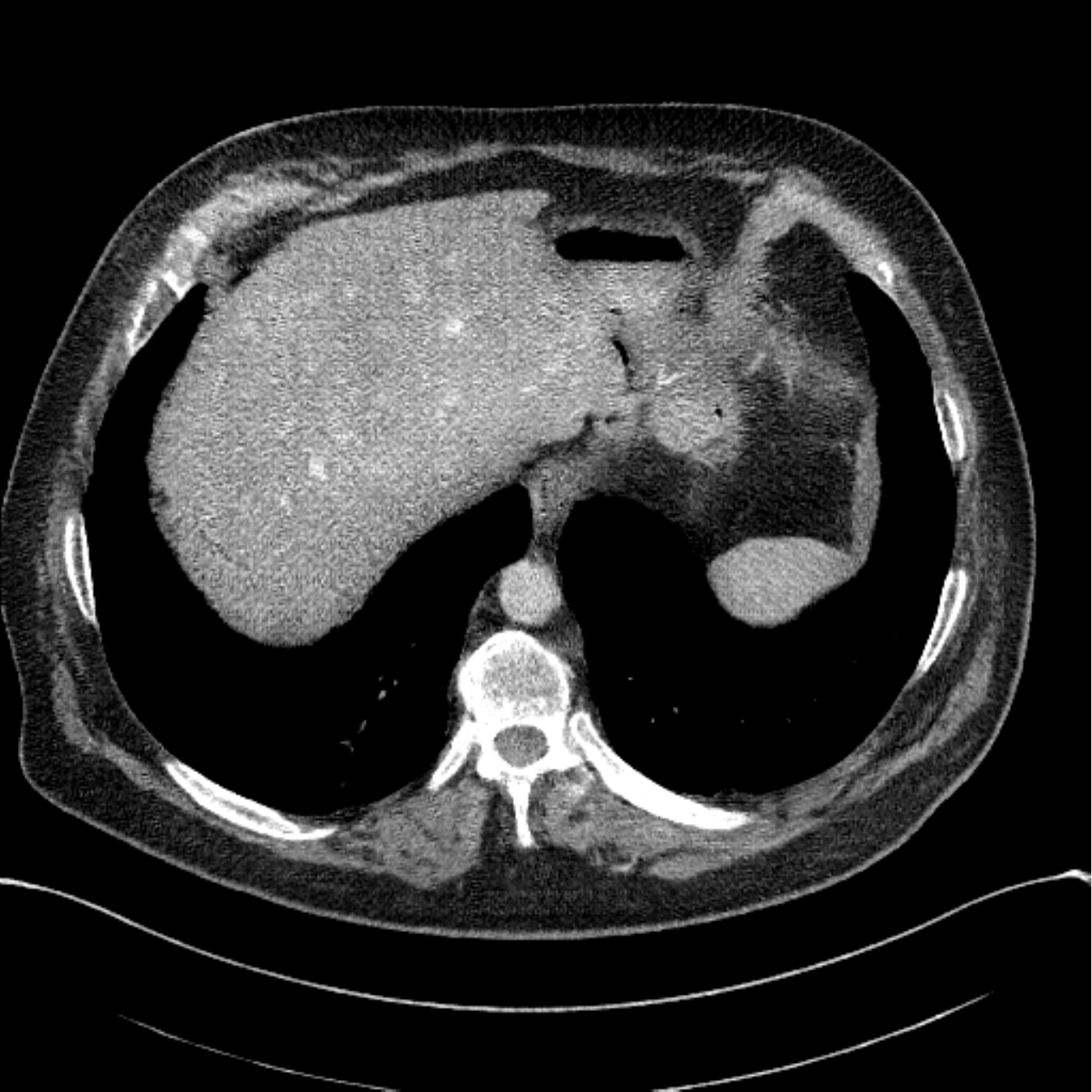}
	\includegraphics[height=0.187\textheight]{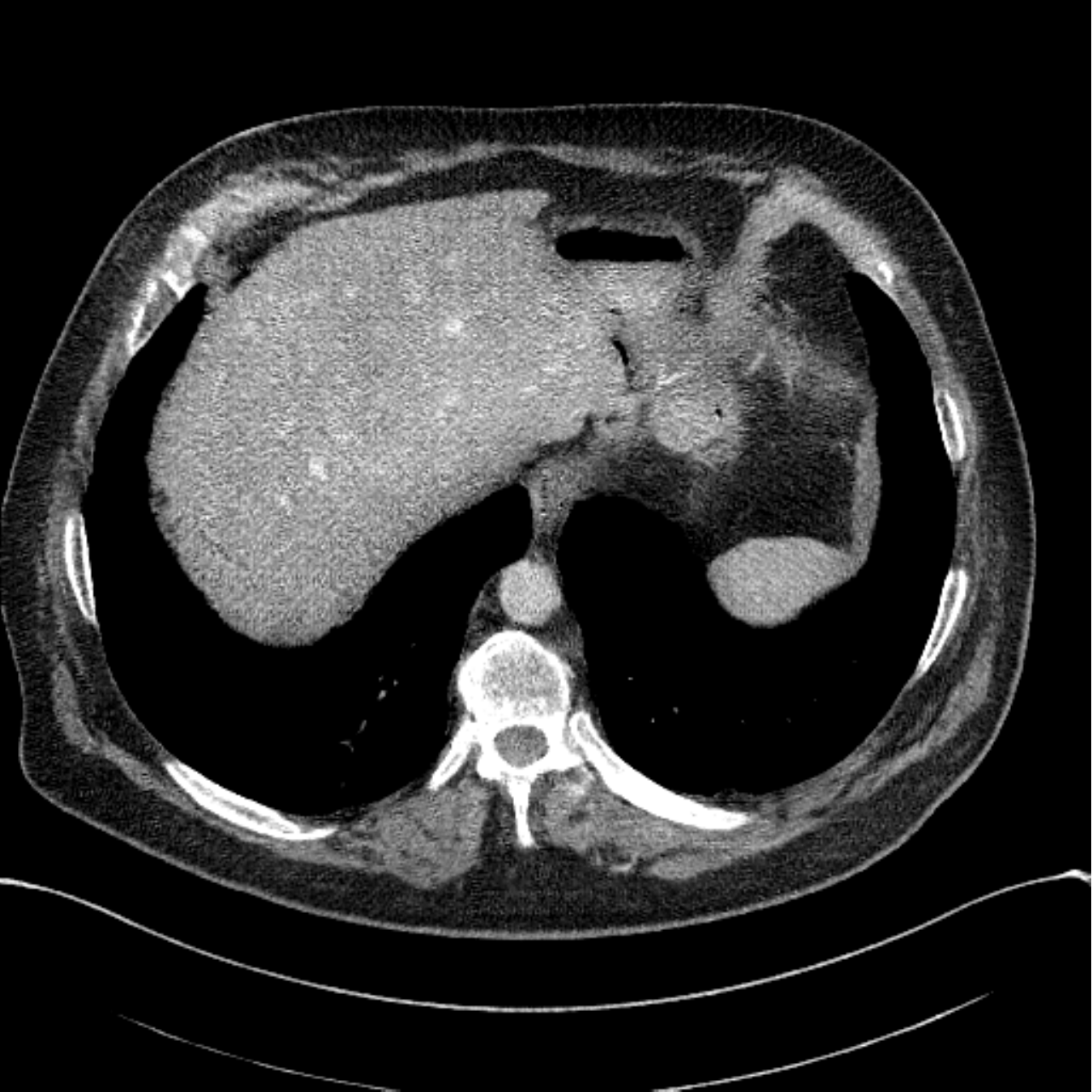}
	\caption{Replication of \cref{fig:posterior_samples} without conditional \ac{WGAN} discriminator shown using the same intensity window. Observe that there is practically no inter-sample variability due to mode collapse, confirming that the conditional \ac{WGAN} discriminator is essential for posterior sampling.}
	\label{fig:posterior_samples_no_minibatch}
\end{figure}

\begin{figure}[t]
	\centering	
	\begin{tabular}{c c c c c c}
		& Posterior sampling & Direct estimation & \hskip0.02\linewidth No cond.\@ GAN discr. &&
		\\
		\raisebox{5em}[0pt][0pt]{\rotatebox[origin=c]{90}{Mean}} 
		\hspace{-4mm} &
		\includegraphics[height=0.15\textheight]{figures/estimators/sample_mean}
		\hspace{-4mm} &
		\includegraphics[height=0.15\textheight]{figures/estimators/denoise_result}
		\hspace{-4mm} &
		\includegraphics[height=0.15\textheight]{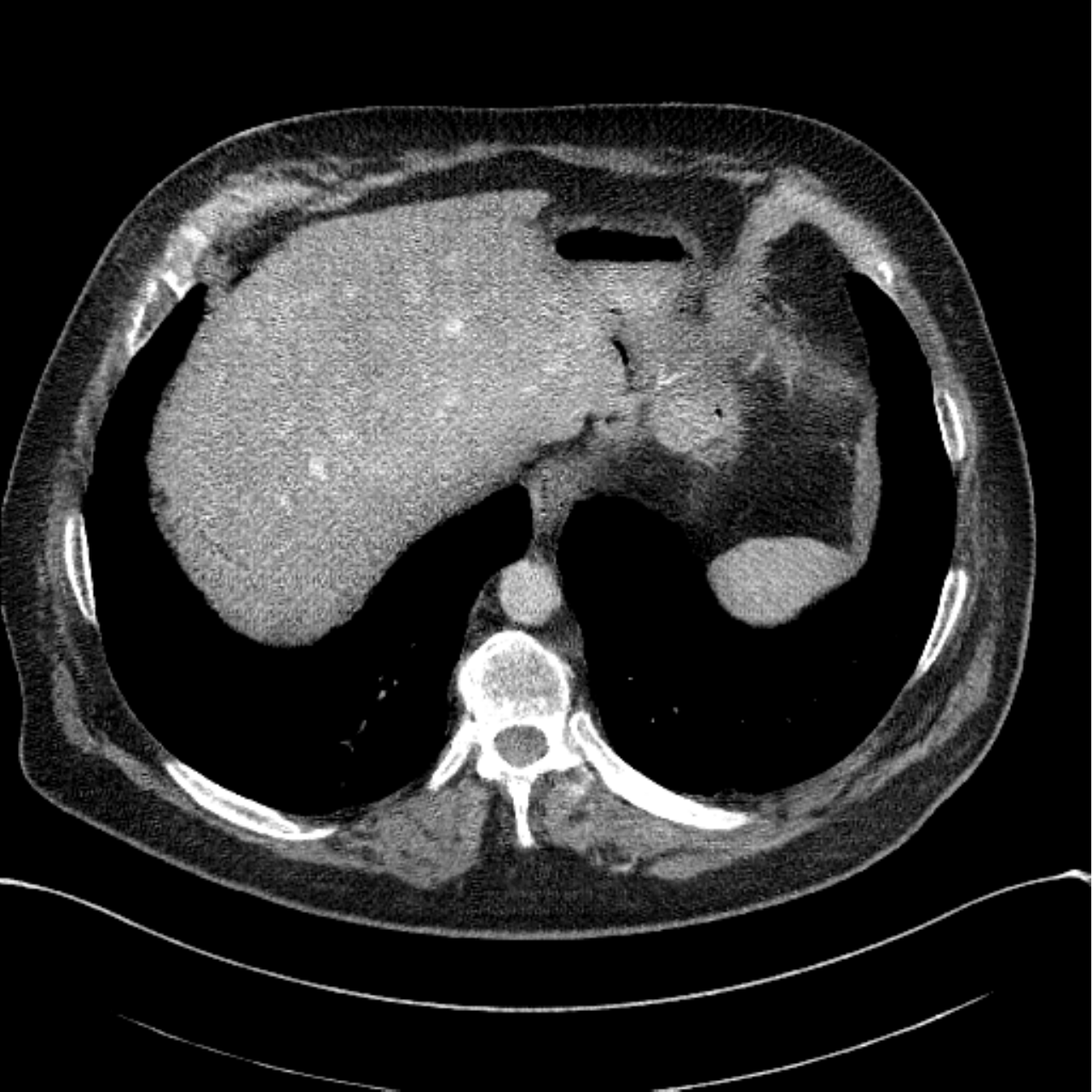}
		\hspace{-4mm} &
		\includegraphics[width=0.02\textwidth, height=0.15\textheight]{figures/estimators/bone}
		\hspace{-4mm} &
		\hspace{-2mm}
		$\overset{\raise0.123\textheight\hbox{\footnotesize\text{200 HU}}}{\footnotesize\text{-150 HU}}$
		\\
		\raisebox{5em}[0pt][0pt]{\rotatebox[origin=c]{90}{pStd}} 
		\hspace{-4mm} &
		\includegraphics[height=0.15\textheight]{figures/estimators/sample_std}
		\hspace{-4mm} &
		\includegraphics[height=0.15\textheight]{figures/estimators/direct_std}
		\hspace{-4mm} &
		\includegraphics[height=0.15\textheight]{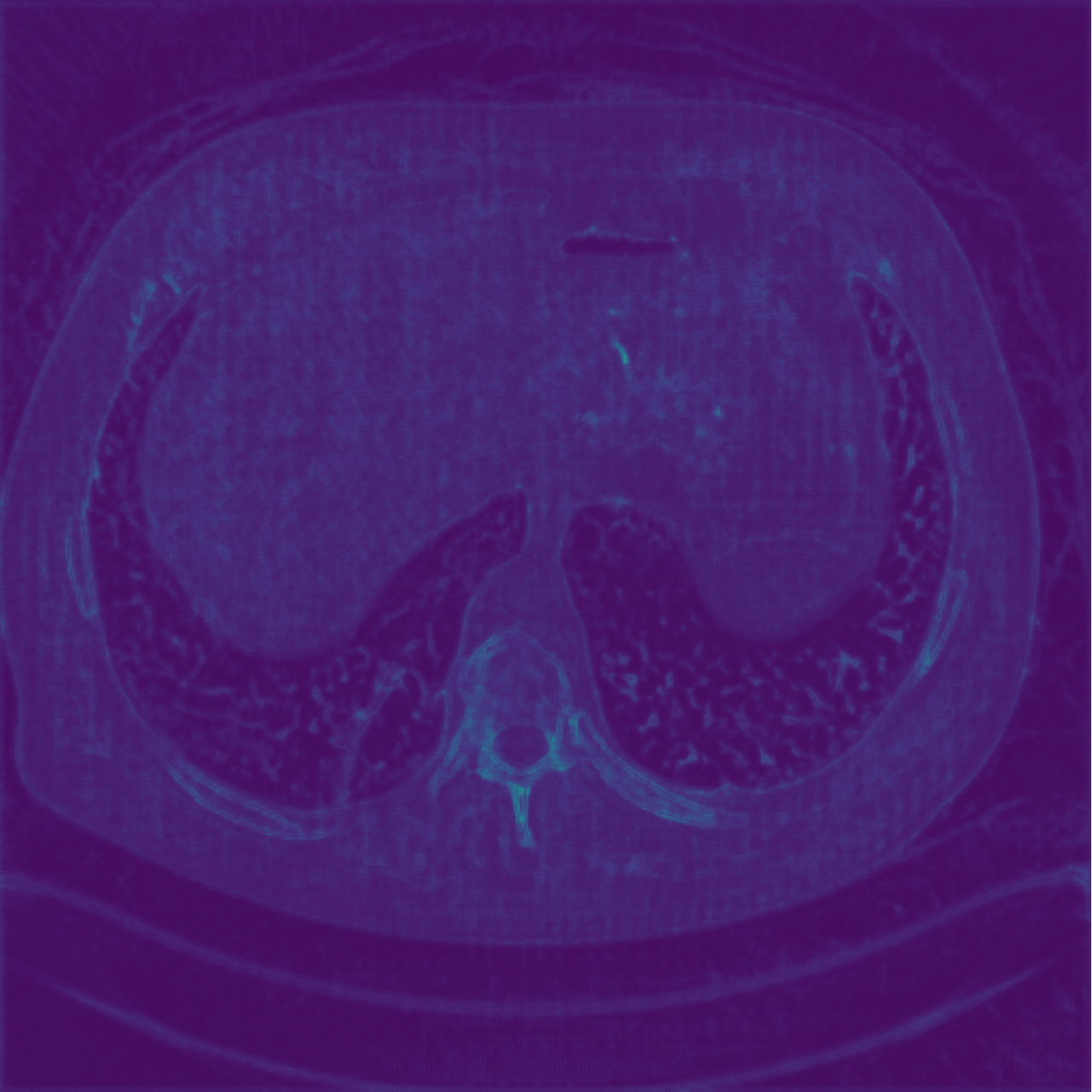} 
		\hspace{-4mm} &		
		\includegraphics[width=0.02\textwidth, height=0.15\textheight]{figures/estimators/viridis}
		\hspace{-4mm} &
		\hspace{-3mm}
		$\overset{\raise0.123\textheight\hbox{\footnotesize\text{50 HU}}}{\footnotesize\text{0 HU}}$
	\end{tabular}
	\caption{Replication of \cref{fig:bayesian_results} also showing (right most column) the sample mean and sample point-wise standard deviation (pStd) when the conditional \ac{WGAN} discriminator is not used. The standard deviation grossly underestimated due to mode collapse.}
	\label{fig:bayesian_results_no_minibatch}
\end{figure}

\subsection{Deep Direct Estimation}\label{app:sampling_free}
The aim here is to show how an appropriately trained deep neural network can be used for approximating a wide range of non-randomized decision rules (estimators) associated, e.g., with uncertainty quantification.
This differs from the posterior sampling approach (\cref{sec:WGAN,app:WGANMath}) where such estimators are computed empirically by sampling from a trained \ac{WGAN} generator.

The idea is to extend the approach in \cite{Adler:2017aa,Adler:2018aa} for learning estimators that minimizes Bayes risk so that it applies to a wider class of estimators.
Our starting point is a well known proposition from probability theory that characterizes the minimizer of the mean squared error loss.
\begin{proposition}\label{prop:condexp}
	Assume that $\DataSpace$ be a measurable space, $\OtherSpace$ is a measurable Hilbert space, and $\stdata$ and $\stotherdata$ are $\DataSpace$- and $\OtherSpace$-valued random variables, respectively.
	Then, the conditional expectation $h^*(\data) := \Expect \bigl[ \stotherdata \mid \stdata = \data \bigr]$ solves 
	\[
		\min_{h \colon \DataSpace \to \OtherSpace}  
		\Expect
		\Bigl[ 
			\bigl\|h(\stdata)  - \stotherdata \bigr\|_{\OtherSpace}^2 
		\Bigr].
	\]
	The minimization above is taken over all $\OtherSpace$-valued measurable functions on $\DataSpace$.	
\end{proposition}
\begin{proof}
	Let $h \colon \DataSpace \to \OtherSpace$ be any measurable function so
	\[
		\Expect
		\Bigl[ 
		\bigl\|h(\stdata)  - \stotherdata \bigr\|_{\OtherSpace}^2 
		\Bigr]
		=
		\Expect
		\Bigl[
		\Expect
		\bigl[ 
			\bigl\|h(\stdata)  - \stotherdata \bigr\|_{\OtherSpace}^2
			\mid
			\stdata
		\bigr]
		\Bigr].
	\]
	Next, $\OtherSpace$ is a Hilbert space so we can expand the squared norm:
	\begin{multline*}
		\bigl\|
		h(\stdata)  - 
		\stotherdata \bigr
		\|_{\OtherSpace}^2
		=\
		\bigl\|
		h(\stdata) -
		\Expect[\stotherdata \mid \stdata] +
		\Expect[\stotherdata \mid \stdata] -
		\stotherdata \bigr\|_{\OtherSpace}^2
		\\
		\shoveleft{\qquad
		=
		\bigl\|
		h(\stdata)  - \Expect[\stotherdata \mid \stdata] \bigr\|_{\OtherSpace}^2
		}\\
		+
		2
		\Bigl\langle
			h(\stdata) - \Expect[\stotherdata \mid \stdata],
			\Expect[\stotherdata \mid \stdata] - \stotherdata
		\Bigr\rangle_\OtherSpace
		+
		\bigl\| \stotherdata  - \Expect[\stotherdata \mid \stdata] \bigr\|_{\OtherSpace}^2.		
	\end{multline*}
	By the law of total expectation and the linearity of the inner product, we get
	\begin{multline*}
		\Expect
		\Bigl[ 
			2
			\Bigl\langle
			h(\stdata) - \Expect[\stotherdata \mid \stdata],
			\Expect[\stotherdata \mid \stdata] - \stotherdata
			\Bigr\rangle_\OtherSpace
		\mid
		\stdata
		\Bigr]
		\\
		= 
		2
		\Bigl\langle
		h(\stdata) - \Expect[\stotherdata \mid \stdata],
		\Expect[\stotherdata \mid \stdata] - 
		\Expect[\stotherdata \mid \stdata]
		\Bigr\rangle_\OtherSpace
		\\
		=
		2
		\bigl\langle
		h(\stdata) - \Expect[\stotherdata \mid \stdata],
		0
		\bigr\rangle_\OtherSpace
		=
		0
	\end{multline*}
	and $\bigl\| \stotherdata  - \Expect[\stotherdata \mid \stdata] \bigr\|_{\OtherSpace}^2$ is independent of $h(\stdata)$. 
	Combining all of this gives 
	\[
		\argmin_{h \colon \DataSpace \to \OtherSpace}  
		\Expect
		\Bigl[ 
			\bigl\|h(\stdata)  - \stotherdata \bigr\|_{\OtherSpace}^2 
		\Bigr]
		=
		\argmin_{h \colon \DataSpace \to \OtherSpace}  
		\Expect
		\Bigl[ 
		\bigl\|
		h(\stdata)  - \Expect[\stotherdata \mid \stdata]
		\bigr\|_{\OtherSpace}^2
		\Bigr]
	\]
	where $h^*(y) = \Expect[\stotherdata \mid \stdata]$ is the solution to the right hand side.
\end{proof}
\Cref{prop:condexp} implies in particular that minimizing Bayes risk with a loss given by the mean squared error amounts to computing the conditional mean.
This result \emph{does not} hold when the loss is the 1-norm, which would give the conditional \emph{median} instead of the conditional mean.
In a finite dimensional setting, \cref{prop:condexp} holds also when the loss is any functional that is the Bregman distance of a convex functional \cite{Banerjee:2005aa}.

In the context of Bayesian inversion, $\stsignal$ and $\stdata$ are the $\RecSpace$– and $\DataSpace$-valued random variables generating the model parameter and data, respectively. 
\Cref{prop:condexp} is then the starting point for studying the relation between the \ac{MAP} and conditional mean estimates \cite{Burger:2014aa}.
In our setting, if $h^* \colon \DataSpace \to \RecSpace$ is the estimator that minimizes Bayes risk using squared loss, then \cref{prop:condexp} (with $\stotherdata:=\stsignal$) implies that 
\begin{equation}
h^* \in \argmin_{h \colon \DataSpace \to \OtherSpace}  
		\Expect
		\Bigl[ 
			\bigl\|h(\stdata)  - \stotherdata \bigr\|_{\OtherSpace}^2 
		\Bigr]
\implies 	
	h^*(\data) 
	= 
	\Expect
	\bigl[ 
		\stsignal \mid \stdata = \data 
	\bigr]
	\quad\text{for } \data \in \DataSpace.
	\label{eq:condexp_recon}
\end{equation}
Since neural networks are universal function approximators, training a neural network using the mean squared error as loss yields an approximation of the conditional mean. 

By selecting some other regression target $\stotherdata$ we can approximate estimators other than the conditional mean.
As an example, let us consider the point-wise conditional variance which is defined as
\begin{equation}
\PointwiseVariance\bigl[ \stsignal \mid \stdata=\data \bigr]
:=
\Expect
\Bigl[
	\bigl(
		\stsignal - \Expect[ \stsignal \mid \stdata=\data ] 
	\bigr)^2 
	\mid 
	\stdata=\data 
\Bigr]
\label{eq:pointwise_variance}
\end{equation}
In the following, we show how the (point-wise) conditional variance can be estimated directly using a neural network trained against supervised data, similar to how we estimate the conditional mean. 
The key step is to re-write the conditional variance as a minimizer of the expectation of some scalar objective w.r.t. the joint law of $(\stsignal,\stdata)$.
\begin{proposition}\label{prop:condvar}
	Assume that $\DataSpace, \RecSpace$ are measurable spaces and that $\RecSpace$ is a Hilbert space. 
	The point-wise variance is then characterized by
	\begin{equation}\label{eq:CondVarExp}
	\PointwiseVariance[\stsignal \mid \stdata = \data ] 
	\in 
	\argmin_{h \colon \DataSpace \to \RecSpace}
	\Expect
	\biggl[ 
	\Bigl\|
	h(\stdata) - 
	\bigl(
	\stsignal - \Expect\bigl[ \stsignal \mid \stdata=\data\bigr]
	\bigr)^2
	\Bigr\|_\RecSpace^2 
	\biggr]
	\end{equation}
	where the minimization is taken over all $\RecSpace$-valued measurable functions on $\DataSpace$.	
\end{proposition}
The proof follows by applying \cref{prop:condexp} with $\stotherdata := \bigl(\stsignal - \Expect\bigl[\stsignal \mid \stdata=\data\bigr]\bigr)^2$, which yields
	\[
		\argmin_{h \colon \DataSpace \to \RecSpace}
		\Expect
		\biggl[ 
		\Bigl\|
		h(\stdata) - 
		\bigl(
		\stsignal - \Expect\bigl[ \stsignal \mid \stdata=\data\bigr]
		\bigr)^2
		\Bigr\|_\RecSpace^2 
		\biggr]
		=
		\Expect
		\Bigl[ 
			\bigl(
				\stsignal - \Expect\bigl[ \stsignal \mid \stdata=\data\bigr]
			\bigr)^2
			\mid
			\stdata = \cdot
		\Bigr].
	\]

In practice we don't have direct access to samples from $\bigl(\stsignal - \Expect\bigl[\stsignal \mid \stdata=\data\bigr]\bigr)^2$, so this cannot be applied as is since we cannot compute the expectation in \cref{eq:CondVarExp}.
However, if there is access to supervised training data as in \cref{eq:TData}, then the conditional expectation in \cref{eq:CondVarExp} can be approximated by a deep neural network trained according to \cref{eq:condexp_recon}, $\Expect\bigl[ \stsignal \mid \stdata=\data\bigr] \approx \RecOp_{\recparam}(\data)$. From this training data one can then generate ``new'' training data of the form 
\[
	\Bigl(
		\bigl(\signal_i - \RecOp_{\recparam}(\data_i)\bigr)^2,
		\data_i
	\Bigr) \in \RecSpace \times \DataSpace
	\quad\text{where $(\signal_i,\data_i) \in \RecSpace \times \DataSpace$ is from \cref{eq:TData}.}
\]
This training data is random samples from a $(\RecSpace \times \DataSpace)$-valued random variable that approximately has required distribution.

Finally, the minimization in \cref{eq:CondVarExp} can be restricted to $\RecSpace$-valued measurable functions on $\DataSpace$ that are parametrized by another deep neural network architecture $\DeepVariation_{\variationparam} \colon \DataSpace \to \RecSpace$. 
Hence, the conditional point-wise variance can be estimated as $\PointwiseVariance[\stsignal \mid \stdata = \data ]  \approx \DeepVariation_{\variationparam^*}(\data)$ where $\variationparam^*$ is obtained from solving the following  training problems:
\begin{align*}
  \recparam^* \in
  & \argmin_{\recparam}
	\biggl\{
	\Expect_{(\stsignal,\stdata)} 
	\Bigl[ \bigl\|
	\stsignal - \RecOp_{\recparam}(\stdata)
	\bigr\|_\RecSpace^2 \Bigr] 
	\biggr\}
\\
\variationparam^* \in
  & \argmin_{\variationparam}
	\biggl\{ 
	\Expect_{(\stsignal,\stdata)} 	
	\Bigl[ \Bigl\|
	\DeepVariation_{\variationparam}(\stdata) - 
	\bigl(
	\stsignal - \RecOp_{\recparam^*}(\stdata)
	\bigr)^2
	\Bigr\|_\RecSpace^2 \Bigr] 
	\biggr\}.
\end{align*}

Direct estimation is a sample free method that has several advantages against posterior sampling. 
First, they are much easier to train, \acp{GAN} that are used for posterior sampling are known for being notoriously hard to train whereas learned iterative methods that underly direct estimation can be trained using standard approaches. 
Next, they are much faster. Evaluating a trained deep neural network for direct estimation requires roughly as much computational power as generating a single sample of the posterior in posterior sampling. 
Since posterior sampling requires several samples to get sufficient statistics, they will require an order of magnitude more time.

A downside with direct estimation is that a separate neural network has to be constructed and trained for each estimator. 
This is especially problematic in cases where we need to answer patient-specific questions that are perhaps unknown during training. 
Another is that direct estimation as introduced here can only be used for estimators that can be re-written as a minimizer of the expectation of some scalar objective w.r.t. the joint law of $(\stsignal,\stdata)$. 
It is well known that conditional distributions can be approximated by Edgeworth expansions that in turn contain such terms  \cite{Pedersen:1979aa}, so in principle any posterior can be approximated in this manner by a series of direct estimations.
However, the computations quickly get complicated and the computational and training related advantages of direct estimation quickly diminishes. 

Finally, results and corresponding proofs as stated in this section are not fully rigorous in the function space setting. As an example, \cref{prop:condvar} would in such a setting involve the theory of higher moments of Banach space valued random variables \cite{Janson:2015aa}, which quickly involves elaborate measure theory. On the other hand, the proofs are straightforward in finite dimensional spaces.

\section{Implementation Details}
\label{sec:implementation_details}

\subsection{Training data}
\label{sec:training_data_details}
Training data is clinical 3D helical \ac{CT} scans from the Mayo Clinic Low Dose CT challenge \cite{AAPMLowDose}.
The data was obtained using a Siemens SOMATOM Definition AS+ scanner and consists of ten abdomen \ac{CT} scans of patients with predominantly liver and lung cancer obtained at normal dose. The scanner is a 64-slice cone beam helical \ac{CT} that further enhances longitudinal resolution by a periodic motion of the focal spot in the z-direction (z-flying focal spot acquisition) \cite{Flohr:2005aa}. The x-ray tube peak voltage (kVp) was 100--\unit{120}{\kilo\volt}, depending on patient size, the exposure time was \unit{500}{\milli\second} and the tube current was 230--\unit{430}{\milli\ampere}, again depending on patient size. 

The normal dose reconstructions are obtained by applying a \ac{FBP}-type of reconstruction scheme, provided by the manufacturer of the scanner, on the full data.
To obtain the low dose images, we first subsampled data and then added noise. The original data is acquired using a 3-PI acquisition geometry \cite{Bontus:2009aa}, meaning that the helical pitch is chosen to oversample each integration line by a factor of three. We sub-sampled the data by excluding the ``upper'' and ``lower'' pitch, which corresponds to data from 1-PI acquisition geometry. This results in a sub-sampling of $33\%$. Furthermore, we split each dataset into three independent datasets by using every third angle. 
This gives a further sub-sampling by $33\%$, for a total subsampling of $\approx 10$. In addition, we added Poisson noise to the data according to \cite{AAPMLowDose} until they corresponded to $2\%$ normal dose scans, i.e. roughly 1\,000~photons per pixel. While electron noise is significant at these dose levels, we chose not to model it.

Standard \ac{FBP} was applied to the above ultra low dose data with a Hann filter with cutoff $0.4$ and the filter frequency was chosen to maximize the \ac{PSNR} of the ultra low dose reconstructions.
The 2D slice size was set to $\unit{512 \times 512}{\pixels}$ with a reconstruction diameter of 370--\unit{440}{\milli\meter} (depending on patient size) and a slice-thickness of \unit{3}{\milli\meter}.
Note that the \ac{FBP} reconstruction operator is formally not information conserving when using a cutoff (information is irreversibly lost), which technically invalidates the claim in \cref{subsec:setup} that \ac{FBP} may be used as a pre-processing step without any information loss.
However, we did not observed any adverse effects in letting $\data$ represent \ac{FBP} reconstructions rather than \ac{CT} data. 

Finally, in order to (approximately) center the images, they were linearly scaled so that zero corresponds to $\unit{0}{\hounsfield}$ and $-1$ to $\unit{-1\,000}{\hounsfield}$.
In total,  supervised training data consisted of 6\,498 pairs of semi-independent 2D images at normal and ultra low dose. 
To further augment the training data during training, we applied random flips (left-right), rotations ($\pm \unit{10}{\degree}$), adding pixel-wise dequantization noise distributed according to $\mathcal{U}(0, 1)$ \hounsfield, and a random mean-value offset distributed according to $\mathcal{N}(0, 10)$ \hounsfield.

\subsection{Neural networks}
\label{sec:architecture}

For simplicity, all networks are based on a similar \ac{CNN} architecture that consists of the following three building blocks:
\begin{itemize}
	\item Averagepooling. Mapping an $2n \times 2n$ image to a $n \times n$ image by taking the average over $\unit{2 \times 2}{\pixel}$ blocks.
	\item Pixelshuffle (also ``space to depth'') \cite{PixelShuffle}. Mapping a $n \times n$ image with $4c$ channels to a $2n \times 2n$ image with $c$ channels by spatially spreading the channels into a $2 \times 2$ block.
	\item Residual blocks \cite{ResNet}. A single residual block consists of applying batch normalization to the input, followed by a nonlinearity, convolution, batch normalization, nonlinearity and finally a convolution. This is added to a $1 \times 1$ convolution of the result of the first batch normalization. Such a block is shown in \cref{fig:residual_unit}.
\end{itemize}
Furthermore, unless otherwise stated, the \ac{CNN} uses $3 \times 3$ convolutions and leaky ReLU ($\alpha = 0.2$) non-linearities \cite{LeakyRelu}. 

For the generator $\Generator_{\genparam} \colon  \GenSpace \times \DataSpace \to \RecSpace$, direct mean estimator $\RecOp_{\recparam} \colon \DataSpace \to \RecSpace$,  and direct variance estimator $\DeepVariation_{\variationparam} \colon \DataSpace \to \RecSpace$, we used an architecture similar to U-Net \cite{UNet} combining down-sampling followed by a residual block until the image is $\unit{8 \times 8}{\pixels}$. At this point we performed up-samplings combined with concatenating skip-connections until we reach the original $\unit{512 \times 512}{\pixel}$ resolution. The network architecture is illustrated in \cref{fig:unet}. For $\RecOp_{\recparam}$ and $\DeepVariation_{\variationparam}$ the input was simply the data $\data$. 
Regarding the generator, we let the random noise $\genvar$ be white noise on $\GenSpace := \RecSpace$, so $\Generator_{\genparam} \colon \RecSpace \times \RecSpace \to \RecSpace$.
For the generator and direct mean estimator we also added an additive skip-connection from $\data$ to result \cite{UnserInverse}.

Finally, the discriminator $\Discriminator_{\discrparam}$ is parametrized using a similar network architecture but stopped at the lowest resolution ($\unit{8\times 8}{\pixels}$) and finished with two fully connected layers (\cref{fig:discriminator}).

\begin{figure}[t]
	\centering
	\includegraphics[scale=0.6]{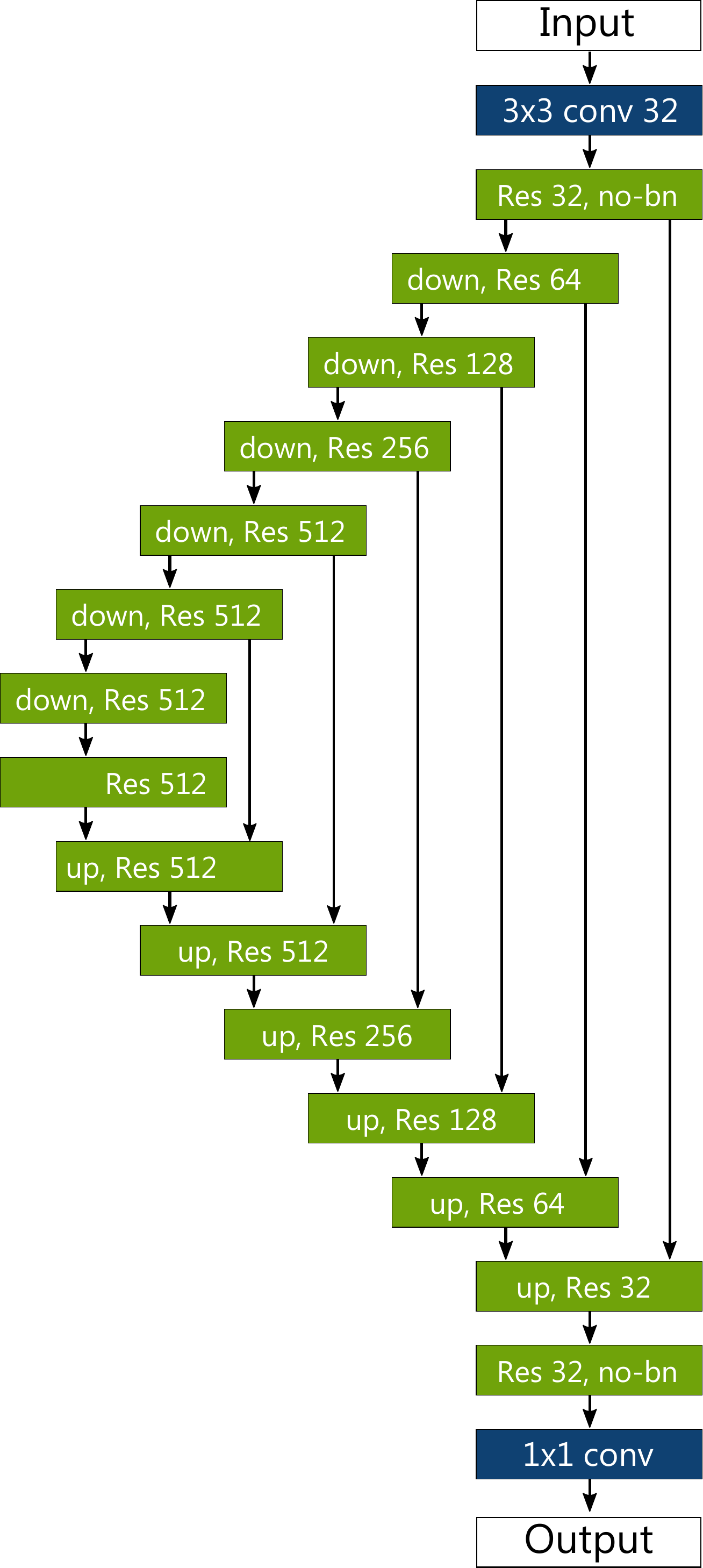}
	\caption{Residual U-Net network architecture. ``down'' indicates that a downsampling is done before the resblock. ``up'' indicates that an upsampling and a concatenation is done before the residual block. ``no-bn'' indicates that batch normalization was not used in that residual block. The input has resolution 512$^2$ while the smallest images have size 8$^2$.}
	\label{fig:unet}
\end{figure}
\begin{figure}[t]
	\centering
	\begin{subfigure}[t]{0.49\linewidth}
		\centering		
		\includegraphics[scale=0.6]{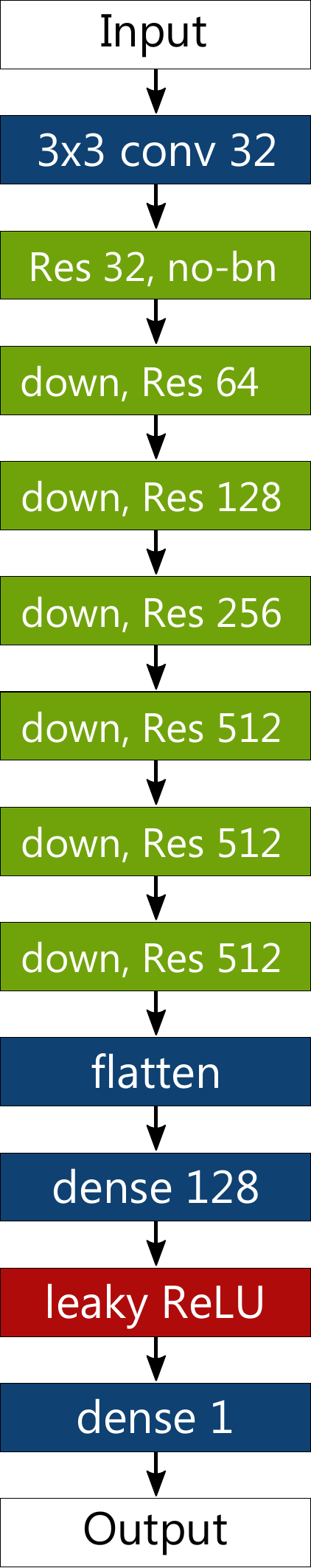}
		\caption{Discriminator network.}
		\label{fig:discriminator}
	\end{subfigure}
	\begin{subfigure}[t]{0.49\linewidth}
		\centering
		\includegraphics[scale=0.6]{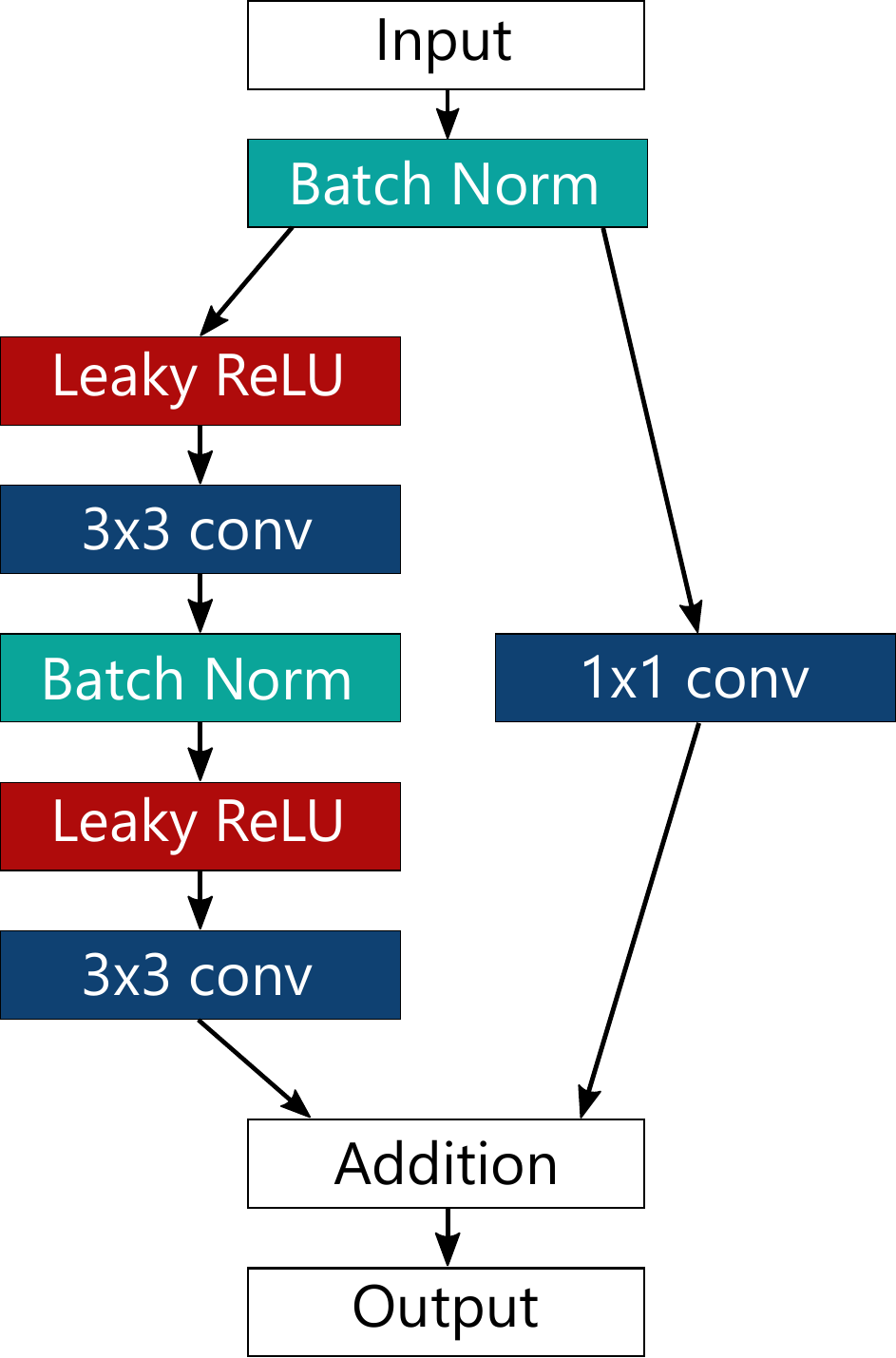}
		\caption{Residual unit used in the network. All convolutions have the same number of output channels.}
		\label{fig:residual_unit}
	\end{subfigure}
\end{figure}

\subsection{Training}
First, all training procedures involved applying a small $L_2$ regularization (weight decay) with constant $10^{-4}$ to complement the expected loss. 
Furthermore, $\widehat{\jointlaw}$ will denote the empirical probability measure derived from the supervised training data \cref{eq:TData} that has undergone data augmentation (\cref{sec:training_data_details}). 

\paragraph{Direct estimation} Training the networks in the direct estimation approach (\cref{app:sampling_free}) amounts to solving 
\begin{align*}
\recparam^* \in&\
\argmin_{\recparam}
\biggl\{
\Expect_{(\stsignal, \stdata)\sim \widehat{\jointlaw}} 
\Bigl[\bigl\|
\stsignal - \RecOp_{\recparam}(\stdata)
\bigr\|_\RecSpace^2
\Bigr]
+
10^{-4}
\|\recparam\|^2
\biggr\}
\\
\variationparam^* \in&\
\argmin_{\variationparam}
\biggl\{
\Expect_{(\stsignal, \stdata)\sim \widehat{\jointlaw}} 
\Bigl[\Bigl\|
\DeepVariation_{\variationparam}(\stdata) - 
\bigl(
\stsignal - \RecOp_{\recparam^*}(\stdata)
\bigr)^2
\Bigr\|_\RecSpace^2 \Bigr]
+
10^{-4}
\|\variationparam\|^2
\biggr\}.
\end{align*}

\paragraph{Posterior sampling}
The \ac{WGAN} loss with the conditional \ac{WGAN} discriminator (\cref{sec:minibatchDiscriminator}) is the objective in \cref{eq:FinalFormulationMiniBatch}, i.e., it is given by
\begin{multline*}
L_{\text{W}}(\genparam, \discrparam)
:= \Expect_{\subalign{& (\stsignal,\stdata) \sim \widehat{\jointlaw} \\ & \stgenvar_1, \stgenvar_2 \sim \genvarprob}} 
    \biggl[
      \frac{1}{2} \Bigl( \Discriminator_{\discrparam}\bigl( \bigl(\stsignal, \Generator_{\genparam}(\stgenvar_1,\stdata) \bigr),\stdata \bigr)  
      +
      \Discriminator_{\discrparam}\bigl( \bigl( \Generator_{\genparam}(\stgenvar_1,\stdata), \stsignal \bigr),\stdata \bigr)
       \Bigr)
\\
      -
      \Discriminator_{\discrparam}\Bigl( \bigl(\Generator_{\genparam}(\stgenvar_1,\stdata),\Generator_{\genparam}(\stgenvar_2,\stdata) \bigr), \stdata\Bigr)  
    \biggr]
\end{multline*}
The set-up in \cref{eq:FinalFormulationMiniBatch} indicates that the discriminator should always be fully trained. Following best practice, instead of minimizing $(\genparam, \discrparam) \mapsto L_{\text{W}}(\genparam, \discrparam)$ jointly, we set-up an intertwined scheme where we take one step to minimize a generator loss  $\genparam \mapsto L_{\Generator}(\genparam)$ keeping $\discrparam$ fixed, then we take five steps to minimize a discriminator loss $\discrparam \mapsto L_{\Discriminator}(\discrparam)$ keeping $\genparam$ fixed. 
In the following, we explain how to construct these generator and discriminator losses.

For training the discriminator, note that $\discrparam \mapsto L_{\text{W}}(\genparam, \discrparam)$ is invariant w.r.t. adding an arbitrary constant to the discriminator. This causes the training to become unstable since the discriminator can drift \cite{karras2018progressive}. We levitate this by adding a small penalization
\[
	L_{\text{drift}}(\discrparam)
	:=
	\Expect_{(\stsignal, \stdata)\sim \widehat{\jointlaw}} \bigl[\Discriminator_\discrparam(\stsignal, \stdata)^2\bigr].
\]
Next, as in \cite{WGAN-GP}, we enforce the 1-Lipschitz condition for the discriminator (see \cref{sec:Wasserstein}) by adding the following gradient penalty term:
\begin{equation*}
L_{\text{grad}}(\genparam, \discrparam)
:=
\Expect_{\subalign{& (\stsignal, \stdata)\sim \widehat{\jointlaw} \\ & \epsilon \sim \mathcal{U}(0, 1) \\
   & \stgenvar_1, \stgenvar_2 \sim \genvarprob}}
\biggl[
\Bigr(
\bigr\|
\Gamma_{\genparam,\discrparam}(\stsignal,\stdata,\stgenvar_1,\stgenvar_2,\stepsilon)
\bigr\|_{\RecSpace^*}
- 1
\Bigr)^2
\biggr]
\end{equation*}
where $\Gamma_{\genparam,\discrparam} \colon \RecSpace \times \DataSpace \times \GenSpace \times \GenSpace \times [0,1] \to \RecSpace^*$ is given as 
\begin{multline*}
\Gamma_{\genparam,\discrparam}(\signal,\data,\genvar_1,\genvar_2,\varepsilon) 
\\
:=
\frac{1}{2} \Biggl\{
  \partial_1 \Discriminator_\discrparam
  \biggl(
      \varepsilon \bigl(\signal, \Generator_{\genparam}(\genvar_1, \data) \bigr)
      + 
      (1 - \varepsilon) \bigl( \Generator_{\genparam}(\genvar_1, \data), \Generator_{\genparam}(\genvar_2, \data) \bigr), 
      \,\data 
   \biggr)
\\
+
  \partial_1 \Discriminator_\discrparam
  \biggl(
      \varepsilon \bigl( \Generator_{\genparam}(\genvar_1, \data), \signal \bigr)
      + 
      (1 - \varepsilon) \bigl( \Generator_{\genparam}(\genvar_1, \data), \Generator_{\genparam}(\genvar_2, \data) \bigr), 
      \,\data 
   \biggr)
\Biggr\}
\end{multline*}
with $\partial_1 \Discriminator_\discrparam$ denoting the first order partial (Banach space) derivative w.r.t. the $(\RecSpace \times \RecSpace)$-variable of $\Discriminator_\discrparam \colon (\RecSpace \times \RecSpace) \times \DataSpace \to \Real$.
Then, the loss $\discrparam \mapsto L_{\Discriminator}(\discrparam)$ for training the discriminator (for fixed generator $\genparam$) becomes 
\begin{equation}\label{eq:impl_final_discr_loss}
L_{\Discriminator}(\discrparam)
	:=
	- L_{\text{W}}(\genparam, \discrparam)
	+
	10
	L_{\text{grad}}(\genparam, \discrparam)
	+
	10^{-3}
	L_{\text{drift}}(\genparam, \discrparam)
	+
	10^{-4}
	\|\discrparam\|^2,
\end{equation}
where the scalings $10$ and $10^{-3}$ were chosen according to best practice \cite{WGAN-GP, karras2018progressive} and not hand-tuned by us.

The loss $\genparam \mapsto L_{\Generator}(\genparam)$ for training the generator (for fixed discriminator $\discrparam$) is
\begin{equation}\label{eq:impl_final_gen_loss}
L_{\Generator}(\genparam)
	:=
	L_{\text{W}}(\genparam, \discrparam)
	+
	10^{-4}
	\|\genparam\|^2.
\end{equation}

\paragraph{Optimization for training}
We used the same optimization method to train all networks (both for direct and posterior sampling), which was the ADAM optimizer \cite{Adam} with $\beta_1 = 0.5$, $\beta_2 = 0.9$ and 50\,000 training steps ($\approx 8$ epochs). 
For the batch normalization \cite{BatchNormalization}, we used decay $0.9$ and $\varepsilon = 10^{-5}$.
Moreover, we reduced the learning rate following Noisy Linear Cosine Decay \cite{NoisyLinearCosineDecay} with default parameters, starting with a learning rate of $2 \cdot 10^{-4}$

Despite our data-augmentation and regularization, we observed some over-fitting during training, and expect that better results than ours could be obtained with more data.


\section{Handcrafted Priors}\label{sec:appendix_analytical_priors}
The samples were generated from Gibbs priors of the form $e^{-S(\signal)}$ where the regularization functional $S \colon \RecSpace \to \Real$ is chosen as indicated by the caption text for the images in the top row of \cref{fig:analytic_priors}. 

An interesting feature is that many of the samples from the priors shown in \cref{fig:analytic_priors} appear to be generated by a Gaussian random field prior. 
This may contradict the conventional wisdom that the choice of prior (regularizer) has a significant impact on the end result.
However, a closer consideration shows that this behavior is to be expected from theory.
It turns out that several priors, including the \ac{TV}-prior $S(\signal) := \|\nabla \signal \|_1$, converge weakly to a standard Gaussian free field as the discretization becomes finer as shown in \cite[Theorem~5.3]{Lassas:2004aa} for the \ac{TV}-prior and in \cite{Lassas:2009aa} for Besov space priors.
The differences in the regularized solution provided by the \ac{MAP} estimator are largely due to a small set of relatively unlikely images.
In conclusion, using such priors in Bayesian inversion of large scale inverse problems has very little effect over, e.g., using a Gaussian random field prior.


\bibliographystyle{plain}
\bibliography{references}

\end{document}